\PassOptionsToPackage{pagebackref=true, colorlinks=true, citecolor=green!60!black, linkcolor=blue}{hyperref}
\PassOptionsToPackage{linesnumbered,noend,algo2e,ruled}{algorithm2e}

\documentclass[11pt]{article}
\usepackage{fullpage}
\usepackage{parskip}

\usepackage{titletoc}

\pdfoutput=1

\usepackage[utf8]{inputenc} %
\usepackage{hyperref}       %
\usepackage{url}            %
\usepackage{booktabs}       %
\usepackage{amsfonts}       %
\usepackage{nicefrac}       %
\usepackage[round]{natbib}	
\usepackage{microtype}      %

\usepackage{amsmath,amsfonts,amsthm,amssymb}
\numberwithin{equation}{section}
\usepackage{thmtools,thm-restate}
\usepackage{enumitem}
\usepackage{bbm}
\usepackage{tikz}
\usepackage{xspace}
\usepackage[algo2e,ruled]{algorithm2e}

\newcommand{\algocfconts}[3]{%
  \@algocf@pre@ruled
  #2\label{#1}\kern2pt\hrule height.8pt depth0pt\kern2pt%
  #3\@algocf@pre@ruled
}
{%
  \ifundef{\algocf}%
  {`algorithm2e' package is required if you want to
   use the algorithm environment}%
  {}%
  \begin{algocf}[#1]%
  \renewcommand\@makecaption[2]{%
    \hskip\AlCapHSkip
    \parbox[t]{\hsize}{\algocf@captiontext{##1}{##2}}%
  }%
}%
{%
  \end{algocf}%
}

\usepackage{mathtools}
\usepackage{multirow}
\usepackage{subcaption}
\usepackage[makeroom]{cancel}
\usepackage{bigints}
\usepackage{relsize}
\usepackage{tabularx}
\usepackage{pgf,tikz}
\usetikzlibrary{patterns}
\usetikzlibrary{arrows}
\usepackage{mathrsfs}
\newcommand{\ud}{\,\mathrm{d}}
\newcommand{\tbeta}{\widetilde{\beta}}
\usepackage[toc,title]{appendix}

\newenvironment{enumerate*}%
  {\vspace{-2ex} \begin{enumerate} %
     \setlength{\itemsep}{-1ex} \setlength{\parsep}{0pt}}%
  {\end{enumerate}}

\newenvironment{itemize*}%
  {\vspace{-2ex} \begin{itemize} %
     \setlength{\itemsep}{-1ex} \setlength{\parsep}{0pt}}%
  {\end{itemize}}

\newenvironment{description*}%
  {\vspace{-2ex} \begin{description} %
     \setlength{\itemsep}{-1ex} \setlength{\parsep}{0pt}}%
  {\end{description}}

\DeclareMathOperator*{\E}{\mathbb{E}}
\let\Pr\relax
\DeclareMathOperator*{\Pr}{\text{Pr}}

\newcommand{\eps}{\varepsilon}
\newcommand{\hbeta}{\hat{\beta}}

\newcommand{\vx}{x}

\newcommand{\uu}{\underline{u}}

\newcommand{\ips}{\term{IPS}}
\newcommand{\hg}{\widehat{g}}

\newcommand{\tg}{\widetilde{g}}

\renewcommand{\b}{b}

\newcommand{\calP}{\mathcal{P}}

\newcommand{\calA}{\mathcal{A}}

\newcommand{\calD}{\mathcal{D}}

\newcommand{\calG}{\mathcal{G}}
\newcommand{\calI}{\mathcal{I}}
\newcommand{\calR}{\mathcal{R}}

\newcommand{\calO}{\mathcal{O}}
\newcommand{\tcalO}{\widetilde{\mathcal{O}}}

\newcommand{\hG}{\widehat{G}}

\newcommand{\1}{\mathbbm{1}}

\newcommand{\act}{\term{act}}
\newcommand{\childprod}{\c_{\term{prod}}}

\renewcommand{\S}{\mathcal{S}}
\newcommand{\Z}{Z}

\newcommand\numberthis{\addtocounter{equation}{1}\tag{\theequation}}
\newcommand{\ust}{u^{\star}}
\newcommand{\vxst}{\vx^{\star}}
\newcommand{\parent}{\term{parent}}

\newcommand{\xst}{x^{\star}}

\newcommand{\jst}{j^{\star}}

\newtheorem{theorem}{Theorem}[section]
\newtheorem{lemma}{Lemma}[section]

\newtheorem{definition}{Definition}[section]
\newtheorem{corollary}{Corollary}[section]

\newtheorem{remark}{Remark}[section]
\newtheorem{proposition}{Proposition}[section]
\newtheorem{claim}{Claim}[section]

\newcommand{\poly}{\text{\upshape poly}}

\newcommand{\squishlist}{
   \begin{list}{$\bullet$}
    { \setlength{\itemsep}{0pt}      \setlength{\parsep}{3pt}
      \setlength{\topsep}{3pt}       \setlength{\partopsep}{0pt}
      \setlength{\leftmargin}{1.5em} \setlength{\labelwidth}{1em}
      \setlength{\labelsep}{0.5em} } }
\newcommand{\squishend}{  \end{list}  }

\usepackage[suppress]{color-edits}
\addauthor{as}{blue}
\addauthor{cp}{red}

\usepackage{xspace,nicefrac}
\newcommand{\xhdr}[1]{\vspace{1mm} \noindent{\bf #1}}
\renewcommand{\refeq}[1]{Eq.~(\ref{#1})}
\newcommand{\OMIT}[1]{}
\newcommand{\ie}{{\em i.e.,~\xspace}}
\newcommand{\eg}{{\em e.g.,~\xspace}}
\newcommand{\term}[1]{\ensuremath{\mathtt{#1}}\xspace}
\newcommand{\tildeO}{\tcalO}
\newcommand{\tOmega}{\widetilde{\Omega}}
\newcommand{\LDOTS}{\, ,\ \ldots\ ,}     %
\newcommand{\rbr}[1]{\left(\,#1\,\right)}
\newcommand{\sbr}[1]{\left[\,#1\,\right]}
\newcommand{\cbr}[1]{\left\{\,#1\,\right\}}
\newcommand{\N} {\ensuremath{\mathbb{N}}} %
\newcommand{\EqComment}[1]{\textsl{(#1)}} %

\newenvironment{OneLiners}[1][\ensuremath{\bullet}]
    {\begin{list}
        {#1}
        {\setlength{\itemsep}{0pt}
          \setlength{\parsep }{0pt}
          \setlength{\topsep }{0pt}
        }}
    {\end{list}}

\newcommand{\covnum}{\mathcal{N}}   %
\newcommand{\covdim}{\term{CovDim}} %
\newcommand{\DblC}{C_\term{dbl}}    %

\renewcommand{\c}{\mathcal{C}}  %
\newcommand{\opt}{\vx^{\star}}  %
\newcommand{\ZoomDAG}{\term{ZoomDAG}}  %
\newcommand{\nodes}{\mathcal{U}}  %
\renewcommand{\root}{\term{root}} %

\newcommand{\arms}{\calA}       %
\newcommand{\dist}{\calD}       %
\newcommand{\gap}{\term{AdvGap}}   %
\newcommand{\gapIID}{\term{Gap}}   %
\newcommand{\zoomdim}{\term{ZoomDim}} %
\newcommand{\advzoomdim}{\term{AdvZoomDim}} %
\newcommand{\must}{\mu^{\star}}

\newcommand{\CONFtot}{\term{conf}^{\term{tot}}}
\newcommand{\CONFinst}{\term{conf}^{\term{inst}}}

\newcommand{\inhbias}{\text{inherited $\tau$-diameter}\xspace}

\newcommand{\repr}{\term{repr}} %
\newcommand{\mass}{\mathcal{M}} %
\newcommand{\activeN}{N^{\term{act}}} %

\newcommand{\reprset}{\arms_{\term{repr}}}

\newcommand{\calE}{\mathcal{E}}
\newcommand{\bxst}{\bar{x}^{\star}}
\newcommand{\ist}{i^{\star}}
\newcommand{\phist}{\phi^{\star}}
\newcommand{\rewE}{\calE^{\term{rew}}_{\term{clean}}}
\newcommand{\algE}{\calE^{\term{alg}}_{\term{clean}}}

\newcommand{\advzoom}{\textsc{AdversarialZooming}\xspace}
\newcommand{\ExpThree}{\term{EXP3}}
\newcommand{\ExpThreeP}{\term{EXP3.P}}
\newcommand{\UcbOne}{\term{UCB1}}

\title{Adaptive Discretization against an Adversary:\\ Lipschitz bandits, Dynamic Pricing, and Auction Tuning%
\footnote{A short version of this paper, titled ``Adaptive Discretization for Adversarial Lipschitz Bandits", was published in \emph{COLT 2021} (34th Conf. on Learning Theory). The conference version does not include applications to dynamic pricing and auction tuning (Section~\ref{sec:applications}), and omits detailed proofs (\ie all appendices). The material in Section~\ref{sec:applications}, except for Section~\ref{sec:DP}, is new as of June'25. 
}}

\date{First version: June 2020\\This version: June 2025}

\author{Chara Podimata\thanks{Part of the research was done while the author was an intern at Microsoft Research NYC.} \\
        MIT \\
        \texttt{podimata@mit.edu}
        \and Aleksandrs Slivkins \\
        Microsoft Research NYC \\
        \texttt{slivkins@microsoft.com}
        }

\begin{document}

\maketitle

\begin{abstract}
Lipschitz bandits is a prominent version of multi-armed bandits that studies large, structured action spaces such as the $[0,1]$ interval, where similar actions are guaranteed to have similar rewards. A central theme here is the adaptive discretization of the action space, which gradually ``zooms in'' on the more promising regions thereof. The goal is to take advantage of ``nicer'' problem instances, while retaining near-optimal worst-case performance. While the stochastic version of the problem is well-understood, the general version with adversarial rewards is not.

We provide the first algorithm (\emph{Adversarial Zooming}) 
 for adaptive discretization in the adversarial version, and derive instance-dependent regret bounds. In particular, we recover the worst-case optimal regret bound for the adversarial version, and the instance-dependent regret bound for the stochastic version.

We apply our algorithm to several fundamental applications---including dynamic pricing and auction reserve tuning---all under adversarial reward models. While these domains often violate Lipschitzness, our analysis only requires a weaker version thereof, allowing for meaningful regret bounds without additional smoothness assumptions. Notably, we extend our results to multi-product dynamic pricing with non-smooth reward structures, a setting which does not even satisfy one-sided Lipschitzness.

\end{abstract}

\newpage
\tableofcontents
\newpage

\section{Introduction}

Multi-armed bandits is a simple yet powerful model for decision-making under uncertainty, extensively studied since 1950ies and exposed in several books, \eg
\citep{Bubeck-survey12,slivkins-MABbook,LS19bandit-book}. In a basic version, the algorithm repeatedly chooses actions (a.k.a. arms) from a fixed action space, and observes their rewards. Only rewards from the chosen actions are revealed, leading to
the exploration-exploitation tradeoff.

We focus on \emph{Lipschitz bandits}, a prominent version that studies  large, structured action spaces such as the $[0,1]$ interval. Similar actions are guaranteed to have similar rewards, as per Lipschitz-continuity or a similar condition. In applications, actions can correspond to items with feature vectors, such as products, documents or webpages; or to offered prices for buying, selling or hiring; or to different tunings of a complex system such as a datacenter or an ad auction. \emph{Dynamic pricing}, a version where the algorithm is a seller and arms are prices, has attracted much attention on its own.

Extensive literature on Lipschitz bandits centers on two key themes. One is \emph{adaptive discretization} of the action space which gradually ``zooms in'' on the more promising regions thereof
\citep{LipschitzMAB-stoc08,LipschitzMAB-JACM, xbandits-nips08-conf,xbandits-nips08, ZoomingRBA-icml10, contextualMAB-colt11, Munos-nips11, ImplicitMAB-nips11, Valko-icml13,Minsker-colt13,Bull-bandits14,RepeatedPA-ec14, Grill-nips15}. This approach takes advantage of ``nice'' problem instances {-- ones in which near-optimal arms are confined to a relatively small region of the action space --} while retaining near-optimal worst-case performance. Another theme is relaxing and mitigating the Lipschitz assumptions
\citep{LipschitzMAB-stoc08,LipschitzMAB-JACM, xbandits-nips08-conf,xbandits-nips08, Bubeck-alt11,Munos-nips11, ImplicitMAB-nips11, Valko-icml13,Minsker-colt13,Bull-bandits14,RepeatedPA-ec14, Grill-nips15, SmoothedRegret-colt19}. The point of departure for all this literature is \emph{uniform discretization} \citep{KleinbergL03,Bobby-nips04,LipschitzMAB-stoc08,LipschitzMAB-JACM,xbandits-nips08-conf,xbandits-nips08}, a simple algorithm which discretizes the action space uniformly and obtains worst-case optimal regret bounds.%

{All these developments concern the \emph{stochastic} version, in which the rewards of each action are drawn from the same, albeit unknown, distribution in each round. A more general version allows the rewards to be adversarially chosen. Known as \emph{adversarial bandits}, this version is also widely studied in the literature \citep[starting from ][]{bandits-exp3}, and tends to be much more challending. The adversarial version of Lipschitz bandits is not understood beyond uniform discretization.}

We tackle adversarial Lipschitz bandits.
We provide the first algorithm for adaptive discretization for the adversarial case, and derive instance-dependent regret bounds.%
\footnote{That is, regret bounds which depend on the properties of the problem instance that are not known initially.} Our regret bounds are optimal in the worst case, and improve dramatically when the near-optimal arms comprise a small region of the action space. In particular, we recover the instance-dependent regret bound for the stochastic version of the problem \citep{LipschitzMAB-stoc08,LipschitzMAB-JACM, xbandits-nips08-conf,xbandits-nips08}.

Further, we consider \emph{dynamic pricing}, a well-studied bandit problem in which a seller repeatedly adjusts prices for some product(s), and the problem of adjusting a reserve price in a truthful auction. While these problems do not inherently satisfy Lipschitz-continuity, we show that our algorithm and analysis carry over without additional Lipschitz assumptions (touching upon the second theme mentioned above). Like in  Lipschitz bandits, prior work on the adversarial version of these problems was limited to uniform discretization.

\xhdr{Problem Statement: Adversarial Lipschitz Bandits.}
We are given a set $\arms$ of actions (a.k.a. \emph{arms}), the time horizon $T$, and a metric space $(\arms,\dist)$, also called the \emph{action space}. The adversary chooses randomized reward functions
$g_1 \LDOTS g_T:\arms \to [0,1]$.
 In each round $t$, the algorithm chooses an arm $x_t\in\arms$ and observes reward $g_t(x_t)\in [0,1]$ and nothing else.
We focus on the \emph{oblivious adversary}: all reward functions are chosen before round $1$. The adversary is restricted in that the expected rewards $\E[g_t(\cdot)]$ satisfy the Lipschitz condition:%
\footnote{The expectation in \eqref{eq:Lip} is over the randomness in the reward functions. While adversarial bandits are often defined with deterministic reward functions, it is also common to allow randomness therein, \eg to include stochastic bandits as a special case. The said randomness is essential to include \emph{stochastic} Lipschitz bandits as a special case. Indeed, for stochastic rewards, \eqref{eq:Lip} specializes to the Lipschitz condition from prior work on stochastic Lipschitz bandits. A stronger Lipschitz condition
    $g_t(x) - g_t(y) \leq \dist(x,y)$
is unreasonable for many applications; \eg if the rewards correspond to user's clicks or other discrete signals, we can only assume Lipschitzness ``on average".}
\begin{align}\label{eq:Lip}
\E\sbr{g_t(x) - g_t(y)} \leq \dist(x,y)\quad
    \forall x,y\in\arms,\,t\in [T].
\end{align}
The algorithm's goal is to minimize \emph{regret}, a standard performance measure in multi-armed bandits:
\begin{align}\label{eq:regret-defn}
R(T) := \textstyle \sup_{x\in\arms}\; \sum_{t\in[T]}
    g_t(x) - g_t(x_t).
\end{align}
A problem instance consists of action space $(\arms,\dist)$ and reward functions $g_1 \LDOTS g_T$. The stochastic version of the problem (\emph{stochastic rewards}) posits that each $g_t$ is drawn independently from some fixed but unknown distribution $\calG$. A problem instance is then the tuple $(\arms,\dist,\calG,T)$.

The canonical examples are a $d$-dimensional unit cube
    $(\arms,\dist) = ([0,1]^d,\,\ell_p)$, $p\geq 1$
(where $\ell_p(x,y) = \|x-y\|_p$ is the $p$-norm), and the \emph{exponential tree metric}, where $\arms$ is a leaf set of a rooted infinite tree, and the distance between any two leaves is exponential in the height of their least common ancestor. Our results are equally meaningful for large but finite action sets.

{Our results are most naturally stated without an explicit Lipschitz constant $L$. The latter is implicitly ``baked into" the metric $\dist$, \eg if the action set is $[0,1]$ one can take $\dist(x,y) =L\, |x-y|$. However, we investigate the dependence on $L$ in corollaries. Absent $L$, one can take $\dist\leq 1$ w.l.o.g.}

\xhdr{Our Results.}
We present \advzoom, an algorithm for adaptive discretization of the action space. Our main result is a regret bound of the form
\begin{align}\label{eq:regret-generic}
\E[R(T)] \leq \tcalO(\; T^{(z+1)/(z+2)}\;),
\end{align}
where $z=\advzoomdim\geq 0$ is a new quantity called the \emph{adversarial zooming dimension}.%
\footnote{As usual, the $\tildeO(\cdot)$ and $\tOmega(\cdot)$ notation hides $\mathrm{polylog}(T)$ factors.} {This quantity, determined by the problem instance, measures how wide-spread the near-optimal arms are in the action space.} In fact, we achieve this regret bound with high probability.

The meaning of this result is best seen via corollaries:

\begin{itemize}
\item
We recover the optimal \emph{worst-case} regret bounds for the adversarial version. Prior work \citep{Bobby-nips04,LipschitzMAB-stoc08,LipschitzMAB-JACM,xbandits-nips08-conf,xbandits-nips08} obtains \refeq{eq:regret-generic} for the $d$-dimensional unit cube, and more generally \refeq{eq:regret-generic} with $z = \covdim$, the covering dimension of the action space. The latter bound is the best possible for any given action space. We recover it in the sense that $\advzoomdim\leq \covdim$. Moreover, we match the worst-case optimal regret
    $\tildeO(\sqrt{KT})$
for instances with $K<\infty$ arms and any metric space.

\item We recover the optimal \emph{instance-dependent} regret bound from prior work on the stochastic version \citep{LipschitzMAB-stoc08,LipschitzMAB-JACM, xbandits-nips08-conf,xbandits-nips08}. This bound is \refeq{eq:regret-generic} with $z=\zoomdim$, an instance-dependent quantity called the \emph{zooming dimension}, and it is the best possible for any given action space and any given value of $\zoomdim$ \citep{contextualMAB-colt11}. $\zoomdim$ can be anywhere between $0$ and $\covdim$, depending on the problem instance. We prove that, essentially, $\advzoomdim= \zoomdim$ for stochastic rewards.

\item Our regret bound can similarly improve over the worst case even for adversarial rewards.
    In particular, we may have $\advzoomdim=0$ for arbitrarily large $\covdim$, even if the  reward functions change substantially. Then we obtain $\tildeO(\sqrt{T})$ regret, as if there were only two arms.

\end{itemize}

Adaptive discretization algorithms from prior work \citep{LipschitzMAB-stoc08,LipschitzMAB-JACM, xbandits-nips08-conf,xbandits-nips08} do not extend to the adversarial version. For example, specializing to $K$-armed bandits with uniform metric $\dist \equiv 1$, these algorithms reduce to a standard algorithm for stochastic bandits \citep[\term{UCB1},][]{bandits-ucb1}, which fails badly for many simple instances of adversarial rewards.

{The per-round running time of our algorithm is
    $\tildeO\rbr{T^{d/(d+2)}}$,
where $d = \covdim$, matching the running time of uniform discretization (with \ExpThree \citep{bandits-exp3}, a standard algorithm for adversarial bandits). In fact, we obtain a better running time when $\advzoomdim<\covdim$.}

\xhdr{Adversarial Zooming Dimension.}
The new notion of $\advzoomdim$ can be defined in a common framework with $\covdim$ and $\zoomdim$ from prior work. All three notions are determined by the problem instance, and talk about \emph{set covers} in the action space. Each notion specifies particular subset(s) of arms to be covered, denoted $\arms_\eps\subset \arms$, $\eps>0$, and counts how many ``small" subsets are needed to cover each $\arms_\eps$. For a parameter $\gamma>0$ called the \emph{multiplier}, the respective ``dimension" is
\begin{align}\label{eq:dims-generic}
\inf\cbr{d\geq 0:\; \text{$\arms_\eps$ can be covered with  $\gamma\cdot \eps^{-d}$ sets of diameter at most $\eps$},
\quad\forall \eps>0}.
\end{align}
Generally, a small ``dimension" quantifies the simplicity of a problem instance.

The covering dimension $\covdim$ has $\arms_\eps \equiv \arms$. The intuition comes from the $d$-dimensional cube, for which $\covdim=d$.
\footnote{More formally, the covering dimension of $([0,1]^d,\,\ell_p)$, $p\geq 1$ is $d$, with multiplier $\gamma = \text{poly}(d,p)$.}
Thus, we are looking for the covering property enjoyed by the unit cube. Note that $\covdim$ is determined by the action space alone, and is therefore known to the algorithm.

Both $\zoomdim$ and $\advzoomdim$ are about covering near-optimal arms. Each subset $\arms_\eps$ comprises all arms that are, in some sense, within $\eps$ from being optimal. These subsets may be easier to cover compared to $\arms$; this may reduce~\refeq{eq:dims-generic} compared to the worst case of $\covdim$.

The zooming dimension $\zoomdim$ is only defined for stochastic rewards. It focuses on the standard notion of \emph{stochastic gap} of an arm $x$ compared to the best arm:
    $\gapIID(x) := \max_{y\in \arms}\E[g_t(y)]- \E[g_t(x)]$.
Each subset $\arms_\eps$ is defined as the set of all arms $x\in\arms$ with
    $\gapIID(x) \leq O(\eps)$.

$\advzoomdim$ extends $\zoomdim$ as follows. The \emph{adversarial gap} of a given arm $x$ measures this arm's suboptimality compared to the best arm on a given time-interval $[0,t]$. Specifically,
\begin{align}\label{eq:AdvGap-defn}
\gap_t(x) :=
    \tfrac{1}{t}\textstyle \max_{y\in\arms}
        \sum_{\tau\in [t]} g_\tau(y) - g_\tau(x).
\end{align}
Given $\eps>0$, an arm $x$ is called inclusively $\eps$-optimal if
    $\gap_t(x) <\calO(\eps\, \ln^{3/2} T)$
for some end-time $t>\Omega(\eps^{-2})$;
the precise definition is spelled out in \refeq{eq:eps-optmal-defn}.
In words, we include all arms whose adversarial gap is sufficiently small at some point in time. It suffices to restrict our attention to a  \emph{representative set}  of arms $\reprset\subset \arms$ with
    $|\reprset|\leq O(T^{1+\covdim})$,
specified in the analysis.%
\footnote{Essentially, $\reprset$ contains a uniform discretization for scale $\nicefrac{1}{T}$ and also the local optima for such discretization.}
Thus, the subset $\arms_\eps$ is defined as the set of all arms $x\in \reprset$ that are inclusively $\eps$-optimal.

By construction,
    $\advzoomdim\leq\covdim$
for any given multiplier $\gamma>0$. For stochastic rewards, $\advzoomdim$ coincides with $\zoomdim$ up to a $\text{polylog}\rbr{T,\, |\reprset|}$ multiplicative change in $\gamma$.

The definition of $\advzoomdim$ is quite flexible. First, we achieve the stated regret bound for all $\gamma>0$ at once, with a multiplicative $\gamma^{1/(z+2)}$ dependence thereon. Second, we could relax \eqref{eq:dims-generic} to hold only for $\eps$ smaller than some threshold $\theta$; the regret bound increases by
    $+\tildeO(\sqrt{T\,\theta^{-\covdim}})$.

\xhdr{Examples.} We provide a flexible family of examples with small $\advzoomdim$. Fix an arbitrary action space $(\calA,\calD)$ and time horizon $T$. Consider $M$ problem instances with stochastic rewards, each with $\zoomdim\leq d$. Construct an instance with adversarial rewards, where each round is assigned in advance to one of these stochastic instances. This assignment can be completely arbitrary: \eg the stochastic instances can appear consecutively in ``phases" of arbitrary duration, or they can be interleaved in an arbitrary way. Then $\advzoomdim\leq d$ for constant $M,d$ under some assumptions.

In particular, we allow arbitrary disjoint subsets $S_1 \LDOTS S_M\subset \arms$ such that each stochastic instance $i\in [M]$ can behave arbitrarily on $S_i$ as long as the spread between the largest and smallest mean rewards exceeds a constant. All arms outside $S_i$ receive the same ``baseline" mean reward, which does not exceed the mean rewards inside $S_i$. The analysis of this example is somewhat non-trivial, and separate from the main regret bound \eqref{eq:regret-generic}.

\xhdr{Challenges and Techniques.}
We build on the high-level idea of \emph{zooming} from prior work on the stochastic version \citep{LipschitzMAB-stoc08,LipschitzMAB-JACM, xbandits-nips08-conf,xbandits-nips08}, but provide a very different implementation of this idea. We maintain a partition of the action space into ``active regions'', and refine this partition adaptively over time. We ``zoom in'' on a given region by partitioning it into several ``children'' of half the diameter; we do it only if the sampling uncertainty goes below the region's diameter. In each round, we select an active region according to (a variant of) a standard algorithm for bandits with a fixed action set, and then sample an arm from the selected region according to a fixed, data-independent rule. The standard algorithm we use is $\ExpThreeP$ \citep{bandits-exp3}; prior work on stochastic rewards used \UcbOne \citep{bandits-ucb1}.

Adversarial rewards bring about several challenges compared to the stochastic version.
First,
the technique in \ExpThreeP does not easily extend to variable number of arms (when the action set is increased via ``zooming''), whereas the technique in \UcbOne does, for stochastic rewards.
Second,
the sampling uncertainty is not directly related to the total probability mass allocated to a given region. In contrast, this relation is straightforward and crucial for the stochastic version.
Third,
the adversarial gap is much more difficult to work with. Indeed, the analysis for stochastic rewards relies on two easy but crucial steps --- bounding the gap for regions with small sampling uncertainty, and bounding the ``total damage'' inflicted by all small-gap arms --- which break adversarial rewards.

These challenges prompt substantial complications in the algorithm and the analysis. For example, to incorporate ``children'' into the multiplicative weights analysis, we split the latter into two steps: first we update the weights, then we add the children. To enable the second step, we partition the parent's weight equally among the children. Effectively, we endow each child with a copy of the parent's data, and we need to argue that the latter is eventually diluted by the child's own data.

Another example: to argue that we only ``zoom in'' if the parent has small adversarial gap, we need to enhance the ``zoom-in rule'': in addition to the ``aggregate'' rule (the sampling uncertainty must be sufficiently small), we need the ``instantaneous'' one: the current sampling probability must be sufficiently large, and it needs to be formulated in just the right way. Then, we need to be much more careful about deriving the ``zooming invariant'', a crucial property of the partition enforced by the ``zoom-in rule''. In turn, this derivation necessitates the algorithm's parameters to change from round to round, which further complicates the multiplicative weights analysis.

An important part of our contribution is formalizing what we mean by ``nice" problem instances, and boiling the analysis down to an easily interpretable notion such as $\advzoomdim$.

\xhdr{Applications: Dynamic Pricing and Auction Tuning.}
We apply our machinery to dynamic pricing and tuning a reserve price in a truthful auction. Here, an algorithm is a seller with unlimited supply of identical items. Standard examples in the literature include digital items (\eg songs, movies or software) and advertising opportunities. In each round $t$, the algorithm chooses a price $x_t\in [0,1]$, offers one item for sale, and observes its reward: the revenue from the sale, if any.

We consider three problem variants, under a common framing of an adversarial bandit problem with action set $[0,1]$. The basic  variant is \emph{dynamic pricing}, where the sale price is $x_t$ in each round $t$. Its stochastic version is well-studied (see Related Work).
Next, we consider a repeated second-price auction with reserve price $x_t$, extending the stochastic version studied in \citet{RepeatedAuctions-soda13}. Third, we consider divisible goods, \eg advertisement opportunities that can be divided among advertisers. We start with an arbitrary truthful single-parameter auction, and ``tune it" with reserve price $x_t$. The precise technical setup for all three variants is spelled out in Section~\ref{sec:apps-defns}.

We obtain regret bound~\eqref{eq:regret-generic} for each problem variant. We remark that our analysis does \emph{not immediately apply} because the Lipschitz condition~\eqref{eq:Lip} is not inherently satisfied; yet, we recover \eqref{eq:regret-generic} \emph{without any additional assumptions}. We accomplish this by leveraging the ``one-sided" Lipschitz property satisfied in each problem variant, which suffices for our analysis: \asedit{for some constant $c_0\geq 1$,}
\begin{align}\label{eq:Lip-DP}
g_t(x) -g_t(x') \leq c_0(x-x') \quad
\text{for any prices $x,x'\in [0,1]$ with $x>x'$}.
\end{align}
This property is immediate for dynamic pricing \asedit{(with $c_0=1$)}, but requires some work for the other two variants.

$\advzoomdim$ can improve over the worst case of $\covdim=1$, \eg in the family of examples described above. We obtain regret $\tildeO(T^{2/3})$ in the worst case, which is optimal even for stochastic dynamic pricing \citep{KleinbergL03}. Prior work (in the adversarial setting) is limited to uniform discretization for dynamic pricing \citep{KleinbergL03}, obtaining regret rate of $\tildeO(T^{2/3})$. Adaptive discretization was only known for the stochastic version (and only for dynamic pricing).

We also consider a \emph{multi-product} generalization of dynamic pricing. Here, the seller has $d$ products for sale, with infinite supply of each, and posts a separate price $x_{i,t}\in[0,1]$ for each product $i$ and each round $t$. Each customer only buys one item of a single product (if any). %
Again, we obtain regret bound~\eqref{eq:regret-generic}. The challenge is that we do \emph{not even have one-sided Lipschitzness} \eqref{eq:Lip-DP}, contrary to the ``single-dimensional" applications discussed above. Instead, we leverage a  non-trivial parameterization of \advzoom and a type of ``monotonicity'' present in the problem, and obtain a certain  Lipschitz-like property (\refeq{eq:sketch-lipschitz-prop}) that underpins our analysis.

\xhdr{Remarks.}
We obtain an \emph{anytime} version, with similar regret bounds for all time horizons $T$ at once, using the standard \emph{doubling trick}: in each phase $i\in\N$, we restart the algorithm with time horizon $T=2^i$. The only change is that the definition of \advzoomdim redefines $\arms_\eps$ to be the set of all arms that are inclusively $\eps$-optimal within some phase.

Our regret bound depends sublinearly on the \emph{doubling constant} $\DblC$: the smallest $C\in\N$ such that any ball can be covered with $C$ sets of at most half the diameter. Note that $\DblC = 2^d$ for a $d$-dimensional unit cube, or any subset thereof. The doubling constant has been widely used in theoretical computer science, e.g., see
\citet{Slivkins-focs04} for references.

\subsection{Related Work}\label{sec:rel-work}

\xhdr{Lipschitz Bandits.}
Our paper is primarily related to the vast literature on Lipschitz bandits, which were introduced in \citep{agrawal-bandits-95} for action space $[0,1]$, and optimally solved in the worst case via uniform discretization in \citep{Bobby-nips04,LipschitzMAB-stoc08,LipschitzMAB-JACM, xbandits-nips08-conf,xbandits-nips08}.
Adaptive discretization was introduced in \citep{LipschitzMAB-stoc08,LipschitzMAB-JACM, xbandits-nips08-conf,xbandits-nips08}, and subsequently extended to contextual bandits \citep{contextualMAB-colt11}, ranked bandits \citep{ZoomingRBA-icml10}, and contract design for crowdsourcing \citep{RepeatedPA-ec14}.
(The terms ``zooming algorithm/dimension'' trace back to \cite{LipschitzMAB-stoc08,LipschitzMAB-JACM}.)
\citet{LipschitzMAB-stoc08,LipschitzMAB-JACM} consider regret rates with instance-dependent constants (\eg $\log(t)$ for finitely many arms), and build on adaptive discretization ideas to characterize the worst-case optimal regret rates for any given metric space. Pre-dating their work on adaptive discretization, \citet{Kocsis-ecml06,yahoo-bandits07,Munos-uai07} allow a ``taxonomy'' on arms without any numerical information (and without any non-trivial regret bounds).

Several papers recover adaptive discretization guarantees under mitigated Lipschitz conditions:
when Lipschitzness only holds near the best arm $x^{\star}$ or when one of the two arms is $x^{\star}$ \citep{LipschitzMAB-stoc08,LipschitzMAB-JACM, xbandits-nips08-conf,xbandits-nips08};
when the algorithm is only given a taxonomy of arms, but not the metric \citep{ImplicitMAB-nips11,Bull-bandits14};
when the actions correspond to contracts offered to workers, and no Lipschitzness is assumed \citep{RepeatedPA-ec14},
and when expected rewards are H\"{o}lder-smooth with an unknown exponent \citep{Locatelli-colt18}.

In other work on mitigating Lipschitzness, \citet{Bubeck-alt11} recover the optimal worst-case bound with unknown Lipschitz constant.
\citet{Munos-nips11,Valko-icml13,Grill-nips15} consider adaptive discretization in the ``pure exploration'' version, and allow for a parameterized class of metrics with unknown parameter. \citet{SmoothedRegret-colt19} posit a weaker, ``smoothed'' benchmark and recover adaptive discretization-like regret bounds without any Lipschitz assumptions.

Contrary to our paper, all work discussed above assumes \emph{stochastic} rewards. Adaptive discretization has been extended to expected rewards with bounded change over time \citep{contextualMAB-colt11}, and to a version with ergodicity and mixing assumptions \citep{Azar-icml14}. For Lipschitz bandits with adversarial rewards --- which are the central focus of the present paper --- the uniform discretization approach easily extends \citep{Bobby-nips04}, and \emph{nothing else is known} prior to our work.%
\footnote{We note that \citet{Munos-ecml10} achieve $O(\sqrt{T})$ regret for the full-feedback version.}

\citet{AlgoChaining-colt17} obtain improved regret rates for some other adversarial continuum-armed bandit problems. Their results, which use different techniques than ours, require more-than-bandit feedback of a specific shape; this holds \eg for learning reserve prices in contextual second-price auctions, but not for dynamic pricing.

Algorithms for continuum-armed and Lipschitz bandits tend to have intensive computational and storage requirements, even for adaptive discretization. Several recent papers
\citep{feng2022lipschitz,zhulipschitz,SmoothedRegret-nips20,zhu2022contextual}
tackle this challenge from various angles.

While Lipschitz bandits only capture ``local'' similarity between arms, other structural models such as convex bandits
\citep[\eg][]{FlaxmanKM-soda05,AgarwalFHKR-nips11,bubeck2017kernel}
and linear bandits
\citep[\eg][]{DaniHK-colt08,AbernethyHR-colt08,Csaba-nips11}
allow for \emph{long-range inferences}: by observing some arms, an algorithm learns something about other arms that are far away. This is why $\tildeO(\sqrt{T})$ regret rates are achievable in adversarial versions of these problems, via different techniques.

\xhdr{Best-of-both-worlds: stochastic vs adversarial.}
Our algorithm achieves optimal worst-case performance for adversarial rewards (in terms of $\covdim$) and matches best-known guarantees for stochastic rewards (in terms of $\zoomdim$). This theme --- performing optimally in both stochastic and adversarial environments, without knowing the environment type in advance --- is known as \emph{best-of-both worlds}. Such results for $K$-armed bandits were obtained in \citet{BestofBoth-colt12} and refined in 
\citet{Auer-colt16,BestofBoth-icml14,Seldin-colt17,Zimmert-jmlr21} (focusing on logarithmic regret for stochastic rewards and $\sqrt{T}$ regret for adversarial rewards). Results of the same flavor were obtained for other domains, \eg  combinatorial semi-bandits \citep{Haipeng-BoB019}, linear bandits \citep{wei2021best,zhou2022almost}, and bandits with knapsacks \citep{Castiglioni-icml22}. 

Further, we obtain improved guarantees in the ``intermediate regime" when the rewards are neither stochastic nor worst-case adversarial (via $\advzoomdim)$. In the literature on ($K$-armed) adversarial bandits, this theme has been addressed by bounding the ``power" of the adversary in various ways
\citep[\eg][]{Hazan-soda09,contextualMAB-colt11,Haipeng-colt18,Thodoris-stoc18}.

\xhdr{Dynamic Pricing.}
Dynamic pricing has a long history in operations research \citep[survey:][]{Boer-survey15}. The regret-minimizing version was optimally solved in the worst case in \citet{KleinbergL03}, via uniform discretization and an intricate lower bound. Extensions to limited supply were studied in \citet{BZ09,BesbesZeevi-OR11,DynPricing-ec12,Wang-OR14,BwK-focs13}. Departing from the stochastic version, \cite{BesbesZeevi-OR11,Keskin-Zeevi-MathOR17,leme2021learning} allow bounded change over time. 
 \citet{cesa2019dynamic} study a version of the problem where the learner is facing finitely many unknown user valuations. \cite{Cohenetal,Lobeletal,intrinsicvol18,LiuetalSODA21,STOC21} study the \emph{contextual} version of the problem, where the value for an item is linearly dependent on a publicly observed context, and a common ground-truth value-vector which is the same for all the agents.

\emph{Learning to bid} in adversarial repeated auctions is related as an online learning problem \asedit{with actions that correspond to monetary amounts}. However, prior work is incomparable with ours, as it assumes more-than-bandit feedback. Specifically,~\cite{Weed-colt16} posit full feedback whenever the algorithm wins the second-price auction, while~\cite{Chara-EC18} do so for the generalized second price auction. \cite{Hanetal} assume full feedback in a learning setting for first-price auctions, which essentially obviates the need for adaptive discretization. Instead, in our model the algorithm only observes whether a purchase happened.

\subsection{Paper Organization}
Our algorithm is presented in Section~\ref{sec:algo}.
Sections~\ref{sec:results} and~\ref{sec:analysis}
state our results and outline the regret analysis. The detailed proofs can be found in the appendices: the weight-update step (Appendix~\ref{app:MWU}) and full regret analysis (Appendix~\ref{sec:analysis-details}). %
The equivalence of $\advzoomdim$ and $\zoomdim$ under stochastic rewards is found in Section~\ref{app:advzoomdim}. The examples with small $\advzoomdim$ are spelled out in Section~\ref{app:examples} and proved in Appendix~\ref{app:examples-proof}. Applications to dynamic pricing and auction tuning are in Section~\ref{sec:applications}.

While the results in Section~\ref{sec:results} are stated in full generality, we present the algorithm and the analysis for the special case of $d$-dimensional unit cube
for ease of exposition. The extension to arbitrary metric spaces requires a careful decomposition of the action space, but no new ideas otherwise; it is outlined in Appendix~\ref{sec:metrics}.

\section{Our Algorithm: Adversarial Zooming}
\label{sec:algo}

For ease of presentation, we develop the algorithm for the special case of $d$-dimensional unit cube,
    $(\arms,\dist) = ([0,1]^d,\,\ell_\infty)$.
Our algorithm partitions the action space into axis-parallel hypercubes. More specifically, we consider a rooted directed tree, called the \emph{zooming tree}, whose nodes correspond to axis-parallel hypercubes in the action space. The root is $\arms$, and each node $u$ has $2^d$ children that correspond to its quadrants. For notation, $\nodes$ is the set of all tree nodes, $\c(u)$ is the set of all children of node $u$, and
    $L(u) = \max_{x,y\in u} \dist(x,y)$
is its diameter in the metric space.%
\footnote{{For $\dist=\ell_\infty$, $L(u)$ is simply the edge-length of the hypercube.}}
Note that $L(\cdot)\leq 1$.

On a high level, the algorithm operates as follows. We maintain a set $A_t\subset \nodes$ of tree nodes in each round $t$, called \emph{active nodes}, which partition the action space. We start with a single active node, the root. After each round $t$, we may choose some node(s) $u$ to ``zoom in'' on according to the \emph{zoom-in rule}, in which case we de-activate $u$ and activate its children. We denote this decision with
    $z_t(u) = \1\{ \text{zoom in on $u$ at round $t$}\}$.
In each round, we choose an active node $U_t$ according to the \emph{selection rule}. Then, we choose a representative arm
    $x_t  = \repr_t(U_t) \in U_t$
to play in this round. The latter choice can depend on $t$, but not on the algorithm's observations; the choice could be randomized, \eg we could choose uniformly at random from $U_t$.

The main novelty {of our algorithm} is in the zoom-in rule. However, presenting it requires some scaffolding: we need to present the rest of the algorithm first. The selection rule builds on \ExpThree \citep{bandits-exp3}, a standard algorithm for adversarial bandits.
We focus on \ExpThreeP, a variant that uses ``optimistic'' reward estimates, the
inverse propensity score (IPS) plus a ``confidence term'' (see \refeq{eq:alg-ghat}). This is because we need a similar ``confidence term'' from the zooming rule to ``play nicely'' with the \ExpThree machinery. If we never zoomed in \emph{and} used $\eta=\eta_t$ for multiplicative updates in each round, then our algorithm would essentially coincide with \ExpThreeP. Specifically, we maintain weights $w_{t,\eta}(u)$ for each active node $u$ and round $t$, and update them multiplicatively, as per \refeq{eq:alg:MW}. In each round $t$, we define a probability distribution $p_t$ on the active nodes, proportional to the weights $w_{t,\eta_t}$. We sample from this distribution, mixing in some low-probability uniform exploration.

\newcommand{\ParamInt}{(0,\nicefrac{1}{2}]}

We are ready to present the pseudocode (Algorithm~\ref{algo:adv-zoom}). The algorithm has parameters
    $\beta_t,\gamma_t,\eta_t \in \ParamInt$
for each round $t$; we fix them later in the analysis as a function of $t$ and $|A_t|$. Their meaning is that $\beta_t$ drives the ``confidence term'', $\gamma_t$ is the probability of uniform exploration, and $\eta_t$ parameterizes the multiplicative update. To handle the changing parameters $\eta_t$, we use a trick from \citep{Bubeck-thesis,Bubeck-survey12}: conceptually, we maintain the weights $w_{t,\eta}$ for all values of $\eta$ simultaneously, and plug in $\eta = \eta_t$ only when we compute distribution $p_t$. Explicitly maintaining all these weights is cleaner and mathematically well-defined, so this is what our pseudocode does.

\LinesNumberedHidden
\begin{algorithm2e}[t]
\caption{\textsc{AdversarialZooming}}\label{algo:adv-zoom}
\DontPrintSemicolon
\SetAlgoLined
{\bf Parameters:}
    $\beta_t,\gamma_t,\eta_t \in \ParamInt$
    ~~~for each round $t$.\;
{\bf Variables:}
    active nodes $A_t\subset \nodes$,
    weights $w_{t,\eta}:\nodes\to (0,\infty]$
    ~~~$\forall$ round $t$, $\eta\in\ParamInt$\;
{\bf Initialization:}
    $w_1(\cdot) = 1$
and
    $A_1 = \{\term{root}\}$
and
    $\beta_1 = \gamma_1 = \eta_1 = \nicefrac{1}{2}$.\;
\For{$t = 1, \dots, T$}{
    $p_t \leftarrow \cbr{
                \text{distribution $p_t$ over $A_t$,
                proportional to weights $w_{t,\eta_t}$}
            }$.\;
    Add uniform exploration: distribution
        $ \pi_t(\cdot) \leftarrow
            \rbr{1 - \gamma_{t}}\, p_t(\cdot) + \gamma_t/|A_t|$
    over $A_t$. \;
    Select a node $U_t \sim \pi_t(\cdot)$, and then its representative: $x_t = \repr(U_t)$. \tcp*{selection rule}
    Observe the reward $g_t(x_t) \in [0,1]$.\;
    \For{$u \in A_t$}{
        \vspace{-3mm}
        \begin{align}\label{eq:alg-ghat}
        \hg_t(u)
            = \frac{g_t(x_t)\cdot \1 \left\{u = U_t \right\}}{\pi_t(u)}
         + \frac{\left(1 + 4\log T \right)\beta_{t}}{\pi_t(u)}
         \quad\text{\tcp*{IPS + "conf term"}}
        \end{align}
        \vspace{-7mm}
        \begin{align}\label{eq:alg:MW}
        w_{t+1,\,\eta}(u) = w_{t,\eta}(u)
            \cdot \exp \rbr{\eta\cdot \hg_t(u)},\; \forall \eta\in \ParamInt \qquad\text{\tcp*{MW update}}
        \end{align}
        \If(\tcp*[f]{zoom-in rule}){$z_t(u) = 1$}{
            $A_{t+1} \leftarrow A_t \cup \c(u)\setminus \{u\}$
            \tcp*{activate children of $u$, deactivate $u$}
            $w_{t+1}(v) = w_{t+1}(u)/|\c(u)|$
            ~~for all $v \in \c(u)$.
            \tcp*{split the weight}
        }
        }
}
\end{algorithm2e}

For the subsequent developments, we need to carefully account for the ancestors of the currently active nodes. Suppose node $u$ is active in round $t$, and we are interested in some earlier round $s\leq t$. Exactly one ancestor of $u$ in the zooming tree has been active then; we call it the \emph{active ancestor} of $u$ and denote $\act_s(u)$. If $u$ itself was active in round $s$, we write  $\act_s(u) = u$.

For computational efficiency, we do not explicitly perform the multiplicative update \eqref{eq:alg:MW}. Instead, we recompute the weights $w_{t,\,\eta_t}$ from scratch in each round $t$, using the following characterization:%
\footnote{We also use Lemma~\ref{lem:wu-rule} in several places in the analysis. The proof can be found in Appendix~\ref{app:MWU}.}

\begin{restatable}{lem}{weightupdate}
\label{lem:wu-rule}
Let
 $\childprod(u) = \prod_v |\c(v)|$,
where $v$ ranges over all ancestors of node $u$ in the zooming tree (not including $u$ itself). Then for all nodes $u \in A_t$, rounds $t$, and parameter values $\eta\in\ParamInt$,
\begin{align}\label{eq:weight-characterization}
w_{t+1,\,\eta}(u)
    =  \childprod^{-1}(u) \cdot \exp
    \textstyle
        \rbr{ \eta \sum_{\tau\in[t]} \hg_\tau(\act_\tau(u))}.
\end{align}
\end{restatable}

\vspace{-2mm}

\xhdr{Remarks.}
We make no restriction on how many nodes $u$ can be ``zoomed-in'' in any given round. However, our analysis implies that we cannot immediately zoom in on any ``newborn children''.

When we zoom in on a given node, we split its weight equally among its children. Maintaining the total weight allows the multiplicative weights analysis to go through, and the equal split allows us to conveniently represent the weights in Lemma~\ref{lem:wu-rule} (which is essential in the multiplicative weights analysis, too). An undesirable consequence is that we effectively endow each child with a copy of the parent's data; we deal with it in the analysis via \refeq{eq:sketch-inherited-bias}.

The meaning of the confidence term in \refeq{eq:alg-ghat} is as follows. Define the \emph{total confidence term}
\begin{align}\label{eq:confTot}
\CONFtot_t(u) :=
\textstyle 1/\beta_{t} +
 \sum_{\tau \in [t]} \; \beta_{\tau}/ \pi_\tau(\act_\tau(u)).
\end{align}
Essentially, we upper-bound the cumulative gain from node $u$ up to time $t$ using
\begin{align}\label{eq:UCB-u}
\textstyle
\CONFtot_t(u) + \sum_{\tau\in [t]}\; \ips_t(\act_t(u)),
\quad\text{where}\quad
\ips_t(u) := g_t(x_t)\cdot \1 \left\{u = U_t \right\}/\pi_t(u).
\end{align}
The $+4\log T$ term in \refeq{eq:alg-ghat} is needed to account for the ancestors later in the analysis; it would be redundant if there were no zooming and the active node set were fixed throughout.

\xhdr{The Zoom-In Rule.} Intuitively, we want to zoom in on a given node $u$ when its per-round sampling uncertainty gets smaller than its diameter $L(u)$, in which case exploration at the level of $u$ is no longer productive. A natural way to express this is
    $\CONFtot_t(u)\leq t\cdot L(u)$,
which we call the \emph{aggregate} zoom-in rule. However, it does not suffice: we also need an \emph{instantaneous} version which asserts that the current sampling probability is large enough. Making this precise is somewhat subtle. Essentially, we lower-bound $\CONFtot_t(u)$ as a sum of ``instantaneous confidence terms''
\begin{align}\label{eq:confInst}
\CONFinst_\tau(u) :=
    \tbeta_\tau + \beta_{\tau}/\pi_\tau(\act_\tau(u)),
\quad \tau\in[t],
\end{align}
where $\tbeta_\tau\in \ParamInt$ are new parameters. We require each such term to be at most $L(u)$. In fact, we require a stronger upper bound $e^{L(u)}-1$, which plugs in nicely into the multipliticative weights argument, and implies an upper bound of $L(u)$. Thus, the zoom-in rule is as follows:
\begin{align}\label{eq:zoom-in-rule}
z_t(u) :=
    \1\cbr{\CONFinst_t(u) \leq e^{L(u)}  - 1}
    \cdot
    \1 \cbr{\CONFtot_t(u) \leq t\cdot L(u)}
\end{align}
Parameters $\tbeta_\tau$ must be well-defined for all $\tau\in[0,T]$ and satisfy the following, for any rounds $t<t'$:
\begin{align}\label{ass:tilde-beta}
\textstyle
\{\; \text{$\tbeta_\tau$ decreases in $\tau$} \;\}
\;\text{and}\;
\{\; \tbeta_t\geq \beta_t \;\}
\;\text{and}\;
 \int_t^{t'} \tbeta_\tau \ud \tau \leq \frac{1}{\beta_{t'}} - \frac{1}{\beta_t}.
\end{align}
We cannot obtain the third condition of~\refeq{ass:tilde-beta} with equality because parameters $\beta_t$ and $\tbeta_t$ depend on $|A_t|$,
and the latter is not related to $t$ with a closed form solution.
\section{Our Results}
\label{sec:results}

\xhdr{Running Time.}
The per-round running time of the algorithm is
    $\tildeO\rbr{T^{d/(d+2)}}$,
where $d = \covdim$. Indeed, given Lemma~\ref{lem:wu-rule}, in each round $t$ of the algorithm we only need to compute the weight $w_{t,\eta}(\cdot)$ for all active nodes and one specific $\eta =\eta_t$. This takes only $O(1)$ time per node (since we can maintain the total estimated reward
    $\sum_{\tau\in[t]} \hg_\tau(\act_\tau(u))$
separately).
So, the per-round running time is $O(|A_T|)$, which is at most
$\tildeO\rbr{T^{d/(d+2)}}$,  as we prove in Lemma~\ref{lem:upp-bound-num-nodes}. Moreover, we obtain an improved bound on $|A_T|$ (and hence on the running time) when $\advzoomdim<\covdim$ and the doubling constant $\DblC$ is $\text{polylog}(T)$, see \refeq{eq:nodes-zoom}.

\xhdr{Regret Bounds.}
Our regret bounds are broken into three steps. First, we state the ``raw'' regret bound in terms of the algorithm's parameters, with explicit assumptions thereon. Second, we tune the parameters and derive the ``intermediate'' regret bound of the form $\tildeO(\sqrt{T\,|A_T|})$. Third, we derive the ``final'' regret bound, upper-bounding $|A_T|$ in terms of $\advzoomdim$. For ease of presentation, we use failure probability $\delta = T^{-2}$; for any known $\delta>0$, regret scales as $\log \nicefrac{1}{\delta}$. The covering dimension is denoted $d$, for some constant multiplier $\gamma_0>0$ (we omit the $\log(\gamma_0)$ dependence). The precise definition of an inclusively $\eps$-optimal arm in the definition of $\advzoomdim$ is that
\begin{align}\label{eq:eps-optmal-defn}
\gap_t(\cdot) < 30 \, \eps \, \ln (T) \, \sqrt{d \, \ln \left( \DblC \cdot T \right)}
\quad \text{for some end-time $t> \eps^{-2}/9$.}
\end{align}

\begin{theorem}\label{thm:main-regr}
Assume the sequences $\{\eta_t\}$ and $\{\beta_t\}$ are decreasing in $t$, and satisfy
\begin{align}\label{eq:param-assns}
 \eta_t \leq \beta_t \leq \gamma_t/|A_t|
 \quad\text{and}\quad
 \eta_t \rbr{ 1 + \beta_t(1 + 4\log T) } \leq \gamma_t/|A_t|.
 \end{align}
With probability at least $1 - T^{-2}$, \advzoom satisfies
\begin{align}
R(T) &\leq \calO(\ln T)\, \rbr{
    \sqrt{d\,T }
    + \frac{1}{\beta_T}
    + \frac{\ln\rbr{\DblC \cdot |A_T|}}{\eta_T}
    + \textstyle \sum_{t \in [T]} \beta_t + \gamma_t\, \ln T}
    \label{eq:thm-raw}
\\&\leq
 \calO\rbr{\sqrt{T\,|A_T|}}\cdot  \ln^2(T)\;
    \sqrt{d \,\ln \rbr{T\,|A_T| } \ln\rbr{ \DblC \, |A_T|}}
    \quad \EqComment{tuning the parameters}
    \label{eq:thm-tuned}
\\&\leq
 \calO\rbr{T^\frac{z+1}{z+2}}
    \cdot
    \rbr{ d^{1/2} \rbr{\gamma\,\DblC}^{1/(z+2)} \ln^5 T },
\label{eq:thm-ZoomDim}
 \end{align}
where $d = \covdim$ and $z = \advzoomdim$ with multiplier $\gamma>0$.
The parameters in \eqref{eq:thm-tuned} are:
\begin{align*}
\beta_t = \tbeta_t = \eta_t
&=
\sqrt{2\ln \rbr{|A_t| \cdot T^3} \ln\rbr{\DblC \cdot |A_t|}}
\;\;/\;\; \sqrt{t\;|A_t|\; d \cdot \ln^2 T},
\\
\gamma_t &=  (2 + 4 \log T)\; |A_t|\cdot \beta_t.
\numberthis{\label{eq:tunings}}
\end{align*}
\end{theorem}

\begin{remark}\label{rem:relaxed-dim}
We can relax the definition of \advzoomdim so that \eqref{eq:dims-generic} needs to hold only for scales $\eps$ smaller than some threshold $\theta$. Then we obtain the regret bound in \eqref{eq:thm-ZoomDim} plus
    $\tildeO\rbr{\sqrt{T\,\theta^{-\covdim}}}$.
\end{remark}

\xhdr{Special cases.}
First, we argue that for stochastic rewards $\advzoomdim$ coincides with the zooming dimension $\zoomdim$ from prior work, up to a small change in the multiplier $\gamma$. (We specify the latter by putting it in the subscript.) The key is to relate each arm's stochastic gap to its adversarial gap.

\begin{restatable}{lem}{zoomdimTOadvzoomdim}\label{lm:zoomdim}
Consider an instance of Lipschitz bandits with stochastic rewards. For any $\gamma>0$, with probability at least $1 - \nicefrac{1}{T}$ it holds that:
\[     \zoomdim_{\gamma\cdot f} \leq \advzoomdim_{\gamma\cdot f}
    \leq \zoomdim_{\gamma},
    \quad\text{where $f = \left(O(\poly (d)\; \ln^3 T)\right)^{\log(\DblC)-\zoomdim_{\gamma}}$}.
\]
\end{restatable}

\noindent This lemma holds for any representative set $\reprset$. Then the base in factor $f$ scales with $\ln(|\reprset|)$.

Second, for problem instances with $K<\infty$ arms, we recover the standard $\tildeO(\sqrt{KT})$ regret bound by observing that any problem instance has $\advzoomdim=0$ with multiplier $\gamma = K$ and $\reprset = [K]$. \advzoom satisfies
    $R(T) \leq \calO(\; \sqrt{K T} \cdot \sqrt{\DblC }\cdot \ln^5 T \;)$
    w.h.p.

Third, we analyze the dependence on the Lipschitz constant. Fix a problem instance, and multiply the metric by some $L>1$. The Lipschitz condition \eqref{eq:Lip} still holds, and the definition of \advzoomdim implies that regret scales as $L^{z/(z+2)}$. This is optimal in the worst case by prior work.%
\footnote{For a formal statement, consider the unit cube with $\arms = [0,1]^d$ and metric $D(x,y) = L\cdot \|x-y\|_p$, for some constants $d\in\N$ and $p\geq 1$. Then the worst-case optimal regret rate is
    $\tilde{\calO}\rbr{ L^{d/(d+2)}\cdot T^{(d+1)/(d+2)}}$.
The proof for $d=1$ can be found, \eg in Ch. 4.1 of
\cite{slivkins-MABbook}; the proof for $d>1$ can be derived similarly.}

\begin{corollary}\label{cor:L}
Fix a problem instance and a multiplier $\gamma>0$, and let $R_\gamma(T)$ denote the right-hand side of \eqref{eq:thm-ZoomDim}. Consider a modified problem instance with metric $\dist' = L\cdot \dist$, for some $L\geq 1$. Then \advzoom satisfies
    $R(T)\leq L^{z/(z+2)}\cdot R_\gamma(T)$,
with probability at least $1 - T^{-2}$.
\end{corollary}

\section{Regret Analysis (Outline)}\label{sec:analysis}

We outline the key steps and the proof structure; the lengthy details are in the %
next section. For ease of presentation, we focus on the $d$-dimensional unit cube
$(\arms,\dist) = ([0,1]^d,\,\ell_\infty)$.

We start with some formalities. First, we posit a \emph{representative arm} $\repr_t(u)\in u$ for each tree node $u$ and each round $t$, so that $x_t = \repr_t(U_t)$. W.l.o.g., all representative arms are chosen before round $1$. Thus, we can endow $u$ with rewards
    $g_t(u):= g_t(\repr_t(u))$.
Second, let
    $\opt_S(u) \in \arg \max_{x \in u} \sum_{t\in S} g_t(x)$
be the best arm in $u$ over the set $S$ of rounds (ties broken arbitrarily). Let $\opt_S = \opt_S(\arms)$ be the best arm over $S$. Let $\ust_t$ be the active node at round $t$ which contains %
$\opt_{[t]}$.

The representative set $\reprset\subset\arms$ (used in the definition of $\advzoomdim$) consists of arms
    $\repr(u)$, $\xst_{[t]}(u)$
for all tree nodes of height at most $1+\log T$ and all rounds $t$. Only these arms are invoked by the algorithm or the analysis. This enables us to transition to deterministic rewards that satisfy a certain ``per-realization" Lipschitz property (\refeq{eq:lm:prob} in the appendix).

\xhdr{Part I: Properties of the Zoom-In Rule.}
This part depends on the zoom-in rule, but not on the selection rule, \ie it works no matter how distribution $\pi_t$ is chosen. First, the zoom-in rule ensures that all active nodes satisfy the following property, called the \emph{zooming invariant}:
\begin{align}\label{eq:sketch-zoom-inv}
\CONFtot_t(u) \geq (t-1)\cdot L(u)
\quad\text{if node $u$ is active in round $t$ }
\end{align}

It is proved by induction on $t$, using the fact that when a node does \emph{not} get zoomed-in, this is because either instantaneous or the aggregate zoom-in rule does not apply.

Let us characterize the \emph{lifespan} of node $u$: the time interval $[\tau_0(u),\,\tau_1(u)]$ during which the node is active. We lower-bound the deactivation time, using the instantaneous zoom-in rule:
\begin{align}\label{eq:sketch-deactivation-time}
\text{node $u$ is zoomed-in}
\quad\Rightarrow\quad
\tau_1(u)\geq 1/L(u).
\end{align}
It follows that only nodes of diameter $L(\cdot) \geq \nicefrac{1}{2T}$ can be activated. Next, we show that a node's deactivation time is (approx.) at least twice as the parent's:
\begin{align}\label{eq:sketch-life-span}
\text{node $u$ is zoomed-in}
\quad\Rightarrow\quad
\tau_1(u) \geq 2\,\tau_1(\term{parent}(u))-2.
\end{align}
We use this to argue that a node's own datapoints eventually drown out those inherited from the parent when the node was activated. Specifically:
\begin{align}\label{eq:sketch-inherited-bias}
\text{node $u$ is active at time $t$}
\quad\Rightarrow\quad
\textstyle \tfrac{1}{t}\;\sum_{\tau \in [t]}\; L(\act_\tau(u))
    \leq 4\, \log (T)\cdot L(u).
\end{align}

Next, we prove that the total probability mass spent on a zoomed-in node must be large:
\begin{align}\label{eq:sketch-prob-mass}
\text{node $u$ is zoomed-in}
\quad\Rightarrow\quad
\mass(u) := \textstyle \sum_{\tau=\tau_0(u)}^{\tau_1(u)}\pi_\tau(u) \geq  \frac{1}{9\,L^2(u)}
\end{align}
This statement is essential for bounding the number of active nodes in Part IV. To prove it, we apply both the zooming invariant \eqref{eq:sketch-zoom-inv} and the (aggregate) zooming rule. Finally, the instantaneous zoom-in rule implies that the zoomed-in node is chosen with large probability:
\begin{align}\label{eq:sketch-large-prob}
\text{node $u$ is zoomed-in at round $t$ }
\quad\Rightarrow\quad
\pi_t(u) / \pi_t(\ust_t) \geq \beta^2_t/e^{L(u)}.
\end{align}

\xhdr{Part II: Multiplicative Weights.}
This part depends on the selection rule, but not on the zooming rule: it works regardless of how $z_t(u)$ is defined.  We analyze the following potential function:
    $\Phi_t(\eta) = \rbr{ \frac{1}{|A_t|} \sum_{u \in A_t}
        w_{t+1,\eta}(u)}^{1/\eta}$,
where $w_{t+1,\eta}(u)$ is given by \eqref{eq:weight-characterization},
with $\Phi_0(\cdot) = 1$.

We upper- and the lower-bound the telescoping product
\begin{align*}%
Q := \ln \left( \frac{\Phi_{T}(\eta_T)}{\Phi_0(\eta_0)}\right)
= \ln \left( \prod_{t=1}^T
\frac{\Phi_t(\eta_t)}{\Phi_{t-1}(\eta_{t-1})}\right)
= \sum_{t\in[T]} Q_t,
\;\text{where}\;
Q_t = \ln \rbr{ \frac{\Phi_t(\eta_t)}{\Phi_{t-1}(\eta_{t-1})} }.
\end{align*}
We lower-bound $Q$ in terms of the ``best node'' $\ust_T$, accounting for the ancestors via $\childprod(\cdot)$:
\begin{align}\label{eq:sketch-pot-lb0}
Q \geq \textstyle
\sum_{t\in[T]} \; \hg_{t}(\act_t(\ust_T)) -
\ln \left(|A_T| \cdot \childprod(\ust_T)\right) / \eta_T.
\end{align}
For the upper bound, we focus on the $Q_t$ terms. We transition from potential $\Phi_{t-1}(\eta_t)$ to $\Phi_{t}(\eta_t)$ in two steps: first, the weights of all currently active nodes get updated, and then we zoom-in on the appropriate nodes. The former is handled using standard techniques, and the latter relies on the fact that the weights are preserved. We obtain:
\begin{align}\label{eq:sketch-upp-main}
Q \leq \textstyle
\sum_{t\in[T]} g_t(\vx_t) +
\sum_{t\in[T]} O(\ln T)\;
    \rbr{\gamma_t + \beta_t \sum_{u \in A_t} \hg_t(u)}.
\end{align}

\xhdr{Part III: from Estimated to Realized Rewards.}
We argue about realized rewards, with probability (sat) at least $1-\nicefrac{1}{T}$. We bring in two more pieces of the overall puzzle: a Lipschitz property and a concentration bound for IPS estimators. If node $u$ is active at time $t$, then
\begin{align}\label{eq:sketch-lipschitz-prop}
\textstyle \sum_{\tau \in [t]} g_\tau(\opt_{[t]}(u))  - \sum_{\tau \in [t]} L(\act_\tau(u)) - 4\sqrt{td} \, \ln T \leq \sum_{\tau \in [t]} g_\tau(\act_\tau(u)).
\end{align}
(We only use Lipschitzness through \eqref{eq:sketch-lipschitz-prop}.)
For any subsets $A_\tau' \subseteq A_\tau$, $\tau\in [T]$ it holds that:
\begin{align}\label{eq:sketch-high-prob}
\textstyle
\left|
    \sum_{\tau \in [t],\; u \in A'_\tau} g_\tau(u) - \ips_\tau(u)
\right|
\leq
    O(\ln T)/\beta_t +
    \sum_{\tau \in [t],\; u \in A'_\tau}
        \beta_{\tau}/\pi_\tau(u).
\end{align}
The analysis of \ExpThreeP derives a special case of \eqref{eq:sketch-high-prob} with $A'_\tau = \{u\}$ in all rounds $\tau$.  The stronger version relies on \emph{negative association} between random variables
$\hg_\tau(u), u \in A'_\tau$.

Putting these two properties together, we relate estimates $\hg_t(u)$ with the actual gains $g_t(u)$. First, we argue that we do not \emph{over}-estimate by too much: fixing round $t$,
\begin{align}\label{eq:sketch-ghat-overestimate}
\textstyle \sum_{\tau \in [t],\; u \in A'_\tau}
    \beta_{\tau} \rbr{\hg_\tau(u) - g_\tau(u)}
 \leq O(\ln T)\rbr{ 1 + \sum_{\tau\in [t]} \beta_{\tau}\,  |A'_\tau| }.
\end{align}
This holds for any subsets $A'_\tau\subset A_\tau$, $\tau\in[t]$ which only contain ancestors of the nodes in $A'_t$.

Second, we need a stronger version for a singleton node $u$, one with $L(u)$ on the right-hand side. If node $u$ is zoomed-in in round $t$, then for each arm $y \in u$ we have:
\begin{align}\label{eq:sketch-ghat-zoomed}
\textstyle
\sum_{\tau\in [t]}\; \hg_\tau(\act_\tau(u)) - g_\tau(y)
&\leq \textstyle
  O\rbr{ L(u)\cdot t\,\ln(T) +  \sqrt{t\, d}\, \ln T + (\ln T)/\beta_t}.
\end{align}
Third, we argue that the estimates $\hg_t(u)$ form an approximate upper bound. We only need this property for singleton nodes: for each node $u$ which is active at round $t$, we have
\begin{align}\label{eq:sketch-ghat-underestimate}
\textstyle
\sum_{\tau\in [t]}\; \hg_\tau(\act_\tau(u))
-
        g_\tau\left(\opt_{[t]}(u)\right)
    &\geq - O\rbr{ \sqrt{t \, d}\, \ln T + (\ln T)/\beta_t}. %
\end{align}
To prove \eqref{eq:sketch-ghat-underestimate}, we also use the zooming invariant \eqref{eq:sketch-zoom-inv} and the bound \eqref{eq:sketch-inherited-bias} on inherited diameters.

Using these lemmas in conjunction with the upper/lower bounds of $Q$ we can derive the ``raw'' regret bound \eqref{eq:thm-raw}, and subsequently the ``tuned'' version \eqref{eq:thm-tuned}.

\xhdr{Part IV: the Final Regret Bound.} We bound $|A_T|$ to derive the final regret bound in Theorem~\ref{thm:main-regr}. First, use the probability mass bound \eqref{eq:sketch-prob-mass} to bound $|A_T|$ in the worst case. We use an ``adversarial activation'' argument: given the rewards, what would an adversary do to activate as many nodes as possible, if it were only constrained by \eqref{eq:sketch-prob-mass}? The adversary would go through the nodes in the order of decreasing diameter $L(\cdot)$, and activate them until the total probability mass exceeds $T$. The number of active nodes with diameter $L(u)\in[\eps,2\eps]$, denoted  $\activeN(\eps)$, is bounded via $\covdim$.

Second, we bound $\gap_t(\cdot)$. Plugging probabilities $\pi_t$ into \eqref{eq:sketch-large-prob}, bound the ``estimated gap",
\begin{align}\label{eq:sketch-estimated-gap}
\textstyle
\sum_{\tau \in [t]} \hg_\tau(\act_\tau(\ust_t)) - \sum_{\tau \in [t]} \hg_\tau(\act_\tau(u)) \leq
\ln \left(\frac{9\childprod(\ust_t)}{\childprod(u)\cdot\beta_t^2}\right)/\eta_t.
\end{align}
for a node $u$ which is zoomed-in at round $t$. To translate this to the actual $\gap_t(\cdot)$, we bring in the machinery from Part III and the worst-case bound on $|A_T|$ derived above.
\begin{align}\label{eq:sketch-adv-gap}
\gap_t(\repr(u)) \leq L(u)
    \cdot \calO\rbr{\ln (T) \sqrt{d \ln \left( \DblC \cdot T \right)}}.
\end{align}
We can now upper-bound $\activeN(\eps)$ via \advzoomdim rather than $\covdim$. With this, we run another ``adversarial activation'' argument to upper-bound $|A_T|$ in terms of $\advzoomdim$.

\section{\advzoomdim under Stochastic Rewards}
\label{app:advzoomdim}
We study $\advzoomdim$ under stochastic rewards, and upper-bound it by $\zoomdim$, thus proving Lemma~\ref{lm:zoomdim}. We prove this lemma in a slightly more general formulation, with an explicit dependence on the representative set $\reprset\subset \arms$.

\begin{lemma}\label{lm:zoomdim-app}
Consider an instance of Lipschitz bandits with stochastic rewards.
Fix a representative set $\reprset\subset \arms$. 
For any $\gamma>0$, with probability at least $1 - \nicefrac{1}{T}$ it holds that:
\begin{align}\label{eq:lm:zoomdim-app}
 \zoomdim_{\gamma\cdot f} \leq \advzoomdim_{\gamma\cdot f}
    \leq \zoomdim_{\gamma},
\end{align}
where $f = \rbr{O \left(\sqrt{d} \cdot \ln^2(T) \cdot \ln \left(|\reprset|  \right)\right)}^{\log(\DblC)-\zoomdim_{\gamma}}$.
\end{lemma}

As before, we use $d$ to denote the covering dimension, for some constant multiplier $\gamma_0>0$, and we suppress the logarithmic dependence on $\gamma_0$. 

To prove Lemma~\ref{lm:zoomdim-app}, we relate the stochastic and adversarial gap of each arm.

\begin{proposition}\label{prop:application-azuma}
Consider an instance of Lipschitz bandits with stochastic rewards. Fix time $t$. For any arm $x \in \reprset$, with probability at least $1 - \nicefrac{1}{T}$ it holds that:
\begin{align*}
\left| \gap_t(x) - \gapIID(x) \right| \leq 3 \;\sqrt{\frac{2\ln \left(T \cdot \left| \reprset\right|\right)}{t}}.
\end{align*}
\end{proposition}

\begin{proof}
Let us first fix an arm $x \in \reprset$. We apply the Azuma-Hoeffding inequality (Lemma~\ref{lem:azuma}) to the following martingale:
\begin{align*}
Y_t &= \sum_{\tau=1}^t \left( g_\tau(\vxst) - g_\tau(\vx) \right) - t \cdot \gapIID(x)
\end{align*}
where $\xst = \arg \max_{x \in \calA} \mu(x)$, and $\mu(x)$ is the mean reward for arm $x$ in the stochastic instance.

Noting that $|Y_{t+1} - Y_t| \leq 1$, and fixing $\delta>0$, we have
$\Pr \sbr{ Y_t \geq \sqrt{2t \cdot \ln (1/\delta)}} \leq \delta$.

Unwrapping the definition of $Y_t$, with probability at least $1-\delta$ we have:
\begin{align*}
\sum_{\tau=1}^t \left( g_\tau(\vxst) - g_\tau(\vx) \right) \leq \gapIID(x) + \sqrt{\frac{2\ln(1/\delta)}{t}} \numberthis{\label{eq:martingale-bound}}
\end{align*}
We next lower bound the left-hand side of \refeq{eq:martingale-bound}.
\begin{align*}
\frac{1}{t} \sum_{\tau=1}^t \left( g_\tau(\vxst) - g_\tau(\vx) \right) &= \frac{1}{t} \sum_{\tau=1}^t \left[ g_\tau(\vxst) - g_\tau(\vx) + g_\tau \left( \xst_{[t]} \right) - g_\tau \left( \xst_{[t]} \right) \right] \\
&= \gap_t(x) + \underbrace{\frac{1}{t} \sum_{\tau=1}^t g_\tau (\xst) - \frac{1}{t} \sum_{\tau=1}^t g_\tau \left( \xst_{[t]} \right)}_{\Delta}
\end{align*}
We move on to lower bounding quantity $\Delta$.
\begin{align*}
\Delta &\geq \sum_{\tau = 1}^t \E \left[g_\tau (\xst) \right] - \sqrt{\frac{2 \ln \left( \nicefrac{1}{\delta} \right)}{t}} - \sum_{\tau=1}^t \E\left[ g_\tau \left( \xst_{[t]} \right) \right] - \sqrt{\frac{2 \ln \left( \nicefrac{1}{\delta} \right)}{t}} \geq -2 \sqrt{\frac{2 \ln \left( \nicefrac{1}{\delta}\right)}{t}} \numberthis{\label{eq:Delta-bound}}
\end{align*}
where the last inequality comes from the fact that $\xst$ is the mean-optimal arm. Substituting~\refeq{eq:Delta-bound} in~\refeq{eq:martingale-bound} we obtain:
\begin{equation}\label{eq:martingale-bound2}
\gap_t(x) \leq \gapIID(x) + 3 \sqrt{\frac{2 \ln \left( \nicefrac{1}{\delta} \right)}{t}}
\end{equation}
Similarly, using the symmetric side of the Azuma-Hoeffding inequality (Lemma~\ref{lem:azuma}) for martingale $Y_t$, we obtain that with probability at least $1 - \delta$:
\begin{align*}
\gapIID(x)  &\leq \frac{1}{t} \sum_{\tau=1}^t \left( g_\tau(\vxst) - g_\tau(\vx) \right)  + \sqrt{\frac{2 \ln (\nicefrac{1}{\delta})}{t}} \\
            &\leq \frac{1}{t} \sum_{\tau=1}^t \left( g_\tau \left(\vxst_{[t]} \right) - g_\tau(\vx) \right)  + \sqrt{\frac{2 \ln (\nicefrac{1}{\delta})}{t}} \tag{$\xst_{[t]} = \arg \max_{x \in \calA} \sum_{\tau \in [t]} g_\tau (x)$}\\
            &= \gap_t(x) + \sqrt{\frac{2 \ln (\nicefrac{1}{\delta})}{t}} \numberthis{\label{eq:symmetric}}
\end{align*}
where the equality comes from the definition of $\gap_t(x)$. Putting~\refeq{eq:martingale-bound2} and~\refeq{eq:symmetric} together we get that \emph{for the fixed arm} $x \in \reprset$:
\begin{equation}\label{eq:both-sides-bef-ub}
\Pr \left[\left| \gap_t(x) - \gapIID(x) \right| \right] \leq 3 \sqrt{\frac{2 \ln (\nicefrac{1}{\delta})}{t}}
\end{equation}
In order to guarantee that~\refeq{eq:both-sides-bef-ub} applies \emph{for all} arms $y \in \reprset$, we apply the union bound; let $\delta'$ be the failure probability. Tuning the failure probability of the original event to be $\delta = \frac{1}{T \cdot |\reprset|}$ (\ie the failure probability of the final event is $\delta'=1/T$) we get the stated result.
\end{proof}

\proof{Lemma~\ref{lm:zoomdim-app}} We first prove the rightmost inequality in \eqref{eq:lm:zoomdim-app}. We fix an instance of the stochastic Lipschitz MAB, and we focus on an arm $x$ for which
    $\gap_t(x) \leq 30 \ln (T) \cdot \sqrt{d \ln \left( \DblC \cdot T \right)} \cdot \eps$, and $\eps = (3 \sqrt{t})^{-1}$. Then, from Proposition~\ref{prop:application-azuma} we obtain that with probability at least $1 -1/T$:
\begin{align*}
\gapIID(x) &\leq 30 \ln (T) \cdot \sqrt{d \ln \left( \DblC \cdot T \right)} \cdot \eps+  \eps \sqrt{18 \ln \left( T \cdot |\reprset|\right)} \\
&\leq 31\cdot \sqrt{d} \cdot \ln (T) \cdot \sqrt{\ln \left(\DblC \cdot T \right) \cdot \ln \left( T \cdot |\reprset| \right)} \cdot \eps
\end{align*}
From the definition of $\zoomdim$, the set of the aforementioned arms can be covered by \[
\gamma \cdot \underbrace{31\cdot \sqrt{d} \cdot \ln (T) \cdot \sqrt{\ln \left(\DblC \cdot T \right) \cdot \ln \left( T \cdot |\reprset| \right)} \cdot \eps
}_{\eps'}^{-\zoomdim_\gamma}
\] sets of diameter $\eps'$ for some constant $\gamma > 0$. Equivalently, the set of these arms can be covered with
\begin{align*}
&\gamma \cdot \left( \frac{\eps'}{\eps} \right)^{\log (\DblC)} \left( 31\cdot \sqrt{d} \cdot \ln (T) \cdot \sqrt{\ln \left(\DblC \cdot T \right) \cdot \ln \left( T \cdot |\reprset| \right)}\right)^{-\zoomdim_\gamma} \cdot \eps^{-\zoomdim_\gamma}\\
= &\gamma \cdot \underbrace{\left( 31\cdot \sqrt{d} \cdot \ln (T) \cdot \sqrt{\ln \left(\DblC \cdot T \right) \cdot \ln \left( T \cdot |\reprset| \right)}\right)^{\log (\DblC) -\zoomdim_\gamma}}_f \cdot \eps^{-\zoomdim_\gamma}
\end{align*}
sets of diameter $\eps$. Since we defined $\advzoomdim$ to be the infimum dimension for which the above holds, $\advzoomdim_{\gamma \cdot f} \leq \zoomdim_\gamma$. In order to conclude the proof and derive the stated expression for $f$, note that $\DblC \leq T$ and $\ln (T \cdot |\reprset|) \leq \ln (T) \cdot \ln \left( |\reprset|\right)$. This concludes the proof of this side of the inequality.

We proceed to prove the leftmost inequality
\asedit{in \eqref{eq:lm:zoomdim-app}.}
We fix an instance of the stochastic Lipschitz MAB, and we focus on an arm $x$ for which $\gapIID(x) \leq \eps$, where $\eps = (3 \sqrt{t})^{-1}$. Then, from Proposition~\ref{prop:application-azuma} we get that with probability at least $1 - \nicefrac{1}{T}$ it holds that:
\begin{equation}
\gap_t(x) \leq \eps + 3\eps\sqrt{18\ln T} \leq 30 \ln (T) \cdot \sqrt{d \cdot \ln (\DblC \cdot T)}\cdot \eps
\end{equation}
From the definition of $\advzoomdim_{\gamma \cdot f}$, the set of these arms can be covered by $\gamma \cdot f \cdot \eps^{-\advzoomdim_{\gamma \cdot f}}$ sets of diameter $\eps$ for some constant $\gamma \cdot f > 0$. Since we defined $\zoomdim$ to be the infimum dimension for which the above holds: $\zoomdim_{\gamma \cdot f} \leq \advzoomdim_{\gamma \cdot f}$.
\endproof

\noindent In order to get the result stated in Lemma~\ref{lm:zoomdim}, we note that $|\reprset| \leq T^{O(\poly (d))}$.

\section{$\advzoomdim$ Examples}
\label{app:examples}

We provide a flexible ``template" for examples with small $\advzoomdim$. We instantiate this template for some concrete examples, which apply generically to adversarial Lipschitz bandits as well as to a more specific problem of adversarial dynamic pricing.

\begin{theorem}\label{thm:example}
Fix action space $(\calA,\calD)$ and time horizon $T$. Let $d$ be the covering dimension.
\footnote{As before, the covering dimension is with some constant multiplier $\gamma_0>0$, and we suppress the logarithmic dependence on $\gamma_0$.}

Consider problem instances $\calI_1 \LDOTS \calI_M$ with stochastic rewards, for some $M$. Suppose each $\calI_i$ has a constant zooming dimension $z$, with some fixed multiplier  $\gamma>0$. Construct the \emph{combined instance}: an instance with adversarial rewards, where each round is assigned in advance (but otherwise arbitrarily) to one of these stochastic instances.

Then $\advzoomdim\leq z$ with probability at least $1-\nicefrac{1}{T}$, with multiplier
    \[\gamma' = \gamma\cdot \rbr{\calO \left( M \ln (T) \sqrt{d \ln \left( \DblC \cdot T \right) \cdot \ln \left( T \cdot |\reprset| \right)}\right)}^{\log (\DblC) - z}. \]
The representative set $\reprset\subset\arms$ (needed to specify $\advzoomdim$) can be arbitrary.

This holds under the following assumptions on problem instances $\calI_i$:

\begin{itemize}
\item There are disjoint subsets $S_1 \LDOTS S_M\subset \arms$ such that each stochastic instance $\calI_i$, $i\in [M]$ assigns the same ``baseline" mean reward $b_i$ to all arms in $\cup_{j\neq i} S_j$, mean rewards \emph{at least} $b_i$ to all arms inside $S_i$, and mean rewards \emph{at most} $b_i$ to all arms in
    $\arms\setminus \sup_{j\in [M]} S_j$.

\item For each stochastic instance $\calI_i$, $i\in [M]$,  the difference between the largest mean reward and $b_i$ (called the \emph{spread}) is at least $\nicefrac{1}{3}$.
\end{itemize}

\end{theorem}

The proof of the theorem can be found in Appendix~\ref{app:examples-proof}.

We emphasize the generality of this theorem. First, the assignment of rounds in the combined instance to the stochastic instances $\calI_1 \LDOTS \calI_M$ can be completely arbitrary: \eg the stochastic instances can appear consecutively in ``phases" of arbitrary duration, or they can be interleaved in an arbitrary way. Second, the subsets $S_1 \LDOTS S_M\subset \arms$ can be arbitrary. Third, each stochastic instance $\calI_i$, $i\in [M]$ can behave arbitrarily on $S_i$, as long as
    $\must_i-b'_i\geq \nicefrac{1}{3}$,
where $\must_i$ and $b'_i$ are, resp., the largest and the smallest rewards on $S_i$. The baseline reward can be any $b_i \leq b'_i$, and outside
    $\cup_{j\in [M]} S_j$
one can have any mean rewards that are smaller than $b_i$.

Now we can take examples from stochastic Lipschitz bandits and convert them to (rather general) examples for adversarial Lipschitz bandits. Rather than attempt a compehensive survey of examples for the stochastic case, we focus on two concrete examples that we adapt from \citep{LipschitzMAB-JACM}: concave rewards and ``distance to the target". For both examples, we posit action space $\arms = [0,1]$ and distances $\dist(x,y) = |x-y|$. Note that the covering dimension is $d=1$. The expected reward of each arm $x$ in a given stochastic instance $i$ is denoted $\mu_i(x)$. In both examples, $\mu_i(\cdot)$ will have a single peak, denoted $\xst_i\in\arms$, and the baseline reward satisfies
    $\mu_i(x^*_i)-b_i\geq \nicefrac{1}{3}$.

\begin{itemize}
\item \emph{Concave rewards:} For each instance $i$, $\mu_i(x)$ is a strongly concave function on $S_i$, in the sense that $\mu_i''(x)$ exists and $\mu_i''(x)<\eps$ for some $\eps>0$. Then the zooming dimension is $z=\nicefrac{1}{2}<d=1$, with appropriately chosen multiplier $\gamma>0$.
    \footnote{A somewhat subtle point: an algorithm tailored to concave-rewards instances can achieve $\tilde{\calO}(\sqrt{T})$ regret, \eg via uniform discretization \citep{KleinbergL03}. However, this algorithm would not be optimal in the worst case: it would only achieve regret $\tilde{\calO}(T^{3/4})$ whereas the worst-case optimal regret rate is $\tilde{\calO}(T^{2/3})$.}

\item \emph{``Distance to target":} For each instance $i$,
    $\mu_i(x) = \min\rbr{0, \mu_i(\xst_i) - \dist(x,\xst_i)}$
    for all arms $x\in S_i$.
Then the zooming dimension is in fact $z=0$.

\end{itemize}

\newcommand{\bst}{v^{\star}}
\section{Applications: Dynamic Pricing and Auction Tuning}
\label{sec:applications}

We apply our machinery to four applications in dynamic pricing and auction tuning. We define these applications in a common framework, and proceed to address them one by one. While these applications do not inherently satisfy Lipschitz-continuity, we show that our algorithm and analysis carry over without additional Lipschitz assumptions.

\subsection{Problem Formulations and a Common Framing}
\label{sec:apps-defns}

We consider three applications under a common framing of an adversarial bandit problem with action set $[0,1]$. (Technically, they generalize one another, in the order of presentation.) Here, an algorithm is a seller with unlimited supply of identical items. In each round $t$, the algorithm chooses a price $x_t\in [0,1]$, offers one item for sale, and observes its reward: the revenue from the sale, if any. The algorithm faces some new customer(s), whose behavior is determined by some \emph{private value(s)}, \ie values not known to the algorithm, whose meaning is spelled out below. All private values are chosen in advance by an adversary, possibly using randomization. In the stochastic version of each application, the tuple of private values for each round is drawn independently from a fixed distribution which is not known to the algorithm.

\begin{enumerate}

\item The basic application is \emph{dynamic pricing}: a new customer arrives at each round $t$, with private value $v_t\in[0,1]$, and buys the item if and only if $x_t\leq v_t$. So, the algorithm's reward is $g_t(x_t) = x_t \cdot \1 \{x_t \leq v_t\}$. The stochastic version of the problem is well-studied (see Sec.~\ref{sec:rel-work}).

\item Next, we consider a repeated second-price auction with reserve price $x_t$ in each round $t$. Here, a fresh batch of customers arrives in each round $t$, each with a private value in $[0,1]$ that is fixed in advance.%
\footnote{Equivalently, customers may be long-lived, but their values are exogenously selected anew in each round (and determined in advance), and their behavior is myopic, \ie the only optimize for the current round.}
The customers bid truthfully, \ie report their values.
The algorithm's reward is
$g_t(x_t) = \max(x_t, \pi_t)\cdot \1 \cbr{x_t \leq \bst_t}$,
where $\bst_t, \pi_t$ are the largest and second-largest private values in this round respectively. We assume that the algorithm does not observe  individual bids; this is justified when the actual auction is conducted by an intermediary platform which provides minimal feedback to the seller. The stochastic variant of this application has been studied in \citet{RepeatedAuctions-soda13}.

\item Third, we consider an arbitrary truthful single-parameter auction for divisible goods (\eg divisible advertising opportunities), and ``tune it" with reserve price $x_t$ in each round $t$. In each round, each customer submits a bid to the auction, without seeing the other bids, and receives some amount of the good. We start with an arbitrary single-parameter allocation rule $\calR$, subject to a monotonicity condition, and create a new allocation rule $\calR_t$ by ignoring all bids less than $x_t$. The truth-inducing payments are then determined by a certain integral of $\calR_t$ as a function of the respective bid \citep{Myerson} (as spelled out in Section~\ref{sec:apps-divisible}). The rest of the setup is defined as before, including truthful bidding.
\end{enumerate}

The fourth application concerns dynamic pricing with $d$ products. The algorithm is a seller with infinite supply of each product. At each round $t$, the algorithm posts a separate price $x_{i,t}\in[0,1]$ for each product $i\in[d]$. We assume \emph{single-item demands}: each customer only buys one item of a single product (if any). A new customer arrives in each round $t$, with a private value $v_{i,t}\in[0,1]$ for each product $i$, and chooses a product $i$ that maximizes utility $v_{i,t}-x_{i,t}$, as long as this utility is non-negative. At the end of the round, the algorithm only learns the product purchased (if any).

We emphasize that in all four applications, the algorithm does not observe the customers' private values, neither before nor after the corresponding round.

Thus, in all four applications the action set is $[0,1]^d$, for some $d\in\N$. We take the metric to be $\dist = \ell_\infty$, like in Section~\ref{sec:algo}. So, \advzoom applies as specified.

We focus on a certain Lipschitz-like property which, as we show, is inherently satisfied in all four applications. To state this property, recall that a given tree node $v\in\nodes$ in \advzoom corresponds to a $d$-dimensional hypercube, $L(v)$ is its edge-length, $\repr(v)$ is its representative arm, and $\xst_t(v)$ is its best arm at time $t$. The property is as follows:
\begin{align}\label{eq:Lip-apps}
    g_t(\xst_t(v)) - g_t(\repr(v)) \leq c_0 \cdot L(v)
    \quad\text{for all rounds $t$, nodes $v\in\nodes$},
\end{align}
for some constant $c_0\geq 1$. This property is sufficient for our analysis, essentially because it implies \refeq{eq:sketch-lipschitz-prop} (which \emph{is} sufficient). The formal statement is as follows.

\begin{theorem}\label{thm:apps}
Consider an adversarial bandit problem with action set $[0,1]^d$, $d\in\N$ which satisfies condition \eqref{eq:Lip-apps} for some constant $c_0\geq 1$ and a particular choice of representative arms $\repr(\cdot)$. Consider \advzoom with these representative arms and parameters as per \refeq{eq:tunings}. This algorithm attains regret
\begin{align}\label{eq:regret-apps}
R(T) \leq \tcalO(T^{(z+1)/(z+2)}), \text{ where } z = \advzoomdim.
\end{align}
(Here, $\tcalO(\cdot)$ hides the dependence on $d$ and $c_0$, as well as a polylogarithmic dependence on $T$.)
\end{theorem}

\begin{proof}[Proof Sketch]
For $c_0=1$, we obtain \refeq{eq:sketch-lipschitz-prop} by summing up over rounds $\tau \leq t$ and applying \eqref{eq:Lip-apps} for each round $\tau$ and the corresponding active ancestor $v=\act_\tau(u)$ of node $u$.  Thus, Theorem~\ref{thm:main-regr} applies as stated, and in turn implies \refeq{eq:regret-apps} for the parameters in \refeq{eq:tunings}.

To handle $c_0>1$, our analysis carries through with $L'(v) = c_0\cdot L(v)$ instead $L(v)$, only affecting the constant in $\tcalO()$ in the final regret bound.
\end{proof}

\subsection{Adversarial Dynamic Pricing}
\label{sec:DP}

The algorithm's rewards for adversarial dynamic pricing satisfy the one-sided Lipschitz condition \eqref{eq:Lip-DP} (with $c_0=1$) because selling at a given price $x$ implies selling at any lower price $x'$. This suffices to guarantee \refeq{eq:Lip-apps} and consequently (by Theorem~\ref{thm:apps}) the regret bound \eqref{eq:regret-apps}. We present this point in a generic way, to make it suitable to two other applications:

\begin{corollary}\label{cor:CAB}
Consider an adversarial bandit problem with action set $[0,1]$ which satisfies the one-sided Lipschitz condition \eqref{eq:Lip-DP}. Consider \advzoom with parameters as per \refeq{eq:tunings}, where the representative arm $\repr_t(u)$ is the lowest point of the price interval that node $u$ corresponds to. This algorithm attains regret as in \refeq{eq:regret-apps}.
\end{corollary}

\begin{proof}
\refeq{eq:Lip-DP} implies \refeq{eq:Lip-apps} for this choice of representative arms.
\end{proof}

\begin{corollary}\label{cor:DP}
For adversarial dynamic pricing, the version of \advzoom from Corollary~\ref{cor:CAB} attains regret \refeq{eq:regret-apps}, without any additional assumptions.
\end{corollary}

\xhdr{Examples for small \advzoomdim.} Both examples from Section~\ref{app:examples} can be ``implemented" in dynamic pricing. It suffices to consider stochastic dynamic pricing with  appropriate distributions over the private values. In fact, the concave-rewards case is known as \emph{regular demands}, a very common assumption in theoretical economics. Below are some concrete instantiations for the two examples.

Note that by definition, the mean reward for dynamic pricing takes the following form:
\[
\mu_i(x) = x \cdot \Pr_{v_i} [x \leq v_i] = x \cdot (1 - \Pr_{v_i} [ v_i < x ])
\]
As a result, in order to provide concrete instantiations, one only needs to control the distribution of $v_i$, \ie term $\Pr_{v_i} [v_i < x]$. At a high level, the distribution of $v_i$'s expresses the preferences of the buyers' population.

For the concave-rewards instantiation, if $v_i \sim \text{Unif}[0,1]$, then $\mu_i(x) = x (1 - x) = x - x^2$ the mean-reward $\mu_i(x)$ satisfies the strong concavity requirements, and hence, the zooming dimension of instance $\calI_i$ is $z = \nicefrac{1}{2}$. In this example, $S_i$ can be any subset of $[0,1]$. In order to obtain disjoint sets, for the different instances, one needs to use different uniform distributions $\text{Unif}[a_j, b_j]$ for different instances $\{ j \in [M]\}$.

As for the distance-to-target instantiation, if $v_i$ is drawn from distribution with cdf
\begin{equation}\label{eq:cdf-target}
F(x) = 1 - \frac{1}{4x} + \frac{\left|x - \nicefrac{1}{2} \right|}{x}, x \in [0.3, 0.7] = S_i
\end{equation}
Then, $\mu_i(x) = \frac{1}{4} - \left|x - \frac{1}{2} \right| = \must_i - \calD(x, \xst_i)$. In order to construct disjoint sets $\{S_i\}_{i \in [M]}$ one needs to use the general form of~\refeq{eq:cdf-target} as follows:
\[
F(x) = 1 - \frac{a_i}{x} + \frac{\left|x - b_i \right|}{x}, x \in S_i
\]
where $a_i = \must_i$ and $b_i = \xst_i$.

\subsection{Repeated Second-Price Auction}

We consider a simplified problem formulation which is mathematically equivalent to the one in Section~\ref{sec:apps-defns}.  Specifically: we have a bandit problem with action set $[0,1]$ and round-$t$ reward functions given by
    $g_t(x) = \max(x, \pi_t)\cdot \1 \cbr{x \leq \bst_t}$,
where $0\leq \pi_t\leq \bst_t\leq 1$. The parameters $\bst_t,\pi_t$, $t\in [T]$ are chosen by an adversary before round $1$.

\begin{lemma}
Reward functions $g_t(x)$ satisfy one-sided Lipschitzness \eqref{eq:Lip-DP} (with $c_0=1$).
\end{lemma}

\begin{proof}
For ease of notation, remove $t$ from the notation: focus on function $g:[0,1]\to [0,1]$ given by
    $g(x) = \max(x, \pi)\cdot \1 \cbr{x \leq \bst}$,
for some fixed parameters $0\leq \pi\leq \bst\leq 1$.

Let $x > x'$ and both $x, x' \in [0,1]$. We distinguish two cases. For the first case, we assume that $\bst > x$ (i.e., a sale happens at price $x$). For the second case, we assume that $\bst < x$ (i.e., no sale happens at price $x$).

\xhdr{Case 1.} Since $x > x'$, this means that $\bst > x'$. Next, we need to distinguish two cases based on what the payment for the buyer is: (i) $\pi > x$ (i.e., the buyer pays $\pi$); and (ii) $\pi < x$, in which case the buyer pays $x$. For (i), since $x > x'$, then $\pi > x'$, and hence: $g(x) - g(x') = \pi - \pi = 0 < x - x'$. For (ii), if $\pi < x'$, then at price $x'$ the buyer needs to pay $x'$. Hence: $g(x) - g(x') = x - x'$. If $\pi > x'$, then at price $x'$ the buyer needs to pay $\pi$. As a result: $g(x) - g(x') = x - \pi < x - x'$.

So far, we have proven that $g(\cdot)$ is one-sided Lipschitz under case 1.

\xhdr{Case 2.} This case needs to also consider two subcases. In the first one, $\bst \leq x'$. If that is the case, then there is no sale either at price $x$ or price $x'$. As a result: $g(x) - g(x') = 0 < x - x'$, where the last inequality is due to the assumption that $x > x'$. For the second subcase: $\bst > x'$; this means that a purchase does happen at price $x'$. In this case we have to distinguish cases based on the payment for the buyer: if $\pi < x'$, then $g(x') = x'$ and $g(x) - g(x') = 0 - x' < 0 < x - x'$. If $\pi > x'$, then: $g(x') = \pi$ and $g(x) - g(x') = 0 - \pi < 0 < x - x'$.
\end{proof}

\begin{corollary}
Consider reserve price tuning in repeated second price auctions, as defined above. The version of \advzoom from Corollary~\ref{cor:CAB} attains regret \refeq{eq:regret-apps}.
\end{corollary}

\subsection{Repeated Truthful Single-Parameter Auction for Divisible Goods}
\label{sec:apps-divisible}

Let us introduce the necessary notation and define the payments. Let $n$ be the number of bidders. Let $\calR_{i}(b;x_t)$ be the bidder-$i$ allocation in round $t$ under bid vector $b$ and reserve price $x_t$, \ie the amount of good allocated to this bidder. Recall that it equals $\calR_i(b)$, the bidder-$i$ allocation under the original allocation rule $\calR$, if all bids less than $x_t$ are ignored (and the respective bidders receive $0$ allocation). We use a standard convention
    $b = (b_i,\, b_{-i})$,
where $b_i$ is the bid of agent $i$ and $b_{-i}$ is the bid tuple of everyone else.

The payments that induce truth-telling are defined as follows, as per \citet{Myerson}. For each bidder $i$ and round $t$, the payment under bid vector $b = (b_i,\, b_{-i})$ is
\begin{align}
\calP_{i}(b;x_t)
    = b_i\cdot \calR_{i}(b;x_t)
        - \int_0^{b_i} \calR_{i}(\rho,\, b_{-i};x_t)\; \mathrm{d}\rho.
\end{align}

Fix some round $t$. Let $v_{i,t}\in [0,1]$ be the private value of bidder $i$ in this round. Since we have defined a truthful auction, we assume truthful bidding, \ie that bidder $i$ submits bid $v_{i,t}$. Let $v_t$ be the value vector. The reward function is therefore given by
\[
g_t(x_t) = \sum_{i \in [n]} \calP_{i}(v_t;x_t).
\]

\begin{lemma}
Reward functions $g_t(x)$ satisfy one-sided Lipschitzness \eqref{eq:Lip-DP}, {with $c_0=n$}.
\end{lemma}

\begin{proof}
Let prices $x, x'$ such that $x > x'$. Then,
\begin{equation}\label{eq:main-diff}
    g_t(x) - g_t(x') = \sum_{i \in [n]}  v_i\cdot \left(\calR_{i}(v;x) -\calR_{i}(v;x')\right)
        - \int_0^{v_i}
        \rbr{\calR_{i}(\rho,\, v_{-i};x) - \calR_{i}(\rho,\, v_{-i};x')} \mathrm{d}\rho.
\end{equation}

Fix a bidder $i \in [n]$. We separate the following cases based on whether the bidder gets any allocation based on the different reserve prices.
\begin{itemize}
    \item If $v_i \geq x$ (\ie $\calR_i(v;x) = \calR_i(v)$) then it is also the case that $v_i > x'$ (\ie $\calR_i(v;x') = \calR_i(v)$); hence,
    \[\calP_i(v_i;x) - \calP_i(v_i;x')
    = - \int_0^{v_i} \rbr{\calR_{i}(\rho,\, v_{-i};x) - \calR_{i}(\rho,\, v_{-i};x')} \mathrm{d}\rho.
    \]
    Note that the curves $\calR_{i}(\rho,\, v_{-i};x)$ and $\calR_{i}(\rho,\, v_{-i};x')$ are identical for $\rho \geq x$. For $\rho \in [x', x]$ however $\calR_{i}(\rho,\, v_{-i};x) = 0$ (since the bidder's allocation will be $0$, as he does not clear the reserve price of $x$), while $\calR_{i}(\rho,\, v_{-i};x')$ is (potentially) greater than $0$ since $v_i > x'$. The difference in the integral is hence upper bounded by $1 \cdot (x - x')$.
    \item If $v_i \leq x' < x$, then $\calP_i(v_i;x) - \calP_i(v_i;x') = 0$ (\ie bidder $i$ does not get any allocation, and does not have to pay anything).
    \item If $x' < v_i < x$, then bidder $i$ gets (potentially) non-zero allocation under reserve price $x'$ and zero allocation (and hence, also payment) under reserve price $x$. As a result,
    \[
    \calP_i(v_i;x) - \calP_i(v_i; x') = -\calP_i(v_i; x') \leq 0
    \]
\end{itemize}
Every bidder $i$ falls in one of the three categories analyzed above. Putting everything together, it holds that for $x > x'$: $g_t(x) - g_t(x') \leq n \cdot (x - x')$, hence satisfying one-sided Lipschitzness.
\end{proof}

\begin{corollary}
Consider reserve price tuning in a repeated truthful single-parameter auction, as defined above. The version of \advzoom from Corollary~\ref{cor:CAB} attains regret \refeq{eq:regret-apps}.
\end{corollary}

\subsection{Multi-Product Dynamic Pricing with Single-Demand Buyers}
\label{sec:apps-multiDP}

Note that the round-$t$ reward function $g_t:[0,1]^d\to[0,1]$ is given by
\[ g_t(x) := x_{i,t} \cdot \1 \{ v_{i,t} \geq x_{i,t} \},
\text{ where }
i = i_t := \max_{j \in [d]} v_{j,t} - x_{j,t}.
\]

Contrary to the applications we have seen so far, here we do \emph{not} have one-sided Lipschitzness. However, we have a type of ``monotonicity'' which makes \refeq{eq:sketch-lipschitz-prop} hold, given an appropriate choice of a representative arm per node $u$. As mentioned previously, \refeq{eq:sketch-lipschitz-prop} suffices for our analysis.

The representative arm $\repr(u)$ for a given node $u$ {with edge-length $L(u)$} is defined as follows.
Order the prices corresponding to its corners in non-decreasing order, as $x_{(1)} \leq x_{(2)} \leq \dots \leq x_{(d)}$.
Now, define
\begin{align}\label{eq:multiDP-repr}
\repr(u) = (\repr_{(1)}(u), \dots, \repr_{(d)}(u)),
\text{ where }
\repr_{(i)}(u) = x_{(i)} - (i-1)\cdot L(u).
\end{align}

\begin{lemma}\label{lem:d-prod}
For each round $t$ and any node $u$ it holds that
 \begin{align}\label{eq:lem:d-prod}
    g_t(\xst(u)) - g_t(\repr(u)) \leq d \cdot L(u).
\end{align}
\end{lemma}

\begin{corollary}
Consider multi-product dynamic pricing with single-demand buyers. {By Theorem~\ref{thm:apps},}
\advzoom with representative arms defined in \refeq{eq:multiDP-repr} and parameters as per \refeq{eq:tunings} yields regret in \refeq{eq:regret-apps}.
\end{corollary}

For clarity of presentation, in the remainder of this section. we focus on a particular round $t$ and drop the ``$t$'' dependence.

Before we present the full proof of the lemma, we present a useful intermediate result the relates prices $x_{(i)}(u)$ and $x_i(u)$ for a node $u$ and product $i \in [d]$.

\begin{lemma}\label{lem:intermediate-d-prod}
    Fix a node $u$. Then, for all $i \in [d]$ it holds that: $x_{(i)}(u) - x_i (u) \leq (i - 1) \cdot L(u)$.
\end{lemma}

\begin{proof}
    We will prove this claim via induction. For the base case, note that $x_{(1)}(u) - x_1(u) \leq 0$ holds because $x_{(1)}(u)$ is the minimum of all $x_{i}(u), i \in [d]$. For the inductive hypothesis, assume that for $i = n$: $x_{(n)}(u) - x_n(u) \leq (n-1)L(u)$. For the inductive step, we need to show that $x_{(n+1)}(u) - x_{n+1}(u) \leq n L(u)$.
    \begin{align*}
        x_{(n+1)}(u) - x_{n+1}(u) &= x_{(n+1)}(u) - x_{(n)}(u) + x_{(n)}(u)- x_{n+1}(u) \tag{$\pm x_{(n)}(u)$} \\
        &\leq x_{(n)}(u) - x_{n+1}(u) \tag{$x_{(n+1)}(u) \leq x_{(n)}(u)$ by definition} \\
        &= x_{(n)}(u) - x_n(u) + x_n(u) - x_{n+1}(u) \\
        &\leq (n-1)\cdot L(u) + L(u) = n L(u)
    \end{align*}
    where the last inequality uses the inductive hypothesis and the fact that $x_i(u) - x_j(u) \leq L(u)$ for any two $i,j \leq d$. This concludes the proof of the lemma.
\end{proof}

We are now ready to present the proof of Lemma~\ref{lem:d-prod}.
\begin{proof}[Proof of Lemma~\ref{lem:d-prod}]
    Our proof is split in 2 cases. In the first case, assume that there is no sale at price vector $\repr(u)$. We are going to show that this means that there is also no sale at price vector $\xst(u)$, and hence the lemma statement is satisfied by default. Indeed, if there is no sale at $\repr(u)$, then it means that $v_i < \repr_{(i)}(u)$ for all items $i \in [d]$. Since $\repr{(i)}(u) = x_{(i)}(u) - (i-1) L(u)$ then $v_i < x_{(i)}(u) - (i-1) L(u)$ for all $i \in [d]$. Using Lemma~\ref{lem:intermediate-d-prod}: $v_i < x_i$ for all $i \in [d]$. Hence, there is no sale at $\xst(u)$.

    In the second case, assume that price vector $\repr(u)$ resulted in the buyer buying item $\jst$, i.e., $v_{\jst} > \repr_{(\jst)}(u)$ and $\arg \max_{i \in [d]} v_i - \repr_{(i)}(u) = \jst$. Let $\ist$ be the item with $\max v_i - x_i$. Note that since there was a sale under price vector $\repr(u)$, then it must be the case that there would be a sale under price vector $\xst(u)$. To see this, assume the contrary, i.e., that $0 > v_{\ist} - x_{\ist}$ and that $\ist \neq \jst$. Then:
    \begin{align*}
        0 > v_{\ist} - x_{\ist} \geq v_{\jst} - x_{\jst} \stackrel{\text{Lem.}~\ref{lem:intermediate-d-prod}}{\geq} v_{\jst} - \underbrace{\left[x_{(\jst)} - (\jst - 1)L(u)\right]}_{\repr_{\jst}(u)}
    \end{align*}
    which is contradicting the fact that $v_{\jst} > \repr_{(\jst)}(u)$.

    Hence, $\ist$ also led to a sale. Now, assume that $\jst \neq \ist$, since if $\jst = \ist$, then the statement of the lemma is satisfied by default. It holds that:
    \begin{align*}
        g(\xst(u)) - g(\repr(u)) &= v_{\ist} - \xst_{\ist} - v_{\jst} + \repr_{\jst}(u) \\
        &\leq v_{\ist} - \xst_{\ist} - v_{\ist} + \repr_{\ist}(u) \tag{by definition of $\jst$} \\
        &\leq (\ist - 1) \cdot L(u) \tag{Lemma~\ref{lem:intermediate-d-prod}}\\
        &< d L(u)
    \end{align*}
    This concludes our proof.
\end{proof}

\section{Conclusion}\label{sec:conclusion}

In this paper we have introduced \advzoom, an algorithm for adaptive discretization in adversarial Lipschitz bandits and proved its instance-dependent regret bounds. Prior to our work, the literature on adaptive discretization for Lipschitz bandits had only tackled the \emph{stochastic} reward settings. Importantly, our algorithm recovers the worst-case optimal regret bound for the adversarial version, and the instance-dependent bound for the stochastic version while being agnostic to the type of rewards it is facing. We have also shown how our algorithm has direct applications in many settings of interest where exact Lipschitzness is violated: e.g., dynamic pricing, dynamic pricing with reserve or posted prices, and multi-item pricing for single-demand buyers. 

There are several exciting research directions stemming from our work; we highlight two in this discussion. The first important direction is to extend the \advzoom framework to \emph{contextual} bandit settings, where the learner observes side information or context before choosing an action. Many practical problems (e.g., personalized pricing or recommendations) require algorithms that can leverage contextual data to make more informed decisions. In the adversarial setting, this extension is particularly challenging, as the reward functions may vary arbitrarily across rounds and contexts. The goal here would be to develop a contextual version of \advzoom that performs adaptive partitioning over the joint action-context space and exploits Lipschitz continuity with respect to a suitable metric defined on this product space. The second most interesting direction for future work lies in improving the practical applicability of \advzoom. While \advzoom enjoys strong theoretical guarantees, its computational complexity in large or continuous action spaces remains a challenge; indeed, in the worst-case, it matches the complexity of uniform discretization in $d$-dimensional spaces. One approach here would be to develop efficient heuristics that (in the general case) do not lose much of the regret performance that \advzoom enjoys. Another approach would be to adapt techniques from~\citet{zhulipschitz} to optimally balance space and regret considerations. Both of these directions require significantly new techniques and are thus beyond the scope of our current work.

More broadly, we consider our work as a stepping stone for learning in settings that lie between stochastic and adversarial assumptions.

\newpage
\addcontentsline{toc}{section}{References}
\bibliography{refs,bib-abbrv,bib-slivkins,bib-bandits,bib-AGT,bib-nodeLabeling,bib-embedding,bib-RL,bib-contextual_pricing}

\newpage

\appendix

\section{Probability Tools}
\label{sec:prelims}
We state the Markov Inequality and the Azuma-Hoeffding Inequality, standard tools that we use.

\begin{lemma}[Markov Inequality]
If $\varphi$ is a monotonically increasing non-negative function for the non-negative reals, $X$ is a random variable, $a \geq 0$ and $\phi(a) > 0$, then:
\begin{align*}
\Pr \left[|X| \geq a \right] \leq \frac{\E \left[\varphi(X) \right]}{\varphi(a)}
\end{align*}
\end{lemma}

\begin{lemma}[Azuma-Hoeffding Inequality]\label{lem:azuma}
Suppose $\{X_k: k= 0,1,2,\dots\}$ is a martingale and $|X_k - X_{k-1}| \leq c_k$. Then, for all positive integers $N$ and all $\eps>0$ we have that: \[ \Pr\left[\left|X_N - X_0 \right| \geq \eps \right] \leq 2\exp\left(-\frac{\eps^2}{2\sum_{k=1}^N c_k^2}\right). \]
\end{lemma}

Also, we use the in-expectation Lipschitz condition~(\refeq{eq:Lip}) to derive a high-probability, per-realization version, which is the version directly used by our analysis. It is a simple corollary of the Azuma-Hoeffding inequality, independent of the rest of the analysis.

\begin{lemma}[Per-realization Lipschitz property]
\label{lm:prob}
Fix round $t$, two sequences of arms
    $(y_1 \LDOTS y_t)$ and $(y'_1 \LDOTS y'_t)$,
and failure probability $\delta>0$. 
Then with probability at least $1 - \delta$ we have:
\begin{align}\label{eq:lm:prob}
\sum_{\tau \in [t]} g_\tau(y_\tau) - \sum_{\tau \in [t]} g_\tau(y_\tau')
    \leq 2\sqrt{2t \ln (2/\delta)}+
        \sum_{\tau \in [t]} \dist(y_\tau,y_\tau').
\end{align}

\end{lemma}

\begin{proof}
We apply Azuma-Hoeffding inequality to martingale
    $Y_t = \sum_{\tau \in [t]} g_\tau(y_\tau) - \E[ g_\tau(y_\tau) ]$.
We conclude that
$\Pr\sbr{ | Y_t | \geq \sqrt{2t \ln (2/\delta)} }\leq \delta$.
In other words, with probability at least $1 - \delta$,
\begin{equation*}%
\left| \sum_{\tau=1}^t g_\tau(y_\tau) - \sum_{\tau=1}^t \E \left[ g_\tau(y_\tau) \right] \right| \leq \sqrt{2t \ln (2/\delta)}.
\end{equation*}
A similar inequality holds for sequence $(y'_1 \LDOTS y'_t)$. Combining
both inequalities with \refeq{eq:Lip} gives the stated result.
\end{proof}

\newpage

\section{Implementation: Multiplicative Update (Proof of Lemma~\ref{lem:wu-rule})}
\label{app:MWU}

We prove Lemma~\ref{lem:wu-rule}, which is required for a computationally efficient implementation of the algorithm, and is used heavily throughout the analysis. We restate this lemma for completeness.

\weightupdate*

\begin{proof}
We prove the lemma by using induction on the zooming tree at round $t+1$, specifically on the path from node $\root$ from $u$. For the base case, if the set of active nodes at round $t+1$ included only node $\root$ then from \refeq{eq:alg:MW}:
\begin{align*}
w_{t+1, \eta}( \root) = w_{t, \eta}(\root) \exp\left( \eta \hg_t(\root) \right) = \exp\left(\eta \sum_{\tau \in [t]} \hg_\tau(\root) \right).
\end{align*}
So \refeq{eq:weight-characterization} holds, since $\childprod(\root) = 1$, as $\root$ is the root node of the tree (i.e., it has no ancestors). We next assume that the active node at round $t+1$ is $\xi$ and that \refeq{eq:weight-characterization} holds for the zooming tree until node $\xi$, such that $\xi = \parent(u)$, i.e.,:
\begin{equation}\label{eq:parent-u}
w_{t+1,\eta}(\xi) = \frac{1}{\childprod(\xi)} \cdot \exp \left(\eta \sum_{\tau \in [t]} \hg_\tau(\act_\tau(\xi)) \right).
\end{equation}
Next assume that the active node at round $t+1$ is node $u$. Then, from \refeq{eq:weight-characterization}
\begin{align*}
w_{t+1, \eta}(u) = w_{t,\eta}(u) \exp \left(\eta \hg_t(u)\right) = w_{t,\eta}(\act_t(u)) \exp \left(\eta \hg_t(\act_t(u))\right)
\end{align*}
since by definition for all the rounds $\tau \leq t+1$ during which $u$ is active $\act_\tau(u) = u$. By definition, from round $\tau_0(u)$ until round $t+1$ node $u$ has not been zoomed-in. Hence, by the weight-update rule for nodes that are not further zoomed-in we have that:
\begin{align*}
w_{t+1,\eta}(u) &= w_{\tau_0(u)-1,\eta}(u) \cdot \prod_{\tau = \tau_0(u)}^t \exp \left(\eta \hg_\tau(\act_\tau(u))\right) \\
& = w_{\tau_0(u)-1,\eta}(u) \cdot \exp \left(\eta \sum_{\tau = \tau_0(u)}^{t}\hg_\tau(\act_\tau(u))\right). \numberthis{\label{eq:weight-update}}
\end{align*}
Since $\tau_0(u)-1$ is the last round of $\xi$'s lifetime, it is the round that $\xi$ got zoomed-in, and the weight of $u$ was initialized to be $1/|\c(\xi)|$ the weight of $\xi$. Hence, \refeq{eq:weight-update} becomes:
\begin{align*}
w_{t+1,\eta}(u) &= \frac{1}{|\c(\xi)|} w_{\tau_1(\xi), \eta}(\xi) \cdot \exp \left(\eta \sum_{\tau = \tau_0(u)}^{t}\hg_\tau(\act_\tau(u))\right)
\end{align*}
But for the rounds where node $\xi$ was active, \refeq{eq:parent-u} was true for node $\xi$. Since $\xi = \parent(u)$, then $\act_\tau(\xi) = \act_\tau(u), \forall \tau \leq \tau_0(u)$. Hence, using \refeq{eq:parent-u} in the latter, we obtain:
\begin{align*}
w_{t+1,\eta}(u) &= \frac{1}{|\c(\xi)| \cdot \childprod(\xi)} \cdot \exp \left(\eta \sum_{\tau \in [\tau_1(\xi)]} \hg_\tau(\act_\tau(u)) \right) \cdot \exp \left(\eta \sum_{\tau = \tau_0(u)}^{t} \hg_\tau(\act_\tau(u)) \right) \\
&= \frac{1}{\childprod(u)} \cdot \exp \left(\eta \sum_{\tau \in [t]} \hg_\tau(\act_\tau(u)) \right).
\end{align*}
where for the penultimate equation, we used the definition of $\childprod(u)$.
\end{proof}

\section{Regret Analysis (Details)}
\label{sec:analysis-details}

We analyze \advzoom and prove the main theorem (Theorem~\ref{thm:main-regr}). We follow the proof sketch in Section~\ref{sec:analysis}, with subsections corresponding to the ``parts'' of the proof sketch. While Theorem~\ref{thm:main-regr} makes one blanket assumption on the parameters, we explicitly spell out which assumptions are needed for which lemma.

In what follows, we use $d$ to denote the covering dimension, for some constant multiplier $\gamma_0>0$. The dependence on $\gamma_0$ is only logarithmic, we suppress it for clarity. 

\xhdr{From randomized to deterministic rewards.}
Define the \emph{representative set} of arms as
\begin{align}\label{eq:repr}
    \reprset := \cbr{ \repr(u),\, \xst_{[t]}(u):\;
        \text{tree nodes $u$ with $h(u)\leq 1+\log T$, rounds $t\in[T]$}}.
\end{align}
Only tree nodes of height at most $1+\log(T)$ can be activated by the algorithm (as per Lemma~\ref{lem:height-node}), so $\reprset$ contains all arms that can possibly be pulled. Note that $|\reprset| \leq (T+1)\,T^d$.

A \emph{canonical arm-sequence} is a sequence of arms $y_1 \LDOTS y_t \in \reprset$, for some round $t$, which contains at most $\log T$ switches. As it happens, we will only invoke Lemma~\ref{lm:prob} on canonical arm-sequences. Since there are at most  $|\reprset|^{\log T}$ such sequences, we can take a Union Bound over all of them. Formally, we define the \emph{clean event} for rewards, denoted $\rewE$, which asserts that \refeq{eq:lm:prob} in Lemma~\ref{lm:prob} holds with
    $\delta = T^{-2-d\log T}$
for all rounds $t$ and all canonical arm-sequences
    $(y_1 \LDOTS y_t)$, $(y'_1 \LDOTS y'_t)$.
Lemma~\ref{lm:prob} implies that $\Pr\sbr{\rewE}\geq 1-\nicefrac{1}{T}$.

From here on, we condition on $\rewE$ without further mention. Put differently, we assume that rewards are determinsitic and satisfy $\rewE$. Any remaining randomness is due to the algorithm's random seed.

\subsection{Properties of the Zoom-In Rule}
\label{app:zoom-rule}

{This part of the analysis depends on the zoom-in rule, but not on the selection rule, \ie it works no matter how distribution $\pi_t$ is chosen.}

We start by defining the \emph{zooming invariant} which holds for \emph{all} active nodes. The zooming invariant is a property of the confidence that we have on the currently active nodes, and it is proved inductively using the fact that when a node does not get zoomed-in, then either the instantaneous or the aggregate rules are not satisfied (\refeq{eq:zoom-in-rule} in Section~\ref{sec:algo}).

\begin{restatable}[Zooming Invariant]{lemma}{zoominginvariant}\label{lem:zooming-invariant}
If node $u$ is active at round $t$, then: $\CONFtot_t(u) \geq (t-1) \cdot L(u)$.
\end{restatable}

\begin{proof}%
Since $u$ is active at round $t$ then $t \in [\tau_0(u), \tau_1(u)]$. We first focus on rounds where $z_\tau(\cdot) = 0$. Since for all rounds $\tau \in [\tau_0(u), \tau_1(u))$ node $u$ was \emph{not} zoomed-in, it must have been because either the instantaneous or the aggregate rules (\refeq{eq:zoom-in-rule}) were not true. Assume first that the aggregate rule was not satisfied for rounds $\tau \in [\tau_0(u), t_1]$ such that $t_1 \leq t$. In other words, from \refeq{eq:zoom-in-rule} we have that:
\begin{equation}\label{eq:cumulative-false}
\sum_{\tau=1}^{t_1} \frac{\beta_\tau}{\pi_{\tau}(\act_\tau(u))} + \frac{1}{\beta_{t_1}} \geq  t_1 \cdot L(u)
\end{equation}
Note that if $t_1 = t$, then the lemma follows directly. Let $t_2 \geq t_1 + 1$ be the round such that for all rounds $\tau \in [t_1 + 1, t_2]$ the aggregate rule \emph{does} hold, but the instantaneous \emph{does not}. Hence, for all such rounds it holds that:
\begin{equation}\label{eq:insta-false}
\frac{\beta_\tau}{\pi_{\tau}(u)} + \tbeta_\tau \geq e^{L(u)} - 1 \geq L(u)
\end{equation}
where the last inequality is due to the property that $e^x - 1 \geq x$. Summing up both sides of \refeq{eq:insta-false} for all rounds $\tau \in [t_1 +1, t_2]$ we get that:
\begin{equation}\label{eq:insta-summed-up}
\sum_{\tau = t_1 + 1}^{t_2} \frac{\beta_\tau}{\pi_{\tau}(u)} + \sum_{\tau = t_1+1}^{t_2} \tbeta_\tau \geq \sum_{\tau = t_1 + 1}^{t_2} L(u)
\end{equation}
Since $\tbeta_t$ is positive: $\sum_{\tau = t_1 + 1}^{t_2} \tbeta_t \leq \sum_{\tau = t_1}^{t_2} \tbeta_t$, and by the assumption on $\tbeta_t$ (\refeq{ass:tilde-beta} in Section~\ref{sec:algo}) we can relax the left hand side of \refeq{eq:insta-summed-up} and obtain:
\begin{equation}\label{eq:insta-summed-up-final}
\sum_{\tau = t_1 + 1}^{t_2} \frac{\beta_\tau}{\pi_{\tau}(u)} + \frac{1}{\beta_{t_2}} - \frac{1}{\beta_{t_1}} \geq (t_2 - t_1) L(u)
\end{equation}
Note that for all rounds $\tau \in [\tau_0(u), \tau_1(u)]$: $\act_\tau(u) = u$. As a result, summing up \refeq{eq:cumulative-false} and \refeq{eq:insta-summed-up-final} we have that:
\begin{equation*}
\sum_{\tau = 1}^{t_2} \frac{\beta_\tau}{\pi_{\tau}(\act_\tau(u))} + \frac{1}{\beta_{t_2}} \geq t_2 \cdot L(u)
\end{equation*}
Applying the same arguments for all rounds $\tau \in [1, t]$ we have that:
\begin{equation}\label{eq:inv-until-last-step}
\sum_{\tau = 1}^{t} \frac{\beta_\tau}{\pi_{\tau}(\act_\tau(u))} + \frac{1}{\beta_{\tau_1(u) - 1}} \geq  tL(u)
\end{equation}
which completes our proof if $t \leq \tau_1(u)-1$. In order to complete the proof we show what happens for the case that $t = \tau_1(u)$.
Note that when node $u$ gets zoomed-in at round $\tau_1(u)$, both the instantaneous and the aggregate rules hold (and $z_\tau(u) = 1$). So we have that:
\begin{align*}
\sum_{\tau = 1}^{\tau_1(u)} \frac{\beta_{\tau}}{\pi_{\tau}(u)} + \frac{1}{\beta_{\tau_1(u)}} &\geq \sum_{\tau = 1}^{\tau_1(u)-1} \frac{\beta_{\tau}}{\pi_{\tau}(\act_\tau(u))} + \frac{1}{\beta_{\tau_1(u)}} &\tag{$\beta_\tau, \pi_\tau(\cdot) > 0$}\\
&\geq \sum_{\tau = 1}^{\tau_1(u)-1} \frac{\beta_{\tau}}{\pi_{\tau}(\act_\tau(u))} + \frac{1}{\beta_{\tau_1(u)-1}} &\tag{$\beta_\tau \leq \beta_{\tau-1}$}\\
&\geq \left(\tau_1(u) - 1\right) L(u) &\tag{\refeq{eq:inv-until-last-step}} %
\end{align*}
This concludes our proof.
\end{proof}

We define the \inhbias as: $\b(u) = \tau_1(u)\cdot L(u)$, i.e., \inhbias is the upper bound on the ``total bias" suffered when we are at node $u$. The next lemma shows that the zooming tree cannot grow arbitrarily large as time goes by, by bounding the de-activation time of a node $u$ with the inverse of its diameter {(\refeq{eq:sketch-deactivation-time}%
)}; in fact, we show that the height of the zooming tree at any node $u$, denoted by $h(u)$, cannot be larger than $\log T$.

\begin{restatable}[Bound on Height of Node]{lemma}{nodeheight}\label{lem:height-node}
Assume that $\tbeta_\tau \geq \beta_\tau, \forall \tau$ and $\beta_\tau \geq 1/\tau$. Then, if node $u$ gets de-activated at round $\tau_1(u)$, it holds that $b(u)\geq 1$ and as a result, $h(u) \leq \log (\tau_1(u)) \leq \log T$.
\end{restatable}

\begin{proof}%
From the instantaneous rule in \refeq{eq:zoom-in-rule} for $z_t(u)$ we have that if we zoom-in at a node $u$ at round $\tau_1(u)$ the following must be true:
\begin{equation}\label{eq:zoom-lb-pi}
\frac{\beta_{\tau_1(u)}}{\pi_t(u)} + \tbeta_{{\tau_1(u)}} \leq e^{L(u)}-1 \Leftrightarrow \pi_{\tau_1(u)}(u) \geq \frac{\beta_{\tau_1(u)}}{e^{L(u)}-1 - \tbeta_{\tau_1(u)}}
\end{equation}
Because of the fact that $\pi_\tau(\cdot)$ is a valid probability distribution: $\pi_\tau(v) \leq 1, \forall v \in A_t$. This imposes the following restriction on the right hand side of \refeq{eq:zoom-lb-pi}:
\begin{align*}%
L(u) &\geq \ln \left(\beta_{\tau_1(u)} + \tbeta_{\tau_1(u)} + 1 \right) \geq \ln \left( 2\beta_{\tau_1(u)} + 1\right) \\
&\geq \frac{2\beta_{\tau_1(u)}}{2\beta_{\tau_1(u)} + 1} = \frac{1}{1 + \frac{1}{2\beta_{\tau_1(u)}}} &\tag{$\ln(1 + x) \geq \frac{x}{x+1}, x \geq -1$} \\
&\geq \frac{1}{\tau_1(u)} \numberthis{\label{eq:height}}
\end{align*}
where the last inequality is due to the fact that $\beta_\tau \geq 1/\tau, \forall \tau$ according to the assumptions of the lemma. Thus, $b(u) \geq 1$. Since every time that a node gets zoomed-in its diameter gets halved, we have that if node $u$ is found at height $h(u)$, then: $L(u) = L(u_0) 2^{-h(u)} = 2^{-h(u)}$ because $L(u_0) = 1$. Substituting this to \refeq{eq:height} and taking logarithms we get that $h(u) \leq \log (\tau_1(u)) \leq \log T$.
\end{proof}

Adding up to the point made earlier (i.e., that the action tree cannot grow arbitrarily large), in the next lemma we show that the lifespan of any node is \emph{strongly} correlated with the lifespan of its ancestors. To be more precise, we show that the de-activation time of a node $u$ is (approximately) at least \emph{twice} larger than its activation {(\refeq{eq:sketch-life-span}%
)}. An important implication of this is that once a node $u$ gets zoomed in at a round $t$ then it will not be possible to immediately zoom-in on any of its children, as we remarked on Section~\ref{sec:algo}.

\begin{restatable}[Lifespan of Node Compared to Ancestors]{lemma}{lifespan}\label{lem:lifespan}
Let $\{\beta_t\}_{t=1}^T$ be a non-increasing sequence. Then, if $u$ gets de-activated at round $\tau_1(u)$, it holds that $\tau_1(u) \geq 2 \tau_1(\xi) - 2$, where $\xi$ is $u$'s parent.
\end{restatable}

\begin{proof}%
Since $\xi$ is $u$'s parent, then the two nodes share the same ancestry tree. Hence, $\forall \tau \in [1, \tau_1(\xi)]: \act_\tau(\xi) = \act_\tau(u)$ and $\forall \tau \in [\tau_0(u), \tau_1(u)]: \act_\tau(u) = u$. For all rounds $\tau \in [\tau_0(\xi), \tau_1(\xi)]$ the zooming invariant holds for node $\xi$. Hence, for $t = \tau_1(\xi)$ and using the notation $b(\xi) = \tau_1(\xi) L(\xi)$:
\begin{equation*}
\sum_{\tau=1}^{\tau_1(\xi)} \frac{\beta_\tau}{\pi_{\tau}(\act_\tau(\xi))} + \frac{1}{\beta_{\tau_1(\xi)}} \geq \b(\xi) - L(\xi)
\end{equation*}
Since the sequence $\{\beta_t \}_{t=1}^T$ is (by assumption) non-increasing and $\tau_1(u) \geq \tau_1(\xi)$ we can relax the left hand side of the above inequality and get:
\begin{equation}\label{eq:invariant-v}
\sum_{\tau=1}^{\tau_1(\xi)} \frac{\beta_\tau}{\pi_{\tau}(\act_\tau(\xi))} + \frac{1}{\beta_{\tau_1(u)}} \geq \b(\xi) - L(\xi)
\end{equation}
By definition, on round $\tau_1(u)$ we decided to zoom-in, hence, the aggregate component of the zoom-in rule was true:
\begin{align*}
&\b(u) \geq \sum_{\tau=1}^{\tau_1(u)}\frac{\beta_\tau}{\pi_{\tau}(\act_\tau(u))} + \frac{1}{\beta_{\tau_1(u)}} \\
&= \sum_{\tau=1}^{\tau_1(\xi)}\frac{\beta_\tau}{\pi_{\tau}(\act_\tau(u))} + \sum_{\tau=\tau_0(u)}^{\tau_1(u)}\frac{\beta_\tau}{\pi_{\tau}(\act_\tau(u))}  + \frac{1}{\beta_{\tau_1(u)}} \\
&= \sum_{\tau=1}^{\tau_1(\xi)}\frac{\beta_\tau}{\pi_{\tau}(\act_\tau(\xi))} + \sum_{\tau=\tau_0(u)}^{\tau_1(u)}\frac{\beta_\tau}{\pi_{\tau}(\act_\tau(u))}  + \frac{1}{\beta_{\tau_1(u)}} &\tag{$\act_\tau(\xi) = \act_\tau(u), \forall \tau \leq \tau_1(\xi)$} \\
&\geq \b(\xi) - L(\xi) - \frac{1}{\beta_{\tau_1(\xi)}} + \sum_{\tau=\tau_0(u)}^{\tau_1(u)}\frac{\beta_\tau}{\pi_{\tau}(\act_\tau(u))}  + \frac{1}{\beta_{\tau_1(u)}} &\tag{zooming invariant for node $\xi$}\\
&\geq \b(\xi) - L(\xi) +\sum_{\tau=\tau_0(u)}^{\tau_1(u)} \tbeta_\tau + \sum_{\tau=\tau_0(u)}^{\tau_1(u)}\frac{\beta_\tau}{\pi_{\tau}(\act_\tau(u))}  &\tag{assumption on $\tbeta_t$}\\
&\geq \b(\xi) - L(\xi) + \sum_{\tau=\tau_0(u)}^{\tau_1(u)}\tbeta_\tau + \sum_{\tau=\tau_0(u)}^{\tau_1(u)}\beta_\tau  &\tag{$\pi_t(\cdot) \leq 1$}\\
&\geq \b(\xi) - L(\xi) \numberthis{\label{eq:inhbias}}
\end{align*}
where the last inequality is due to the fact that $\beta_\tau >0, \forall \tau$. To complete the proof, we use the definition of $\b(\cdot)$ along with the fact that $L(\xi) = 2 L(u)$ (since every time that a node gets de-activated its diameter gets halved) in \refeq{eq:inhbias}.
\end{proof}

{Next, we analyze the \emph{total inherited diameter} of a node $u$ in round $t$, defined as
    $\sum_{\tau \in [t]} L(\act_\tau(u))$.
We relate it to the ``total bias", as expressed by the node's diameter,  that the Lipschitz condition would have imposed on this node had it been active from round $1$ (this is \refeq{eq:sketch-inherited-bias} in Section~\ref{sec:analysis}).}

\begin{restatable}[Inherited Bias Bound]{lemma}{inhbias}\label{lem:inh-bias}
For a node $u$ that is active at round $t$, the total inherited diameter that it has suffered is upper bounded by: $\sum_{\tau \in [t]} L(\act_\tau(u)) \leq 4t \log (T) L(u)$.
\end{restatable}

\begin{proof}%
We first prove that for any node $u$ it holds that
\begin{equation}\label{eq:lifespan2}
\tau_1(u) \geq 2^{h(u)}\left(\tau_1(u) -2 \right)
\end{equation}
where $u'$ is any node in the path from $\root$ to node $u$. Indeed, from Lemma~\ref{lem:lifespan} and denoting by $v_i, i = \{1, \dots, h(u) \}, v_0 = \root$ the path from $\root$ to $u$ we have that:
\begin{align*}
\tau_1(u) &\geq 2 (\tau_1(v_{h(u)-1}) - 1) &\tag{Lemma~\ref{lem:lifespan} for node $u$}\\
&\geq 2 ( 2(\tau_1(v_{h(u)-2}) - 1) - 1) &\tag{Lemma~\ref{lem:lifespan} for node $v_{h(u)-1}$}\\
&\geq 2 ( 2( 2(\tau_1(v_{h(u)-3}) - 1) - 1) -1) &\tag{Lemma~\ref{lem:lifespan} for node $v_{h(u)-2}$}\\
&= 2^{h(u) - h(v_j)} \tau_1(v_j) - \left(1 + 2 + 4 + \dots  \right) \\
&= 2^{h(u) - h(v_j)} \tau_1(v_j) - 2^{h(u)-h(v_j)+1} \\
&= 2^{h(u) - h(u')}\left(\tau_1(u') -2 \right)
\end{align*}

For the ease of notation we denote node $v_{h(u)-1}$ as node $\xi$. Then, by the definition of total inherited diameter it holds that:
\begin{align*}
\sum_{\tau=1}^t &L(\act_\tau(u)) = \tau_1(v_0) L(v_0) + \left( \tau_1(v_1) - \tau_0(v_1) + 1 \right)L(v_1) + \dots + (t - \tau_0(u) +1) \cdot L(u) \\
&= \tau_1(v_0) L(v_0) + \left( \tau_1(v_1) - \tau_1(v_0)\right)L(v_1) + \dots + (t - \tau_1(\xi)) \cdot L(u) \\
&= \left[ \tau_1(v_0) \left( L(v_0) - L(v_1) \right) + \tau_1(v_1) \left( L(v_1) - L(v_2) \right) + \cdots \right] + t L(u) \\
&= \left[ \tau_1(v_0) \frac{L(v_0)}{2} + \tau_1(v_1) \frac{L(v_1)}{2} + \cdots \right] + t L(u) \\
&= \frac{1}{2} \left[ \tau_1(v_0) L(v_0) + \tau_1(v_1) L(v_1) + \cdots \right] + t L(u) \\
&= \frac{1}{2} \left[ \tau_1(v_0) \cdot 2^{h(\xi)} \cdot L(\xi) + \tau_1(v_1) \cdot 2^{h(\xi)-1} \cdot L(\xi) + \cdots \right] + t L(u) \\
&\leq \frac{1}{2} \left[\tau_1(\xi) \cdot L(\xi)\cdot h(\xi) + L(\xi) \left(2 + 4 + 8  \cdots \right) \right] + t L(u) &\tag{\refeq{eq:lifespan2}} \\
&\leq t L(u) h(u) + 2 L(u) \cdot 2^{h(\xi)+1} + t L(u) &\tag{$L(\xi) = 2L(u)$ and geometric series} \\
&\leq 2 t L(u) h(u) + 2 L(u) \cdot 2^{h(u)} \numberthis{\label{eq:before-inh-bias}}
\end{align*}
where the first equality is due to the fact that $\tau_1(v_j) + 1 = \tau_0(v_{j+1})$ and the fourth and fifth equalities is due to the fact that every time that we zoom-in the diameter of the parent node gets halved. From Lemma~\eqref{lem:height-node} we have that $h(u) \leq \log T$, and hence \refeq{eq:before-inh-bias} becomes: $\sum_{\tau=1}^t L(\act_\tau(u)) \leq 4 t L(u) \log T$.
\end{proof}

The next lemma shows that the probability mass that has been spent on a node from the round it gets activated until the round it gets de-activated is inversely proportional to the square of the diameter of the node {(\refeq{eq:sketch-prob-mass}%
)}. This property will be very important for arguing about the adversarial zooming dimension.

\begin{restatable}[Probability Mass Spent on A Node]{lemma}{probmass}\label{lem:prob-mass}
For a node $u$ that gets de-activated at round $\tau_1(u)$, the probability mass spent on it from its activation time until its de-activation, $\mass(u)$, is:
\[ \mass(u) = \sum_{\tau=\tau_0(u)}^{\tau_1(u)}\pi_\tau(u) \geq \frac{1}{9L^2(u)} \]
\end{restatable}

\begin{proof}%
Let node $\xi$ be node $u$'s parent. Since at round $\tau = \tau_1(u)$ we zoom-in on node $u$ then, from the aggregate zoom-in rule it holds that:
\begin{equation}\label{eq:pm1}
\sum_{\tau=1}^{\tau_1(\xi)} \frac{\beta_\tau}{\pi_{\tau}(\act_\tau(u))} + \sum_{\tau=\tau_0(u)}^{\tau_1(u)} \frac{\beta_\tau}{\pi_{\tau}(u)} + \frac{1}{\beta_{\tau_1(u)}} = \sum_{\tau=1}^{\tau_1(u)} \frac{\beta_\tau}{\pi_{\tau}(\act_\tau(u))} + \frac{1}{\beta_{\tau_1(u)}} \leq \b(u)
\end{equation}
where the first equality is due to the fact that for rounds $\tau \in [\tau_0(u), \tau_1(u)]$: $\act_\tau(u) = u$. Since nodes $u$ and $\xi$ share the same ancestors, for all rounds $\tau \leq \tau_1(\xi)$ it holds that: $\act_\tau(u) = \act_\tau(\xi)$, and \refeq{eq:pm1} can be rewritten as:
\begin{equation}\label{eq:pm2-0}
\sum_{\tau=1}^{\tau_1(\xi)} \frac{\beta_\tau}{\pi_{\tau}(\act_\tau(\xi))} + \sum_{\tau=\tau_0(u)}^{\tau_1(u)} \frac{\beta_\tau}{\pi_{\tau}(u)} + \frac{1}{\beta_{\tau_1(u)}} \leq \b(u)
\end{equation}
From the zooming invariant (Lemma~\ref{lem:zooming-invariant}) for round $\tau_1(\xi)$ we have that:
\begin{equation*}
\sum_{\tau=1}^{\tau_1(\xi)} \frac{\beta_\tau}{\pi_{\tau}(\act_\tau(\xi))} \geq \b(\xi) - L(\xi) - \frac{1}{\beta_{\tau_1(\xi)}}
\end{equation*}
Substituting the latter to \refeq{eq:pm2-0} we get that:
\begin{equation}\label{eq:pm2}
\sum_{\tau=\tau_0(u)}^{\tau_1(u)} \frac{\beta_\tau}{\pi_{\tau}(u)} + \frac{1}{\beta_{\tau_1(u)}} - \frac{1}{\beta_{\tau_1(\xi)}} \leq \b(u) - \b(\xi) + L(\xi)
\end{equation}
Since the sequence of $\beta_\tau$'s in non-increasing and $\tau_1(\xi) \leq \tau_1(u)$, then $1/\beta_{\tau_1(u)} - 1/\beta_{\tau_1(\xi)} \geq 0$ and also $\beta_\tau \geq \beta_{\tau_1(u)}, \forall \tau \leq \tau_1(u)$, so \refeq{eq:pm2} becomes:
\begin{equation}\label{eq:pm3}
\beta_{\tau_1(u)} \cdot \sum_{\tau=\tau_0(u)}^{\tau_1(u)} \frac{1}{\pi_{\tau}(u)} \leq \b(u)- \b(\xi) + L(\xi)
\end{equation}
From the properties of the harmonic mean it holds that $\frac{n}{\sum_{i=1}^t x_i^{-1}} \leq \frac{1}{n} \sum_{i=1}^n x_i$ and the above can be relaxed to:
\begin{equation}\label{eq:pm4}
\beta_{\tau_1(u)} \cdot \left(\tau_1(u) - \tau_1(\xi) \right)^2 \cdot \frac{1}{\mass(u)} \leq \b(u) - \b(\xi) + L(\xi)
\end{equation}
At round $\tau_1(\xi)$, the aggregate rule holds for node $\xi$, so:
\begin{equation}\label{eq:zoom-aggr-xi}
\sum_{\tau=1}^{\tau_1(\xi)} \frac{\beta_{\tau}}{\pi_{\tau}(\act_\tau(u))} + \frac{1}{\beta_{\tau_1(\xi)}} \leq \b(\xi)
\end{equation}
Since $\sum_{\tau=1}^{\tau_1(\xi)} \frac{\beta_{\tau}}{\pi_{\tau}(\act_\tau(u))} \geq 0$ the left hand side of the above becomes: $\frac{1}{\beta_{\tau_1(\xi)}} \leq \b(\xi)$. Since the sequence of $\beta_\tau$'s is non-increasing: $\beta_\tau \leq \beta_1 \leq 1/2$ thus we have: $\b(\xi) \geq 2 \geq L(\xi), \forall \xi$. This implies that the right hand side of \refeq{eq:pm4} can be relaxed and so:
\begin{equation}\label{eq:pm5}
\beta_{\tau_1(u)} \cdot \left(\tau_1(u) - \tau_1(\xi) \right)^2 \cdot \frac{1}{\mass(u)} \leq \b(u)
\end{equation}
From Lemma~\eqref{lem:lifespan}, we have that $\tau_1(u) \geq 2(\tau_1(\xi) - 1)$ and it also holds that $\tau_1(u)/2 - 1 \geq \tau_1(u)/3$. Combining this with \refeq{eq:pm5} we have that:
\begin{equation}\label{eq:pm6}
\frac{\beta_{\tau_1(u)} \cdot \tau_1(u)}{9 \cdot \mass(u)} \leq L(u) \Leftrightarrow \mass(u) \geq \frac{\beta_{\tau_1(u)} \cdot \tau_1(u)}{9 L(u)}
\end{equation}
In the next step, we will show that $\tau_1(u) \cdot \beta_{\tau_1(u)} \geq 1/L(u)$. At round $\tau_1(u)$ node $u$ gets zoomed-in, so the aggregate zoom-in rule holds for node $u$:
\begin{align*}
\b(u) = \tau_1(u) \cdot L(u) &\geq \sum_{\tau=1}^{\tau_1(u)} \frac{\beta_\tau}{\pi_\tau(\act_\tau(u))} + \frac{1}{\beta_{\tau_1(u)}} \\
&\geq \frac{1}{\beta_{\tau_1(u)}} &\tag{$\beta_\tau, \pi_\tau(\cdot) > 0$}
\end{align*}
As a result, $\tau_1(u) \cdot \beta_{\tau_1(u)} \geq 1/L(u)$ and \refeq{eq:pm6} becomes:
\begin{equation*}
\mass(u) \geq \frac{1}{9 L^2(u)}
\end{equation*}
This concludes our proof.
\end{proof}

Another important property of the zoom-in rule (and to be more precise, of its \emph{instantaneous} component) is that when \advzoom zooms-in on a node $u$, then, the probability $\pi_t(u)$ on this node is large {(\refeq{eq:sketch-large-prob}%
)}. Formally, this is stated below.

\begin{restatable}[Probability of Zoomed-In Node]{lemma}{estimatedgap}\label{lem:prob-zoomed-in}
If node $u$ gets zoomed-in at round $t$, then:
\begin{equation}\label{eq:zoom-lb-pi}
\frac{\beta_t}{\pi_t(u)} + \tbeta_t \leq e^{L(u)}-1 \Leftrightarrow \pi_t(u) \geq \frac{\beta_t}{e^{L(u)}}
\end{equation}
\end{restatable}

Note here that relaxing the right hand side of \refeq{eq:zoom-lb-pi} using the two facts that $\beta_t \geq \beta_t^2$ and $1 \geq \pi_t(\ust_t)$ gives the stated form of \refeq{eq:sketch-large-prob}%
.

\begin{proof}%
From the instantaneous zoom-in rule for node $u$ we have that if we zoom-in at a node $u$ at round $t$ the following must be true:
\begin{equation*}%
\frac{\beta_t}{\pi_t(u)} + \tbeta_t \leq e^{L(u)}-1 \Leftrightarrow \pi_t(u) \geq \frac{\beta_t}{e^{L(u)}- 1 - \tbeta_t} \geq \frac{\beta_t}{e^{L(u)}}
\end{equation*}

\end{proof}

\subsection{Properties of the Selection Rule}

This part of the analysis depends on the selection rule in the algorithm, but not on the zoom-in rule. Specifically, it works regardless of how $z_t(u)$ is defined in Line 10 of the algorithm. We consider the multiplicative weights update, as defined in \eqref{eq:alg:MW} and \eqref{eq:weight-characterization}, and derive a  lemma which corresponds to \refeq{eq:sketch-pot-lb0} and \refeq{eq:sketch-upp-main} in Section~\ref{sec:analysis}. {The proof encompasses the standard multiplicative-weights arguments, with several key modifications due to zooming.}

\begin{lemma}\label{lem:potential}
Assume the sequences $\{\eta_t\}$ and $\{\beta_t\}$ are decreasing in $t$, and satisfy
\begin{align}%
 \eta_t \leq \beta_t \leq \gamma_t/|A_t|
 \quad\text{and}\quad
 \eta_t \rbr{ 1 + \beta_t(1 + 4\log T) } \leq \gamma_t/|A_t|.
 \end{align}
Then, the following inequality holds:
\begin{align*}
\sum_{t \in [T]} \hg_t(\act_t(\ust_T)) - \sum_{t \in [T]} g_t(\vx_t) &\leq \frac{\ln \left(|A_T| \cdot \childprod(\ust_T) \right)}{\eta_T} + 4(1+\log T) \sum_{t \in [T]} \gamma_t + \\
&\hspace{0.5cm} +2 \sum_{t \in [T]} \eta_t (1 + (1 + 4 \log T) \beta_t) \sum_{u \in A_t} \hg_t(u)
\end{align*}
\end{lemma}

\begin{proof}
We use the following potential function:
\begin{equation}\label{eq:potential}
\Phi_t(\eta) = \left( \frac{1}{|A_t|} \sum_{u \in A_t} \frac{1}{\childprod(u)}\cdot\exp \left( \eta \sum_{\tau=1}^t \hg_\tau(\act_\tau(u)) \right) \right)^{1/\eta}, \quad \Phi_0(\eta) = 1,\forall \eta
\end{equation}
where $\Phi_0(\cdot) = 1$ since at round $0$ there is only one active node (the $\root$ with $h(\root) = 0$) and the estimator of the cumulative reward is initialized to $0$. We next upper and lower bound the quantity
\begin{equation}\label{eq:quantity}
Q = \ln \left( \frac{\Phi_{T}(\eta_T)}{\Phi_0(\eta_0)}\right) = \ln \left( \prod_{t=1}^T \frac{\Phi_t(\eta_t)}{\Phi_{t-1}(\eta_{t-1})}\right) = \sum_{t=1}^T \underbrace{\ln \left( \frac{\Phi_t(\eta_t)}{\Phi_{t-1}(\eta_{t-1})}\right)}_{Q_t}
\end{equation}
Expanding $Q$ we get:
\begin{align*}
\ln \left( \frac{\Phi_{T} (\eta_T)}{\Phi_0(\eta_0)} \right) &= \frac{1}{\eta_T} \ln \left( \frac{1}{|A_T|} \sum_{u \in A_T} \frac{1}{\childprod(u)}\cdot\exp \left( \eta_T \sum_{t=1}^T \hg_t(\act_t(u)) \right) \right) - \ln \left( \Phi_0(\eta_0) \right)\\
&= \frac{1}{\eta_T} \ln \left( \frac{1}{|A_T|} \sum_{u \in A_T} \frac{1}{\childprod(u)} \cdot\exp \left( \eta_T \sum_{t=1}^T \hg_t(\act_t(u)) \right) \right) &\tag{$\Phi_0(\cdot) = 1$} \\
&\geq \frac{1}{\eta_T} \ln \left( \frac{1}{|A_T|} \cdot \frac{1}{\childprod(u)}\cdot\exp \left( \eta_T \sum_{t=1}^T \hg_t(\act_t\left(\ust_T\right)) \right) \right) &\tag{$e^x > 0$} \\
&= \frac{1}{\eta_T} \ln \left(\exp \left(\eta_T \sum_{t=1}^{T} \hg_t(\act_t(\ust_T)) \right) \right) - \frac{\ln \left(|A_T| \cdot \childprod(\ust_T)\right)}{\eta_T} \\
&= \sum_{t=1}^T \hg_{t}(\act_t(\ust_T)) - \frac{\ln \left(|A_T| \cdot \childprod(\ust_T)\right)}{\eta_T} \numberthis{\label{eq:pot-lb0}} %
\end{align*}
For the upper bound we first focus on quantity $Q_t$ from \refeq{eq:potential}, and we start by breaking $Q_t$ into the following parts:
\begin{align*}
Q_t &= \ln \left(\frac{\Phi_t(\eta_t)}{\Phi_{t-1}(\eta_{t-1})} \right) = \ln \left(\frac{\Phi_t(\eta_t)}{\Phi_{t-1}(\eta_{t})} \cdot \frac{\Phi_{t-1}(\eta_t)}{\Phi_{t-1}(\eta_{t-1})} \right) \\
& = \underbrace{\ln \left( \frac{\Phi_t(\eta_t)}{\Phi_{t-1}(\eta_{t})} \right)}_{\widehat{Q}_t} + \underbrace{\ln \left( \frac{\Phi_{t-1}(\eta_{t})}{\Phi_{t-1}(\eta_{t-1})} \right)}_{\widetilde{Q}_t} \numberthis{\label{eq:q-defs}}
\end{align*}
We define the auxiliary function $f(\eta) = \ln \left(\Phi_{t-1}(\eta)\right)$ and prove that $f'(\eta) \geq 0$, hence the function is \emph{increasing} is $\eta$. Since $\eta_{t-1} \geq \eta_t$, this implies that quantity $\widetilde{Q}_t$ is \emph{negative} for all $t \in [T]$. For the ease of notation of this part we denote by $\hG_t(u)$ the quantity $\sum_{\tau=1}^t g_\tau(\act_\tau(u))$. For the derivative of function $f(\eta)$ we have:
\begin{align*}
f'(\eta) &= -\frac{1}{\eta^2} \ln \left(\frac{1}{|A_{t-1}|} \sum_{u \in A_{t-1}} \frac{1}{\childprod(u)} \exp \left(\eta \hG_t(u) \right) \right) \\
&\qquad + \frac{1}{\eta} \frac{\sum_{u \in A_{t-1}} \frac{1}{\childprod(u)} \hG_t(u) \exp \left( \eta \hG_t(u) \right)}{\sum_{u \in A_{t-1}} \frac{1}{\childprod(u)} \exp \left( \eta \hG_t(u) \right)}\\
&= \frac{1}{\eta^2} \frac{1}{\sum_{u \in A_{t-1}}\frac{1}{\childprod(u)} \exp \left( \eta \hG_t(u) \right)}\sum_{u \in A_{t-1}}\frac{1}{\childprod(u)} \exp \left( \eta \hG_t(u) \right) \cdot \\
&\hspace{4cm} \cdot \left[\eta \hG_t(u) - \ln \left(\frac{1}{|A_{t-1}|} \sum_{u \in A_{t-1}} \frac{1}{\childprod(u)} \exp \left(\eta \hG_t(u) \right) \right)\right] \\
&\geq \frac{1}{\eta^2} \frac{1}{\sum_{u \in A_{t-1}}\frac{1}{\childprod(u)} \exp \left( \eta \hG_t(u) \right)}\sum_{u \in A_{t-1}}\frac{1}{\childprod(u)} \exp \left( \eta \hG_t(u) \right) \cdot \\
&\hspace{2cm} \cdot \left[\eta \hG_t(u) - \ln \left(\frac{\childprod(u)}{|A_{t-1}|} \sum_{u \in A_{t-1}} \frac{1}{\childprod(u)} \exp \left(\eta \hG_t(u) \right) \right)\right] \numberthis{\label{eq:bef-kl}}
\end{align*}
where the last inequality is due to the fact that $\childprod(u) \geq 1, \forall u$. We define now the following two probability distributions:
\begin{align*}
D_1(u) &= \frac{\frac{1}{\childprod(u)} \exp \left( \eta \hG_t(u) \right)}{\sum_{u' \in A_{t-1}}\frac{1}{\childprod(u')} \exp \left( \eta \hG_t(u') \right)} \\
D_2(u) &= \frac{1}{|A_{t-1}|}
\end{align*}
Then, the right hand side of \refeq{eq:bef-kl} is the KL-divergence from $D_2$ to $D_1$. Because the KL-divergence is a non-negative quantity, $f'(\eta) \geq 0$, as desired.

We turn our attention to term $\widehat{Q}_t$ now and we break the process of transitioning from potential $\Phi_{t-1}(\eta_t)$ to $\Phi_t(\eta_t)$ into two steps. In the first step, the potential gets updated to be $\Phi_{t}^I(\eta_t)$\footnote{$I$ stands for ``intermediate''}, where the weights of all active nodes $u \in A_{t-1}$ get updated according to the $\hg_t(u)$ estimator. In the second step, the zoom-in happens and the potential transitions from $\Phi_{t}^I(\eta_t)$ to $\Phi_t(\eta_t)$. Note that since $|A_{t-1}| \leq |A_t|$ and for all children $v$ of $u$: $\sum_{v \in C(u)} w_{t+1}(v) = w_{t+1}(u)$ (and similarly for the probability of the children nodes) we have that, irrespective of whether we zoom-in on node $u$ or not, $\Phi_{t+1}^I(\eta_t) = \Phi_{t+1}(\eta_t)$. Hence,
\begin{align*}
&\widehat{Q}_t = \ln \left( \frac{\Phi_{t}(\eta_t)}{\Phi_t^I(\eta_t)} \cdot \frac{\Phi_{t}^I(\eta_t)}{\Phi_{t-1}(\eta_t)}\right) = \ln \left( \frac{\Phi_{t}(\eta_t)}{\Phi_t^I(\eta_t)} \right) + \ln \left(\frac{\Phi_{t}^I(\eta_t)}{\Phi_{t-1}(\eta_t)} \right) \\
&= \ln \left( \frac{\Phi_{t}^I(\eta_t)}{\Phi_{t-1}(\eta_t)} \right) &\tag{$\Phi_{t}(\eta_t) = \Phi_{t}^I(\eta_t)$} \\
&= \frac{1}{\eta_t} \cdot \ln \left(\frac{\frac{1}{|A_t|} \cdot \sum_{u \in A_{t}} w_t(\act_t(u),\eta_t) \cdot \exp \left(\eta_t \hg_t(\act_t(u)) \right)}{\frac{1}{|A_t|} W_t(\eta_t)} \right) &\tag{Lemma~\ref{lem:wu-rule}}\\
&= \frac{1}{\eta_t} \ln \left(\sum_{u \in A_{t}} p_t(\act_t(u)) \cdot \exp \left(\eta_t \hg_t(\act_t(u) \right) \right) %
= \frac{1}{\eta_t} \ln \left(\sum_{u \in A_{t}} p_t(u) \cdot \exp \left(\eta_t \hg_t(u) \right) \right) \numberthis{\label{eq:ub-middle}}
\end{align*}
where the last equality comes from the fact that by definition $\act_t(u) = u$ for all the rounds $t$ that $u$ was active.

Next we show that choosing $\beta_t,\gamma_t$ and $\eta_t$ according to the assumptions of the lemma, we have that $\eta_t \hg_t(\act_t(u)) \leq 1$. Indeed,
\begin{align*}
\eta_t \hg_t(u) &= \eta_t \cdot \frac{g_t(\vx_t)\1\left\{u = U_t \right\} + (1 + 4 \log T )\beta_t }{\pi_t(u)} &\tag{by definition of $\hg_t(\cdot)$} \\
&\leq \eta_t \cdot \frac{1 + (1 + 4 \log T )\beta_t }{\pi_t(u)} &\tag{$g_t(\cdot) \in [0,1]$} \\
&\leq \frac{\eta_t (1 + (1 + 4\log T)\beta_t) \cdot |A_t|}{\gamma_t} &\tag{$\pi_t(u) \geq \gamma_t/|A_t|$} \\
&\leq 1 &\tag{assumptions of Lemma}
\end{align*}
Since $\eta_t \hg_t(u) \leq 1$, then we can use inequality $e^x \leq 1 + x + x^2$ for $x \leq 1$ in \refeq{eq:ub-middle} and we have that:
\begin{align*}
\widehat{Q}_t &\leq \frac{1}{\eta_t} \left(\ln \left(\sum_{u \in A_{t}} p_t(u) \left( 1 + \eta_t \hg_t(u) + \eta_t^2 \hg_t^2(u)\right)\right)\right) \\
&= \frac{1}{\eta_t} \left( \ln \left( 1 + \eta_t \sum_{u \in A_{t}} p_t(u)\hg_t(u) + \eta_t^2 \sum_{u \in A_{t}} p_t(u)\hg_t^2(u)\right)\right) &\tag{$\sum_{u \in A_{t}} p_t(u) = 1$} \\
&\leq \sum_{u \in A_{t}} p_t(u)\hg_t(u) + \eta_t \sum_{u \in A_{t}} p_t(u)\hg_t^2(u) &\tag{$\ln(1 + x) \leq x, x \geq 0$}\\
&\leq \frac{1}{1 - \gamma_t}\sum_{u \in A_{t}} \pi_t(u)\hg_t(u) + \frac{\eta_t}{1-\gamma_t} \sum_{u \in A_{t}} \pi_t(u)\hg_t^2(u) \numberthis{\label{eq:ub2}}
\end{align*}
where the last inequality uses the fact that since for any node $u$: $\pi_t(u) = (1-\gamma_t)p_t(u) + \gamma_t/|A_t|$ then $p_t(u) \leq \pi_t(u)/(1 - \gamma_t)$. We next analyze term $\sum_{u \in A_{t}} \pi_t(u)\hg_t(u)$:
\begin{align*}
\sum_{u \in A_{t}} \pi_t(u)\hg_t(u) &= \sum_{u \in A_{t}} \pi_t(u)\left(\frac{g(\vx_t) \1\{ u = U_t\}}{\pi_t(u)} + \frac{(1 + 4\log T)\beta_t}{\pi_t(u)}\right) \\
&= \sum_{u \in A_{t}} \pi_t(u) \left(\frac{g_t(\vx_t) \1\{ u = I_t\}}{\pi_t(u)} + \frac{(1 + 4 \log T)\beta_t}{\pi_t(u)} \right) \\
&= g_t(\vx_t) + (1 + 4 \log T) \beta_t |A_{t}| \numberthis{\label{eq:first-mom}}
\end{align*}
Next, we analyze term $\sum_{u \in A_{t}} \pi_t(u)\hg_t^2(u)$:
\begin{align*}
\sum_{u \in A_{t}} \pi_t(u)\hg_t^2(u) &= \sum_{u \in A_{t}} \left( \pi_t(u)\hg_t(u)\right) \cdot \hg_t(u)\\
&= \sum_{u \in A_{t}} \left( \pi_t(u) \frac{g_t(\vx_t) \1\{u = U_t\} + (1 + 4\log T)\beta_t}{\pi_t(u)} \right) \cdot  \hg_t(u) \\
&\leq \sum_{u \in A_{t}} \left(1 + (1 + 4 \log T)\beta_t \right) \cdot \hg_t(u)  \numberthis{\label{eq:sec-mom}}
\end{align*}
where the last inequality is due to the fact that $g_t(\cdot) \in [0,1]$. Using \refeq{eq:first-mom} and \refeq{eq:sec-mom} in \refeq{eq:ub2}, we get that:
\begin{align*}%
\widehat{Q}_t &\leq \frac{1}{1 - \gamma_t} \cdot \left[g_t(\vx_t) + (1 + 4 \log T)\beta_t |A_{t}| + \eta_t (1+(1 + 4\log T)\beta_t)\sum_{u \in A_{t}} \hg_t(u) \right]\\
&\leq (1 + 2 \gamma_t)  \left[g_t(\vx_t) + (1 + 4 \log T)\beta_t |A_{t}| + \eta_t (1+(1 + 4\log T)\beta_t)\sum_{u \in A_{t}} \hg_t(u) \right] \numberthis{\label{eq:almost3}}
\end{align*}
where the last inequality comes from the assumption that $\gamma_t \leq 1/2$. Summing up both sides of the above for rounds $t = 1, \cdots, T$ we have that:
\begin{align*}%
\sum_{t=1}^T\widehat{Q}_t &\leq \sum_{t=1}^T g_t(\vx_t) + 4 (1 + \log T) \sum_{t=1}^T \gamma_t + 2 \sum_{t=1}^T \eta_t (1 + (1 + 4\log T) \beta_t) \sum_{u \in A_t} \hg_t(u) %
\end{align*}
From the assumption that $\eta_t \leq \beta_t$ the latter becomes:
\begin{align*}
\sum_{t=1}^T\widehat{Q}_t &\leq \sum_{t=1}^T g_t(\vx_t) + 2 \sum_{t=1}^T \gamma_t + 2 (1 + 4\log T) \sum_{t=1}^T \gamma_t + 4(1 + 2\log T) \sum_{t=1}^T \beta_t \sum_{u \in A_t} \hg_t(u) \numberthis{\label{eq:upp}}
\end{align*}
Combining and re-arranging \refeq{eq:pot-lb0} and~\refeq{eq:upp} concludes the proof of the lemma.
\end{proof}

\subsection{From Estimated to Realized Rewards.}
\label{app:advzoom-analysis-estimators}

In this part of the analysis, we go from properties of the estimated rewards to those of realized rewards. Recall that we posit deterministic rewards, conditioning on the clean event $\rewE$. Essentially, per-realization Lipschitzness \eqref{eq:lm:prob} holds for all rounds $t$ and all canonical arm-sequences. We prove several high-probability statements about estimated rewards; they hold with probability  at least $1-O(\delta)$ over the algorithm's random seed, for a given $\delta>0$. 

The confidence bounds are somewhat more general than those presented in the proof sketch in Section~\ref{sec:analysis}: essentially, the sums over rounds $\tau$ are weighted by time-dependent multipliers $\hbeta_\tau \leq\beta_\tau$. This generality does not require substantive new ideas, but it is essential for the intended applications. 

We revisit per-realization Lipschitzness and derive a corollary for inherited rewards. This is the statement and the proof of~\refeq{eq:sketch-lipschitz-prop}. 

\begin{restatable}[One-Sided-Lipschitz]{lemma}{lipschitzness}\label{lem:lipschitz}
For any node $u$ that is active at round $t$ it holds that:
\begin{equation}
\sum_{\tau \in [t]} g_\tau \left(\opt_{[t]}(u) \right)  - \sum_{\tau \in [t]} L(\act_\tau(u)) - \cpedit{4\sqrt{t d} \ln T } \leq \sum_{\tau \in [t]} g_\tau(\act_\tau(u)) %
\end{equation}
\end{restatable}

\begin{proof}
We use \eqref{eq:lm:prob} with two sequences of arms:
    $y_\tau = \repr(\act_\tau(u))$, $\tau\in[t]$
and
    $y'_\tau \equiv \opt_{[t]}(u)$.
By Lemma~\ref{lem:height-node}, node $u$ has height at most $1+ \log T$, so
    $(y_1 \LDOTS y_t)$
is a canonical arm-sequence. The other sequence,
    $(y'_1 \LDOTS y'_t)$
is a canonical arm-sequence since
    $\opt_{[t]}(u)\in\reprset$
by \eqref{eq:repr}. So, the Lemma follows by definition of the clean event $\rewE$.

\end{proof}

We prove a series of lemmas, which heavily rely on the properties of the zoom-in rule derived in Appendix~\ref{app:zoom-rule}. 
 The next lemma formally states the high probability confidence bound {(\refeq{eq:sketch-high-prob}%
)}.

\begin{restatable}[High Probability Confidence Bounds]{lemma}{likebcb}\label{lem:like-bcb12}
For any round $t\in[T]$,
any $\hbeta_\tau \in (0,1]$ such that $\hbeta_\tau \leq \beta_\tau$, $\delta > 0$ and any subset $A_\tau' \subseteq A_\tau$, with probability at least $1-\cpedit{T^{-2}}$ it holds that:
\begin{align*}%
\left| \sum_{\tau \in [t]} \hbeta_\tau \sum_{u \in A'_\tau} g_\tau(u) - \sum_{\tau \in [t]} \hbeta_\tau \sum_{u \in A'_\tau} \ips_\tau(u) \right| \leq \sum_{\tau \in [t]} \hbeta_\tau \sum_{u \in A'_\tau} \frac{\beta_\tau}{\pi_\tau(u)}  + \cpedit{2 \ln T} %
\end{align*}
\end{restatable}

\begin{proof}%
To simplify notation in the next proof, we define two quantities which when aggregated over time will correspond to the upper and lower confidence bounds of the true cumulative reward, for each node $u \in A_t$:
\begin{equation}\label{eq:upper}
\tg_t^+(u) = \ips_t(u) + \frac{\beta_t}{\pi_t(u)}, \; \text{and} \; \sum_{s=1}^{t} \tg_{s}^+(u) = \sum_{s=1}^t \left(\ips_s(\act_s(u)) + \frac{\beta_s}{\pi_s(\act_s(u))} \right)
\end{equation}
\begin{equation}\label{eq:lower}
\tg_{t}^-(u) = \ips_t(u) - \frac{\beta_t}{\pi_t(u)}, \; \text{and} \;  \sum_{s=1}^{t} \tg_{s}^-(u) = \sum_{s=1}^t \left(\ips_s(\act_s(u)) - \frac{\beta_s}{\pi_s(\act_s(u))}\right)
\end{equation}
It is easy to see that for all $u$ we can express $\hg_t(u)$ in terms of $\tg_t^+(u)$ and $\tg_t^-(u)$ as follows:
\begin{equation}\label{eq:hatg-up}
\hg_{t}(u) = \tg_t^+(u)+ \frac{4 \log(T) \beta_t}{\pi_t(u)}, \; \text{and} \;  \sum_{s=1}^{t} \hg_{s}(u) = \sum_{s=1}^t \left(\tg_s^+(\act_s(u)) + \frac{4 \log (T) \beta_s}{\pi_s(\act_s(u))} \right)
\end{equation}
\begin{equation}\label{eq:hatg-lo}
\hg_{t}(u) = \tg_t^-(u) + \frac{\left(3 + 4\log(T)\right)\beta_t}{\pi_t(u)}, \; \text{and} \;  \sum_{s=1}^{t} \hg_{s}(u) = \sum_{s=1}^t \left( \tg_s^-(\act_s(u)) + \frac{\left(3 + 4\log(T) \right)\beta_s}{\pi_s(\act_s(u))} \right)
\end{equation}

We prove the lemma in two steps. First, we show that for any $\delta >0$: 
\begin{align*}
\sum_{\tau \in [t]} \hbeta_\tau \sum_{u \in A'_\tau} g_\tau(u) &\leq \sum_{\tau \in [t]} \hbeta_\tau \sum_{u \in A'_\tau} \hg^+_\tau(u) + \ln \left(1/\delta \right) \numberthis{\label{eq:like-bcb12-up}}
\end{align*}

We denote by $\E_\tau$ the expectation conditioned on the draws of nodes until round $\tau$, i.e., conditioned on $U_1, \dots, U_{\tau-1}$. Assume node $v$ was active on round $\tau$. We will upper bound the quantity $\E_\tau \left[ \exp \left( \hbeta_\tau g_\tau(v) - \hbeta_\tau \tg_\tau^+(v) \right) \right]$:
\begin{align*}
\E_\tau &\left[ \exp \left( \hbeta_\tau g_\tau(v) - \hbeta_\tau \tg_\tau^+(v) \right) \right] = \E_\tau \left[ \exp \left( \hbeta_\tau g_\tau(v) - \hbeta_\tau \cdot \frac{g_\tau(\vx_\tau)\1 \{v = U_\tau \} + \beta_\tau}{\pi_\tau(v)} \right) \right] \\
&\leq \E_\tau \left[1 + \hbeta_\tau g_\tau(v)- \hbeta_\tau \frac{g_\tau(\vx_\tau)\1 \{u = U_\tau \}}{\pi_\tau(v)} + \left( \hbeta_\tau g_\tau(v)- \hbeta_\tau \frac{g_\tau(\vx_\tau)\1 \{v = U_\tau \}}{\pi_\tau(v)}\right)^2\right] \\
&\hspace{7cm}\cdot \exp \left(-\frac{\hbeta_\tau \cdot \beta_\tau}{\pi_\tau(v)} \right) \\
&= \left(1 + \E_\tau\left[ \hbeta_\tau g_\tau(v)- \hbeta_\tau \frac{g_\tau(\vx_\tau)\1 \{v = U_\tau \}}{\pi_\tau(v)}\right] + \E_\tau \left[\left(\hbeta_\tau g_\tau(v)- \hbeta_\tau \frac{g_\tau(\vx_\tau)\1 \{v = U_\tau \}}{\pi_\tau(v)} \right)^2\right]\right)\\
&\hspace{7cm}  \cdot \exp \left(-\frac{\hbeta_\tau \cdot \beta_\tau}{\pi_\tau(v)} \right) &\tag{linearity of expectation}\\
&\leq \left( 1 + \hbeta_\tau^2 \cdot \frac{g_\tau^2(\vx_\tau)}{\pi_\tau(v)}\right)\cdot \exp \left(-\frac{\hbeta_\tau \cdot \beta_\tau}{\pi_\tau(v)} \right) \leq 1 \numberthis{\label{eq:high-prob0}}
\end{align*}
where the first inequality is due to the fact that $e^x \leq 1 + x + x^2, x \leq 1$ and the last inequality is due to $1 + x \leq e^x$ and $\hbeta_\tau \leq \beta_\tau$.

Let $X_u$ be a \emph{binary} random variable associated with node $u$ which takes the value $1$ if node $u$ was chosen at this round by the algorithm (i.e., $u = U_\tau$) and $0$ otherwise. Clearly, $\{X_u\}_{u \in A_\tau}$ are \emph{not} independent. However, since $\sum_{u \in A_\tau} X_u = 1$ (i.e., at every round we play only one arm) then from \citet[Lemma~8]{DR96}, they are \emph{negatively associated}. As a result, from \citet{JDP83} for any non-increasing functions $f_u(\cdot), \forall u \in A_\tau$ it holds that: \[ \E \left[ \prod_{u \in A_\tau} f_u(X_u) \right] \leq \prod_{u \in A_\tau} \E \left[ f_u(X_u)\right] \] In our case, the functions $f_u, \forall u \in A'_\tau \subseteq A_\tau$ are defined as: \[f_u(X_u) = \exp \left( \hbeta_\tau g_\tau(u) - \hbeta_\tau \cdot \frac{g_\tau(\vx_\tau) \cdot X_u + \beta_\tau}{\pi_\tau(u)} \right), \quad \text{and} \quad \forall u \notin A'_\tau: f_u(X_u) = 1\] which are non-increasing for each node. Multiplying both sides of \refeq{eq:high-prob0} for all nodes $v \in A'_\tau$ and using the above stated properties we obtain:
\begin{align*}
\prod_{u \in A'_\tau} \E_\tau \left[ \exp \left( \hbeta_\tau g_\tau(u) - \hbeta_\tau \tg_\tau^+(u) \right) \right] \leq 1 \Leftrightarrow \E_\tau \left[ \prod_{u \in A'_\tau} \exp \left( \hbeta_\tau g_\tau(u) - \hbeta_\tau \tg_\tau^+(u) \right) \right] \leq 1
\end{align*}
Since both sides of \refeq{eq:high-prob0} are positive and at each round we take the expectation conditional on the previous rounds, we have that:
\begin{equation}\label{eq:bef-markov}
\E \left[\exp \left(\sum_{\tau=1}^t \hbeta_\tau \sum_{u \in A'_\tau} g_\tau(u) - \sum_{\tau=1}^t \hbeta_\tau \sum_{u \in A'_\tau} \tg_\tau^+(u) \right) \right] \leq 1
\end{equation}
From Markov's inequality, we have that $\Pr\left[ X  > \ln \left(1/\delta \right)\right] \leq \delta \E \left[e^X \right]$ for any $\delta >0$. Hence, from \refeq{eq:bef-markov} we have that with probability at least $1 - \delta$: \[\sum_{\tau=1}^t \hbeta_\tau \sum_{u \in A'_\tau} g_\tau(u) - \sum_{\tau=1}^t \hbeta_\tau \sum_{u \in A_\tau'} \tg_\tau^+(u) \leq \ln \left(1/\delta\right) \]

Next, we show that:
\begin{align*}
\sum_{\tau \in [t]} \hbeta_\tau \sum_{u \in A'_\tau} g_\tau(u) \geq \sum_{\tau \in [t]} \hbeta_\tau \sum_{u \in A'_\tau} \hg^-_\tau - \ln \left( 1/\delta \right) \numberthis{\label{eq:like-bcb12-lo}}
\end{align*}
Using exactly the same techniques we can show that:
\begin{equation*}\label{eq:bef-markov2}
\E \left[\exp \left(\sum_{\tau=1}^t \hbeta_\tau \sum_{u \in A'_\tau} \tg_\tau^-(u) - \sum_{\tau=1}^t \hbeta_\tau \sum_{u \in A'_\tau} g_\tau(u) \right) \right] \leq 1
\end{equation*}
and ultimately: \[\sum_{\tau=1}^t \hbeta_\tau \sum_{u \in A'_\tau} \tg_\tau^-(u) - \sum_{\tau=1}^t \hbeta_\tau \sum_{u \in A'_\tau} g_\tau(u) \leq \ln \left(1/\delta\right). \]Substituting $\delta = T^{-2}$ we get the result. 
\end{proof}

\begin{corollary}\label{cor:like-bcb12}
Fix round $t$ and tree node $u$. 
With probability at least $1 - T^{-2}$, we have:
\begin{align*}
\sum_{\tau=1}^t g_\tau (\act_\tau(u)) &\leq \sum_{\tau=1}^t \tg^+_\tau (\act_\tau(u)) + \frac{2\ln T}{\beta_t} \numberthis{\label{eq:like-bcb12-up-cor}} \\
\sum_{\tau=1}^t g_\tau (\act_\tau(u)) &\geq \sum_{\tau=1}^t \tg^-_\tau (\act_\tau(u)) - \frac{2\ln T}{\beta_t}  \numberthis{\label{eq:like-bcb12-lo-cor}}
\end{align*}
\end{corollary}

\begin{proof}
Apply Lemma~\ref{lem:like-bcb12} for $\hbeta_\tau = \beta_t$ and singleton subsets $A'_\tau = \act_\tau(u), \forall \tau \leq t$.
\end{proof}

The next lemma relates the estimates $\hg_t(u)$ with the optimal values $g_t(\opt_t(u))$ {(\refeq{eq:sketch-ghat-overestimate}%
)}.

\begin{restatable}{lemma}{validghat}\label{lem:valid-ghat}
Fix round $t > 0$ and a
non-increasing sequence of scalars $\hbeta_\tau$ such that $\hbeta_\tau \leq \beta_\tau, \forall \tau \leq t$, and $\beta_\tau \cdot |A_\tau| \leq \gamma_\tau$. %
Then, with probability at least $1-T^{-2}$ it holds that:
\begin{align*}
\sum_{\tau=1}^t \hbeta_\tau \sum_{u \in A_\tau} \hg_\tau(u) &- \sum_{\tau=1}^t \hbeta_\tau \sum_{u \in A_\tau} g_\tau(u) \leq 5 \log (T) \sum_{\tau=1}^t \hbeta_\tau |A_\tau| + 2 \ln T %
\end{align*}
\end{restatable}

\begin{proof}%
From \refeq{eq:hatg-lo}, we can relate function $\hg(\cdot)$ with function $\tg^-(\cdot)$ as follows:
\begin{align*}
\sum_{\tau=1}^t \hbeta_\tau \sum_{u \in A_\tau} \hg_\tau(u) &= \sum_{\tau=1}^t \hbeta_\tau \sum_{u \in A_\tau} \tg_\tau^-(u) + \sum_{\tau=1}^t \hbeta_\tau \sum_{u \in A_\tau} \frac{\beta_\tau\left( 3 + 4 \log T \right)}{\pi_\tau(u)} \\
&\leq \sum_{\tau=1}^t \hbeta_\tau \sum_{u \in A_\tau} g_\tau(u) + \sum_{\tau=1}^t \hbeta_\tau \sum_{u \in A_\tau} \frac{\beta_\tau\left( 3 + 4 \log T \right)}{\pi_\tau(u)} + 2 \ln T \numberthis{\label{eq:bef-lipschitz2}}
\end{align*}
where the last inequality comes from \refeq{eq:like-bcb12-lo} (Lemma~\ref{lem:like-bcb12}). %
Re-arranging, for \refeq{eq:bef-lipschitz2} we have:
\begin{align*}
\sum_{\tau=1}^t \hbeta_\tau \sum_{u \in A_\tau} \hg_\tau(u) - \sum_{\tau=1}^t \hbeta_\tau \sum_{u \in A_\tau} g_\tau(u) &\leq \underbrace{\sum_{\tau=1}^t \hbeta_\tau \sum_{u \in A_\tau} \frac{\beta_\tau\left( 3 + 4 \log(T) \right)}{\pi_\tau(u)}}_{\Gamma_3} + 2 \ln T \numberthis{\label{eq:bef}}
\end{align*}
Next, we focus on term $\Gamma_3$:
\begin{align*}
\Gamma_3 &= \sum_{\tau=1}^t \hbeta_\tau \sum_{u \in A_\tau} \frac{\beta_\tau\left( 3 + 4 \log(T) \right)}{\pi_\tau(u)} \leq \sum_{\tau=1}^t \hbeta_\tau \sum_{u \in A_\tau} \frac{\beta_\tau \cdot |A_\tau| \left( 3 + 4 \log(T) \right)}{\gamma_\tau} &\tag{$\pi_\tau(\cdot) \geq \gamma_\tau/|A_\tau|$} \\
&\leq (3 + 4 \log T ) \sum_{\tau=1}^t \hbeta_\tau |A_\tau| \numberthis{\label{eq:gamma3-ready}}
\end{align*}
where the last inequality comes from the fact that by assumption $\beta_\tau |A_\tau| \leq \gamma_\tau$. Substituting \refeq{eq:gamma3-ready} in \refeq{eq:bef} we have that:
\begin{align*}
\sum_{\tau=1}^t \hbeta_\tau \sum_{u \in A_\tau} \hg_\tau(u) &- \sum_{\tau=1}^t \hbeta_\tau \sum_{u \in A_\tau} g_\tau(u) \leq 5 \log (T) \sum_{\tau=1}^t \hbeta_\tau |A_\tau| + 2 \ln T
\end{align*}
\end{proof}

The next lemma formalizes \refeq{eq:sketch-ghat-zoomed} and \refeq{eq:sketch-ghat-underestimate}.%

\begin{restatable}[Singleton sets]{lemma}{singleton}\label{lem:lite-conf-lemma}
Fix any round $t$ and any tree node $u$ that is active at round $t$. 
With probability at least $1 - T^{-2}$, we have that:
\begin{align*}
\sum_{\tau=1}^t \hg_\tau(\act_\tau(u)) &\geq \sum_{\tau=1}^t g_\tau\left(\opt_{[t]}(u)\right) - \frac{4 \ln T}{\beta_t}- 4 \sqrt{t d } \ln T
\end{align*}
Moreover, if round $t$ is the zoom-in round for node $u$ and arm $y \in u$ then the following also holds:
\begin{align*}
\sum_{\tau=1}^t \hg_\tau(\act_\tau(u)) &\leq \sum_{\tau=1}^t g_\tau \left(y\right)  +  9t L(u)\ln T  + \frac{\ln T}{\beta_t} + 4 \sqrt{t d }\ln T
\end{align*}
\end{restatable}

\begin{proof}%
We start from the lower bound which holds for any $t \in [\tau_0(u), \tau_1(u)]$. From \refeq{eq:hatg-up}, we can first relate function $\hg(\cdot)$ with function $\tg^+(\cdot)$ as follows:
\begin{align*}
\sum_{\tau=1}^t &\hg_\tau(\act_\tau(u)) = \sum_{\tau=1}^t \tg_\tau^+(\act_\tau(u)) + \sum_{\tau=1}^t \frac{4 \log(T)\beta_\tau}{\pi_\tau(\act_\tau(u))} \\ %
&\geq \sum_{\tau=1}^t g_\tau(\act_\tau(u)) + \sum_{\tau=1}^t \frac{4 \log (T) \beta_\tau}{\pi_\tau(\act_\tau(u))} - \frac{2 \ln T}{\beta_t} &\tag{Eq.~\eqref{eq:like-bcb12-up-cor} of Cor.~\ref{cor:like-bcb12}} \\
&\geq \sum_{\tau=1}^t g_\tau(\opt_{[t]}(u)) - 4 \sqrt{td } \ln T - \underbrace{\sum_{\tau=1}^t L(\act_\tau(u))  + \sum_{\tau=1}^t \frac{4 \log (T) \beta_\tau}{\pi_\tau(\act_\tau(u))}}_{\Gamma} - \frac{2 \ln T}{\beta_t} \numberthis{\label{eq:bef-zooming-inv-lite}}
\end{align*}
where the last inequality comes from the one-sided-Lipschitzness lemma (Lemma~\ref{lem:lipschitz}). %
From the zooming invariant (Lemma~\ref{lem:zooming-invariant}) we have that:
\begin{equation}\label{eq:zoom-inv-derive}
\sum_{\tau=1}^t \frac{\beta_\tau}{\pi_\tau(\act_\tau(u))} \geq t L(u) - \frac{1}{\beta_t}
\end{equation}
From Lemma~\ref{lem:inh-bias} we can upper bound the inherited bias of node $u$ with respect to its diameter: $\sum_{\tau=1}^t L(\act_\tau(u)) \leq 4 t \log(t) L(u)$. Combining this with \refeq{eq:zoom-inv-derive} we obtain that the term $\Gamma$ of \refeq{eq:bef-zooming-inv-lite} is:
\begin{align*}
\Gamma \geq 4t L(u) \left( \log T - \log t \right) - \frac{1}{\beta_t} \geq - \frac{1}{\beta_t}
\end{align*}
Thus, \refeq{eq:bef-zooming-inv-lite} becomes:
\begin{align*}
\sum_{\tau=1}^t \hg_\tau(\act_\tau(u)) &\geq \sum_{\tau=1}^t g_\tau\left(\opt_{[t]}(u)\right) - \frac{4 \ln T}{\beta_t} - 4 \sqrt{t d } \ln T %
\end{align*}
This concludes our proof for the lower bound.

We now turn our attention to the upper bound. From \refeq{eq:hatg-lo}, we can first relate function $\hg(\cdot)$ with function $\tg^-(\cdot)$ as follows:
\begin{align*}
\sum_{\tau=1}^t &\hg_\tau(\act_\tau(u)) = \sum_{\tau=1}^t \tg_\tau^-(\act_\tau(u)) + \sum_{\tau=1}^t \frac{(3 + 4 \log T )\beta_\tau}{\pi_\tau(\act_\tau(u))} \\
&\leq \sum_{\tau=1}^t g_\tau(\act_\tau(u)) + \sum_{\tau=1}^t \frac{(3 + 4 \log T) \beta_\tau}{\pi_\tau(\act_\tau(u))} + \frac{2 \ln T}{\beta_t} &\tag{Eq.~\eqref{eq:like-bcb12-lo-cor} of Cor.~\ref{cor:like-bcb12}} \\
&\leq \sum_{\tau=1}^t g_\tau(\opt_\tau(\act_\tau(u))) + \underbrace{\sum_{\tau=1}^t \frac{(3 + 4 \log T) \beta_\tau}{\pi_\tau(\act_\tau(u))}}_{\Gamma_1} + \frac{2 \ln T}{\beta_t} \numberthis{\label{eq:bef-aggr-zooming-lite}}
\end{align*}
where the last inequality comes from the fact that $g_\tau(y) \leq g_\tau(\opt_\tau(v))$ for any node $v$ and point $y \in v$. From the one-sided-Lipschitzness lemma (Lemma~\ref{lem:lipschitz}) we have that:
\begin{align*}
\sum_{\tau=1}^t g_\tau(\opt_\tau(\act_\tau(u))) &\leq \sum_{\tau=1}^t g_\tau\left(y\right) + \sum_{\tau=1}^t L(\act_\tau(u)) + 4 \sqrt{t d }\ln T \\
&\leq \sum_{\tau=1}^t g_\tau\left(y\right) + 4t L(u) \log T + 4 \sqrt{t d }\ln T \numberthis{\label{eq:almostlip}}
\end{align*}
where the last inequality comes from Lemma~\ref{lem:inh-bias}.

For term $\Gamma_1$ since $t = \tau_1(u)$ we can apply the aggregate zoom-in rule and obtain:
\begin{align*}
\Gamma_1 \leq (3 + 4 \log T) t L(u) - \frac{3 + 4 \log T}{\beta_t} \leq 5t L(u) \ln T  - \frac{\ln T}{\beta_t}
\end{align*}
where the last inequality comes from the fact that $3 + 4 \log T \geq \ln T$ and $3 + 4 \log T \leq 5 \ln T$. Substituting the upper bound for $\Gamma_1$ and \refeq{eq:almostlip} in \refeq{eq:bef-aggr-zooming-lite} we get the result.
\end{proof}

\subsection{Regret Bounds}

Now, we bring the pieces together to derive the regret bounds. As in Theorem~\ref{thm:main-regr}, we first derive a ``raw'' regret bound \refeq{eq:thm-raw} in terms of the parameters, then we tune the parameters and derive a cleaner regret bound \refeq{eq:thm-tuned} in terms of $|A_T|$ (the number of active nodes). Then, we upper-bound $|A_T|$ via the adversarial zooming dimension to derive the final regret bound \refeq{eq:thm-ZoomDim}.

For this section, we condition on the clean event for the algorithm, denoted $\algE$. This event states that, essentially, all relevant high-probability properties from Appendix~\ref{app:advzoom-analysis-estimators} actually hold. More formally, for each round $t\in [T]$, using failure probability $\delta = T^{-2}$, the following events hold:
\begin{OneLiners}
\item the event in Lemma~\ref{lem:valid-ghat} holds with
    $\hbeta_\tau \equiv\beta_\tau$.
\item the event in Lemma~\ref{lem:lite-conf-lemma} holds for each nodes $u$ that are active in round $t$ and $y = \opt_{[t]}(y)$.
\end{OneLiners}
Taking appropriate union bounds, we find that
$\Pr\sbr{ \algE } \geq 1-\delta$. In the subsequent lemmas, we condition on both clean events, $\rewE$ and $\algE$, without further notice.

Recall that
$\childprod(u) := \prod_v |\c(v)|$, as in Lemma~\ref{lem:wu-rule}.

\begin{lemma}\label{lem:bef-tuning}
\advzoom incurs regret:
\begin{align*}
R(T) \leq 4 \sqrt{T d } \ln T
+ \frac{6 \ln T}{\beta_T} + \frac{\ln (T) \ln \left(\DblC \cdot |A_T|\right)}{\eta_T}
+ \sum_{t=1}^T
    4(1+2\log T)  \beta_t +42 \log^2(T) \gamma_t.
\end{align*}
\end{lemma}

\begin{proof}
We begin with transitioning from Lemma~\ref{lem:potential} to a statement that only includes realized rewards (rather than estimated ones). To do so, by the clean event $\algE$ (namely, Lemma~\ref{lem:lite-conf-lemma})
and the fact that $\opt_{[T]}(\ust) = \vxst_{[T]}$ it follows that
\[\sum_{t=1}^T \hg_{t}\left(\act_t\left(\ust\right)\right) \geq \sum_{t=1}^T g_t \left(\vxst_{[T]}\right) - \frac{2\ln(1/\delta)}{\beta_T} - 4 \sqrt{T d} \ln T\]
As a result, the lower bound of quantity $Q$ becomes:
\begin{equation*}
\ln \left( \frac{\Phi_{T} (\eta_T)}{\Phi_0(\eta_0)} \right) \geq \sum_{t=1}^T g_t(\vxst_{[T]}) - \frac{4 \ln T}{\beta_T} - \frac{ \ln \left(|A_T| \cdot \childprod(\ust)\right)}{\eta_T}-  4 \sqrt{T d } \ln T
\end{equation*}
By the clean event $\algE$ (namely, Lemma~\ref{lem:valid-ghat})
for $t = T$ and $\hbeta_\tau = \beta_\tau$ we have that
\begin{align*}
\sum_{t=1}^T \beta_t \sum_{u \in A_t} \hg_t(u) &\leq \sum_{t=1}^T \beta_t \sum_{u \in A_t}g_t(u) + 5 \log (T) \sum_{t=1}^T \beta_t |A_t| + 2 \ln T \\
&\leq \sum_{t=1}^T \beta_t + 5 \log (T) \sum_{t=1}^T \gamma_t + 2 \ln T
\end{align*}
where the last inequality is due to the fact that by assumption $\beta_t |A_t| \leq \gamma_t$ and $g_t(\cdot) \in [0,1]$. As a result, \refeq{eq:upp} becomes:
\begin{align*}
\sum_{t=1}^T\widehat{Q}_t &\leq \sum_{t=1}^T g_t(\vx_t) + 4(1+2\log T)\sum_{t=1}^T \beta_t + 42 \log^2(T) \sum_{t=1}^T \gamma_t \numberthis{\label{eq:upp2}}
\end{align*}
As a result, from Lemma~\ref{lem:potential} we get that:
\begin{align*}
\sum_{\tau=1}^T &g_t (\vxst_{[T]}) - \sum_{t=1}^T g_t(\vx_t) \\
&\leq  \frac{6 \ln T}{\beta_T} + \frac{\ln \left(|A_T| \cdot \childprod(\ust) \right)}{\eta_T} + 4(1+2\log T) \sum_{t=1}^T \beta_t + 42 \log^2(T) \sum_{t=1}^T \gamma_t \numberthis{\label{eq:bef-childprod-bound}}
\end{align*}
Using the fact that $\childprod(\ust) \leq \DblC^{\log T}$ (Lemma~\ref{lem:childprod-bound}), \refeq{eq:bef-childprod-bound} becomes:
\begin{align*}
\sum_{\tau=1}^T g_t (\vxst_{[T]}) - \sum_{t=1}^T g_t(\vx_t) \leq &4 \sqrt{T d } \ln T +\frac{6 \ln T}{\beta_T} + \frac{\ln (T) \ln \left(\DblC \cdot |A_T|\right)}{\eta_T} +\\
 & \quad + 4(1+2\log T) \sum_{t=1}^T \beta_t + 42 \log^2(T) \sum_{t=1}^T \gamma_t \numberthis{\label{eq:regr-bef-tuning}}
\end{align*}
\end{proof}

\begin{lemma}\label{lem:tuned}
Tune $\beta_t, \gamma_t, \eta_t$ as in \refeq{eq:tunings}.
 \advzoom incurs regret
\begin{align*}
\sum_{t=1}^T g_t (\vxst) - \sum_{t=1}^T g_t(\vx_t) \leq 100 \; \ln^2 (T) \sqrt{d \cdot |A_T| \cdot T \cdot \ln \left(|A_T| \cdot T^3\right) \ln \left( \DblC \cdot |A_T| \right)}.
\end{align*}

\end{lemma}

\begin{proof}%
{First, we verify that the tuning in \refeq{eq:tunings} satisfies the  various assumptions made throughout the proof. It is easy to see that the assumptions in \refeq{eq:param-assns} hold; we omit the easy details. As for the assumption \eqref{ass:tilde-beta} on $\tbeta_t$ parameters:}
\begin{align*}
\sum_{\tau=t}^{t'} \tbeta_\tau &= \sum_{\tau=t}^{t'} \sqrt{\frac{2 \; \ln \left(|A_\tau| \cdot T^3\right) \ln \left( \DblC \cdot |A_\tau| \right)}{d \cdot \ln^2 T \cdot |A_\tau| \cdot \tau}} \\
&\leq \sum_{\tau=t}^{t'} \frac{1}{\sqrt{\tau}} &\tag{$2\ln \left(|A_\tau| \cdot T^3\right) \ln \left( \DblC \cdot |A_\tau| \right) \leq d \cdot |A_\tau|$ asymptotically} \\
&\leq \frac{\ln (T) \cdot \sqrt{2 d \cdot |A_t| \cdot\ln \left(|A_t| \cdot T^3\right) \ln \left( \DblC \cdot |A_t| \right)}}{2} \sum_{\tau=t}^{t'} \frac{1}{\sqrt{\tau}} &\tag{$d \cdot |A_t| \geq 4$} \\
&\leq \frac{\ln(T) \cdot\sqrt{2 d \cdot |A_t| \cdot\ln \left(|A_t| \cdot T^3 \right) \ln \left( \DblC \cdot |A_t| \right)}}{2} \int_{t}^{t'} \frac{1}{\sqrt{\tau}} \ud \tau  \\
&\leq \ln (T) \cdot \sqrt{2 d \cdot |A_t| \cdot\ln \left(|A_t| \cdot T^3\right) \ln \left( \DblC \cdot |A_t| \right)} \left( \sqrt{t'} - \sqrt{t} \right)  \\
&\leq \frac{1}{\beta_{t'}} - \frac{1}{\beta_{t}} &\tag{$|A_{t}| \leq |A_{t'}|, \forall t \leq t'$}
\end{align*}

Now, let us plug in the parameter values into Lemma~\ref{lem:bef-tuning}.
For the term $\sum_{t=1}^T \gamma_t$, we have:
\begin{align*}
\sum_{t=1}^T \gamma_t &= (2 + 4 \log T)\sum_{t=1}^T \sqrt{\frac{2 |A_t| \cdot \ln \left(|A_t| \cdot T^3\right) \ln \left( \DblC \cdot |A_t| \right)}{t \cdot d \cdot \ln^2(T)}} \\
&\leq 5 \frac{\log T}{\ln T}\sqrt{|A_T| \cdot\ln \left(|A_T| \cdot T^3\right) \ln \left( \DblC \cdot |A_T| \right)} \cdot \sum_{t=1}^T\sqrt{\frac{1}{d \cdot} t} &\tag{$|A_t| \leq |A_T|, \forall t \in [T]$} \\
&\leq \frac{3}{d}\sqrt{|A_T| \cdot\ln \left(|A_T| \cdot T^3\right) \ln \left( \DblC \cdot |A_T| \right) \cdot T},
\end{align*}
where the last inequality comes from the fact that since $\nicefrac{1}{\sqrt{t}}$ is a non-increasing function: $\sum_{t=1}^T \frac{1}{\sqrt{t}} \leq \int_0^T \frac{1}{\sqrt{t}} \ud t = 2 \sqrt{T}$. Substituting gives the result.
\end{proof}

Next, we upper-bound the ``estimated gap" using properties of multiplicative weights (this is \refeq{eq:sketch-estimated-gap} in Section~\ref{sec:analysis}).

\begin{restatable}[Estimated Gap of Zoomed-In Node]{lemma}{estimatedgap}\label{lem:est-gap-zoomed-in}
If node $u$ gets zoomed-in at round $t$, then:
\begin{equation*}
\sum_{\tau \in [t]} \hg_\tau(\act_\tau(\ust_t)) - \sum_{\tau \in [t]} \hg_\tau(\act_\tau(u)) \leq \frac{\ln \left(\frac{9\childprod(\ust_t)}{\childprod(u)\cdot\beta_t^2}\right)}{\eta_t}.
\end{equation*}
\end{restatable}

\begin{proof}
Substituting $\pi_t(u) = (1-\gamma)p_t(u) + \frac{\gamma}{|A_t|}$ in Lemma~\ref{lem:prob-zoomed-in}, we have that:
\begin{align*}
p_t(u) &\geq \left(\frac{\beta_t}{e^{L(u)} - 1 - \tbeta_t} - \frac{\gamma_t}{|A_t|} \right) \cdot \frac{1}{1-\gamma_t} \\
&\geq \frac{\beta_t}{e^{L(u)} - 1 - \tbeta_t} - \frac{\gamma_t}{|A_t|}&\tag{$0 < \gamma_t <1/2$}\\
&= \frac{\beta_t}{e^{L(u)} - 1 - \tbeta_t} - \beta_t &\tag{$\beta_t \cdot |A_t| {\leq} \gamma_t$}\\
&\geq \frac{\beta_t^2}{e^{L(u)}} \numberthis{\label{eq:zoom-lb-exp3prob}}
\end{align*}
where the last inequality is due to the fact that $L(u) \leq L(u_0) = 1$.

We denote by $W_{t,\eta}$ the quantity \[\sum_{u' \in A_t} \frac{1}{\childprod (u')}\exp \left(\eta \sum_{\tau = 1}^t \hg_t(\act_\tau(u')) \right)\]By the definition of the probability being the normalized weight: $p_t(u) = w_{t,\eta_t}(u)/W_{t,\eta_t}$, so \refeq{eq:zoom-lb-exp3prob} becomes:
\begin{equation}\label{eq:weights-zoomed}
w_{t,\eta_t}(u) \geq \frac{\beta_t^2}{e^{L(u)}} \cdot W_{t,\eta_t} \geq \frac{\beta_t^2}{e^{L(u)}} \cdot w_{t,\eta_t}(\ust_t)
\end{equation}
where the second inequality comes from the fact that the weights are non-negative. Using the definition of the weights update rule from \refeq{eq:weight-characterization}, \refeq{eq:weights-zoomed} becomes:
\begin{equation*}
\frac{1}{\childprod(u)} \exp \left(\eta_t \sum_{\tau=1}^t \hg_\tau(\act_\tau(u)) \right) \geq \frac{\beta_t^2}{e^{L(u)}} \cdot \frac{1}{\childprod(\ust_t)} \exp \left(\eta_t \sum_{\tau=1}^t \hg_\tau(\act_\tau(\ust_t)) \right)
\end{equation*}
Taking logarithms on both sides of the latter, reordering and dividing by $\eta_t$ we get that:
\begin{equation}\label{eq:ineq-hgs}
\sum_{\tau=1}^t \hg_\tau(\act_\tau(\ust_t)) - \sum_{\tau=1}^t \hg_\tau(\act_\tau(u)) \leq \frac{2 \ln \left(1/\beta_t\right) +L(u) + \ln \left(\childprod(\ust)/\childprod(u) \right)}{\eta_t}
\end{equation}
Using the fact that $L(u) \leq \ln 3$  we get the result.
\end{proof}

Using this, we can derive a statement about the actual gap in the realized rewards for any round. We first state an auxiliary lemma, which we will use in order to derive a simplified statement of the actual gap in the realized rewards for any round.

\begin{lemma}[Upper Bound on Number of Nodes Created]\label{lem:upp-bound-num-nodes}
In the worst-case, at any round $t$, the total number of nodes that \advzoomdim has activated is: $|A_t| \leq (9t)^{d/(d+2)}$, where $d = \covdim$.
\end{lemma}

\begin{proof}
In order to find the worst-case number of nodes that can be activated, we are going to be thinking from the perspective of an \emph{adversary}, whose goal is to make the algorithm activate as many nodes as possible. However, from Lemma~\ref{lem:prob-mass} we have established that in order to zoom-in on a node $u$ a probability mass of at least $\mass(u) \geq \frac{1}{9L^2(u)}$ is required. So, the best that the adversary can do is activate nodes ``greedily'', i.e., when $\mass(u) = \frac{1}{9L^2(u)}$ to make node $u$ become active.

Let $\zeta$ be diameter of the smallest node $\uu$ that the adversary has been able to construct after $t$ rounds, i.e., $\zeta = L(\uu)$ where $\uu = \arg\min_{u \in A_t} L(u)$. We fix the constant multiplier of the covering dimension to be $\gamma = 1$. Then, from the definition of the covering dimension, in the worst case, the adversary has been able to activate $\zeta^{-\covdim}$ such nodes. However, since for each of these nodes the probability mass spent on the node is at least $1/9\zeta^2$, and the total probability mass available for $t$ rounds is $t$ (i.e., a total of $1$ for each round) we have that in the worst case: \[\zeta^{-\covdim} \cdot \frac{1}{9\zeta^2} = t\] Solving the latter we get that in this case $\zeta = \left(9t\right)^{-1/(\covdim + 2)}$. By the definition of the covering dimension this means that the maximum number of active nodes at this round is:
\[|A_t| \leq (9t)^{\covdim/(\covdim+2)} \]
\end{proof}

\begin{restatable}{lemma}{advgaplemma}\label{lem:gap}
Suppose node $u$ is zoomed-in in some round $t = \tau_1(u)$. For each arm $\vx\in u$, its adversarial gap at time $t$ is
\begin{align*}
t\cdot \gap_{t}(x) \leq 9 t L(u) \ln T + 4 \sqrt{t d } \cdot \ln T %
+ \frac{\ln T+\ln \left(\frac{9 \childprod(\ust_t)}{\beta_{t}^2 \cdot \childprod(u)} \right)}{\eta_{t}} \numberthis{\label{eq:bef-gap}}
\end{align*}
Simplified:
\begin{equation}\label{eq:advgap-def}
\gap_{t}(x) \leq \calO \left( \ln (T) \cdot  \sqrt{\ln (\DblC \cdot t)} \cdot L(u) \right).
\end{equation}
\end{restatable}

{Before proving Lemma~\ref{lem:gap}, we upper bound $\childprod(u)$ in terms of the doubling constant $\DblC$.}

\begin{restatable}{lemma}{childprodbound}\label{lem:childprod-bound}
$\childprod(u) \leq \DblC^{\log T}.$ %
\end{restatable}

\begin{proof}%
The maximum number of children that a node $u$ can activate is $|\c(u)| \leq \DblC$. As a result, and since the total number of ancestors that node $u$ has is $h(u)$ we have that: \[ \childprod(u) \leq \DblC^{h(u)} \leq \DblC^{\log T}\]
where the last inequality is due to the fact that $\forall v: h(v) \leq \log T$ (Lemma~\ref{lem:height-node}).
\end{proof}

\begin{proof}[Proof of Lemma~\ref{lem:gap}]
To prove this lemma, we apply the clean event $\algE$
(namely, Lemma~\ref{lem:lite-conf-lemma}) in Lemma~\ref{lem:est-gap-zoomed-in} and for notational convenience we use: $t = \tau_1(u)$.%
\begin{align*}
\sum_{\tau=1}^{t} g_\tau(\vxst_{[t]}) &- \sum_{\tau=1}^{t} g_\tau(x) \leq 9 t L(u) \ln T + 4\sqrt{t d }\ln T + \frac{\ln T}{\beta_{t}} + \frac{\ln \left(\frac{9 \childprod(\ust_t)}{\beta_{t}^2 \cdot \childprod(u)} \right)}{\eta_{t}} \numberthis{\label{eq:bef-gap}}
\end{align*}
where $x \in u$. %
This proves the first part of Lemma~\ref{lem:gap}. We move next to proving its simplified version.
Since $\childprod(u) > 1$ and $\childprod(u) \leq \DblC^{\log T}$ for all nodes $u$, the latter can be relaxed to:
\begin{align*}
\gap_{t}(x) \leq 9 L(u) \ln T + \frac{\ln T}{t\beta_{t}} + 4\sqrt{\frac{ d \cdot }{t}}\cdot \ln T + \frac{2 \ln T \cdot \ln \left(\DblC \cdot t \cdot |A_{t}| \right)}{t\eta_{t}}
\end{align*}
where we have used again the notation $t = \tau_1(u)$. Substituting the values for $\eta_{t}$ and $\beta_{t}$ from Theorem~\ref{thm:main-regr}  the latter becomes:
\begin{align*}
\gap_{t}(x) &\leq 9 L(u) \ln T + \frac{\ln (T) \cdot \sqrt{|A_t| \, d \, \ln(T)}}{\sqrt{2t\, \ln \left( |A_t| \, T^3\right) \cdot \ln \left(\DblC \cdot |A_t|\right)}} + 4\sqrt{\frac{d}{t}} \cdot \ln (T)\\
&\qquad \qquad \qquad + \frac{2 \ln (T) \sqrt{\ln \left( \DblC \cdot t \cdot |A_t| \right) \cdot|A_t| \cdot d \cdot \ln (T)}}{\sqrt{2 t \cdot \ln \left( |A_t| \cdot T^3 \right)}} \\
&\leq 9 L(u) \ln T + \frac{\ln (T) \sqrt{|A_t| \, d}}{\sqrt{2 t}} +\frac{\ln (T) \sqrt{d}}{\sqrt{t}} + \frac{2 \ln (T) \sqrt{|A_t| \cdot d \cdot \ln \left(\DblC \cdot t \cdot |A_t|  \right)}}{{\sqrt{t}}}\\
&\leq 9 L(u) \ln T + \frac{4 \ln(T)\sqrt{ d \cdot |A_t| \cdot\ln \left( \DblC \cdot t \cdot |A_{t}|\right)} }{\sqrt{t}} \numberthis{\label{eq:bef-num-arms}}
\end{align*}
where the second inequality comes from the fact that $|A_t| > 1$ and $\DblC > 1$.

We know from Lemma~\ref{lem:upp-bound-num-nodes} that at any round $t$, the number of active nodes is upper bounded as $|A_t| \leq (9t)^{\covdim/(\covdim+2)}$.
As a result, \refeq{eq:bef-num-arms} becomes:
\begin{align*}
\gap_{t}(x) &\leq 9 L(u) \ln T +\frac{4 \cdot 3^{\frac{\covdim}{\covdim+2}}\cdot \ln (T) \cdot \sqrt{d \cdot \frac{2\covdim +2}{\covdim+2} \cdot \ln \left(\DblC \cdot 9t\right)}}{t^{\frac{1}{\covdim+2}}}
\end{align*}
But, $t^{-\frac{1}{\covdim+2}}$ is the smallest possible diameter (i.e., $L(v)$) that an adversary could have been able to force our algorithm to construct, as we stated earlier for the chosen $\zeta$ of Lemma~\ref{lem:upp-bound-num-nodes}. In other words and using the fact that $(\covdim+1)/(\covdim+2) \leq 1$ we get:
\begin{equation*}
\gap_{t}(x) \leq 7 \cdot 3^{\frac{\covdim}{\covdim+2}} \cdot \ln(T) \cdot \sqrt{d \cdot \ln (9\DblC \cdot t)} \cdot L(u)
\end{equation*}
Or simply:
\begin{equation*}
\gap_{t}(x) \leq \calO \left( \ln (T) \cdot  \sqrt{d \cdot \ln (\DblC \cdot t)} \cdot L(u) \right)
\end{equation*}
where the $\calO(\cdot)$ notation hides constant terms.
\end{proof}

Using Lemma~\ref{lem:gap} we can derive the regret statement of \advzoom in terms of \advzoomdim (i.e., \refeq{eq:thm-ZoomDim}). %
We clarify that in order to achieve \refeq{eq:thm-ZoomDim} we will eventually have to prove a stricter bound on the number of active nodes.

We are now ready to state the regret guarantee of \advzoom using the notion of \advzoomdim. This corresponds to the derivation in \refeq{eq:thm-ZoomDim}.

\begin{lemma}
With probability at least $1-T^{-2}$, \advzoom incurs regret:
\begin{align*}
R(T) \leq 100 \cdot \frac{3z}{z+2} \cdot d^{\frac{1}{2}} \cdot T^{\frac{z+1}{z+2}}\cdot\left(\gamma \cdot \DblC \cdot 2^z \cdot \log T \right)^{\frac{1}{z+2}} \sqrt{2\ln\left(T^3 \cdot \gamma \cdot \DblC \cdot \log T \right)  \cdot \ln \left( \DblC \cdot T\right)}
\end{align*}
where $z = \advzoomdim$ and $d = \covdim$.
\end{lemma}

\begin{proof}
Starting from $L(\root) = 1$ every time that a zoom-in happens on \advzoom, the diameter of the interval gets halved. We call this process an increase in \emph{scale}. Let $\S$ be the total number of scales from node $\root$ to the smallest created node after $T$ rounds. $\S$ depends on the problem instance. Let $Z(\eps_i)$ (where $i \in [\S]$) denote the number of nodes $u$ with diameter $L(u) \geq \eps_i$ with gap $\gap_\rho(x) \leq \calO (\eps_i \cdot \ln (T) \cdot \sqrt{d \cdot \ln \left(\DblC \cdot \rho \right)} )$ for $x \in \reprset$ such that $x \in u$ at some round $\rho = c /\eps_i^2$, where $c>0$ is a constant.

In order for a node $u$ to get zoomed-in, a probability mass of $\mass(u) \geq 1/9L^2(u)$ is required. Since $\mass(u) = \sum_{\tau = \tau_0(u)}^{\tau_1(u)} \pi_\tau(u) \leq \tau_1(u) - \tau_0(u) + 1 \leq \tau_1(u)$, then this means that for the de-activation time of node $u$ it holds that $\tau_1(u) \geq 1/9L^2(u)$. We choose $\eps_i$ and $c$ in a way that $\rho = \tau_1(u)$ for the maximum number of nodes. Then, all of these nodes belong in the set $\Z(\eps_i)$ and all of them get zoomed-in. When a node $u$ gets zoomed-in at most $|\c(u)| \leq \DblC$ children-nodes get activated. Inductively, after $T$ rounds and given that there is a total probability mass of $T$, we have the following:
\begin{equation}\label{eq:zoom-dim-tune}
\DblC \cdot \Z(\eps_0) \cdot \frac{1}{9 \epsilon_0^2} + \DblC \cdot \Z(\eps_1) \cdot \frac{1}{9 \eps_1^2} + \cdots + \DblC \cdot \Z(\eps_\S) \cdot \frac{1}{9 \eps_\S^2} = T \Leftrightarrow \frac{\DblC}{9} \sum_{i \in [\S]} \frac{Z(\eps_i)}{\eps_i^2} = T
\end{equation}

On the other hand, the number of active nodes after $T$ rounds is at most the number of zoomed in nodes $u$ at each scale, multiplied by $|\c(u)| \leq \DblC$.
\begin{equation}\label{eq:num-nodes-zoom}
|A_T| \leq  \DblC \sum_{i \in [\S]} Z(\eps_i)
\end{equation}
We next cover each $Z(\eps_i)$ with sets of diameter $\eps_i/2$ (in order to guarantee that each center of the nodes belongs in only one set). Using the definition of the zooming dimension, \refeq{eq:num-nodes-zoom} becomes:
\begin{align*}
T &= \frac{\DblC}{9} \sum_{i \in [\S]} \gamma \cdot \frac{1}{2}^{-z} \frac{\eps_i^{-z}}{\eps_i^2} = \frac{\gamma \cdot \DblC \cdot 2^{z}}{9} \sum_{i \in [\S]} \eps_i^{-z-2} \\
&= \frac{\gamma \cdot \DblC \cdot 2^z}{9} \sum_{i \in [\S]} 2^{-(\S - i)(z+2)} \eps_{\S}^{-z-2}  &\tag{$\eps_i = 2\eps_{t+1}$} \\
&\leq \eps_{\S}^{-z-2} \cdot \frac{\gamma \cdot \DblC \cdot 2^z}{9}\cdot \S
\end{align*}
where $z = \advzoomdim$ and $\gamma$ is the chosen constant multiplier from the definition of the zooming dimension. As a result, re-arranging, the latter inequality becomes:
\begin{equation}\label{eq:def-eps}
\eps_{\S} = \left( \frac{9T}{\gamma \cdot \DblC \cdot \S \cdot 2^z}  \right)^{-\frac{1}{z+2}}
\end{equation}
From \refeq{eq:num-nodes-zoom} and plugging in the zooming dimension definition we have that:
\begin{align*}
|A_T| &= \gamma \cdot \DblC \cdot 2^z \sum_{i \in [\S]} \eps_i^{-z} \leq \gamma \cdot \DblC \cdot 2^z  \cdot \S \cdot \eps_{\S}^{-z} \\
&\leq \left(\gamma \cdot \DblC \cdot 2^z \cdot \S \right)^{\frac{2}{z+2}} \cdot \left( 9T \right)^{\frac{z}{z+2}}\numberthis{\label{eq:nodes-zoom}}
\end{align*}
Since $\S \leq \log T$ and using this bound together with Lemma~\ref{lem:tuned} we have that with probability at least $1-T^{-2}$ we get the result.
\end{proof}

\section{\advzoomdim Examples: Proof of Theorem~\ref{thm:example}}
\label{app:examples-proof}
Fix $\eps>0$. Let us argue about arms that are inclusively $\eps$-optimal in the combined instance. Recall that such arms satisfy $\gap_t(\cdot)\leq \calO(\eps \cdot {\ln} (T) {\cdot \sqrt{d \cdot  \ln(\DblC \cdot T) \cdot \ln \left(T \cdot |\reprset| \right)}})$ for some round $t \geq \eps^{-2}/9$, and we need to cover them with sets of diameter at most $\eps$. So, fix some round $t \geq \eps^{-2}/9$.

{We now define the \emph{induced adversarial instance} at round $t$ which is comprised by the \emph{realized} rewards example described above. Let $I_\tau$ denote the stochastic instance that we face at round $\tau \leq t$. In the induced adversarial instance the total reward of an arm $x \in \calA$ after $t$ rounds is: $G(x) = \sum_{\tau \in [t]} g_{I_\tau}(x)$, where for all $i \in [M]$, $g_i(x)$ is drawn from distribution with mean $\mu_i(x)$.

Let $f_i$ denote the empirical frequency of instance $\calI_i$ up to time $t$. Then, the mean reward for an arm $x \in \calA$ in the induced adversarial instance takes the following form $\mu(x) = \sum_{i \in [M]} f_i \cdot \mu_i(x)$, where $\mu_i(\cdot)$ is the mean reward function for instance $\calI_i$. For the family of examples we consider, this translates to arms $x \in S_i$ for some $i \in [M]$ having mean reward \[ \mu (x) = \sum_{\substack{j \in [M];\\j\neq i}} f_j \cdot b_j + f_i \cdot \mu_i(x),\]and arms $y \in \calA \setminus \cup_{i \in [M]} S_i$ having mean reward $\mu(y) = \sum_{j \in [M]} f_j \cdot b_j$.

}

The next lemma relates the stochastic gap of the induced adversarial instance (denoted by $\gapIID(x)$) with the stochastic gap of one of the instances $\{\calI\}_{i \in [M]}$ (denoted by $\gapIID_i(x)$). To be more specific, let $\xst_j$ denote the optimal arm for instance $\calI_j$, and $\bxst$ the mean-optimal arm in the induced stochastic instance (\ie $\bxst = \arg\max_{x \in \calA} \mu(x)$. Then, $\gapIID_i(x) = \mu_i(\xst_i) - \mu_i(x)$, but $\gap(x) = \mu(\bxst) - \mu(x)$.

\begin{lemma}\label{lem:key-lemma}
If arm $x$ has $\gapIID(x)\leq \eps$ in the induced {adversarial} instance, then
    $\gapIID_{i}(x)\leq \calO(\eps)$
in some instance $\calI_i$.
\end{lemma}

In order to prove Lemma~\ref{lem:key-lemma}, we need some auxiliary claims, which we prove below. We prove first that the mean-optimal arm is among the peaks of the stochastic instances.

\begin{claim}\label{claim:mean-opt-arm}
The mean-optimal arm of the induced stochastic instance $\bxst$ belongs in $\{\xst_i\}_{i \in [M]}$.
\end{claim}

\begin{proof}
We prove this claim by contradiction; assume that there exists arm $y \notin \{\xst_i\}_{i \in [M]}$ and $\mu(y) > \mu(\xst_i), \forall i \in [M]$. If $y$ does not belong in any of the $S_i$'s, then its total mean-payoff at round $t$ is \emph{at most} $\mu(y) = \sum_{i \in [M]} f_i \cdot b_i$. Pick any peak point $y' \in \{\xst_i\}_{i \in [M]}$ at random such that $y' \in S_j$ where instance $\calI_j$ is such that $f_j > 0$. Then, by our setting's assumptions:
\[
\mu(y) = \sum_{i \in [M]} f_i \cdot b_i < \sum_{i \in [M]; i \neq j} f_i \cdot b_i + f_j \cdot \must_j = \mu(y')
\]
which is a contradiction to the fact that $y$ is the mean-optimal arm.
\end{proof}
{Note here that we can only make this claim because the sets $\{S_i\}_{i \in [M]}$ are disjoint; otherwise, we would not be able to claim that fixing point $y'$ instead of $y$ only changes the reward received for the rounds where instance $\calI_j$ appears.}

We next state a useful property regarding the empirical frequency for instance $\calI_{\ist}$ (\ie the instance for which $\bxst \in S_{\ist}$).

\begin{claim}\label{claim:freq-ist}
The empirical frequency of instance $\calI_{\ist}$ is $f_{\ist} \geq \nicefrac{1}{3M}$.
\end{claim}

\begin{proof}
Let $\phist$ the index of the instance with the maximum empirical frequency, i.e., $f_{\phist} = \max_{i \in [M]} f_i$. Then, it holds that $f_{\phist} \geq \nicefrac{1}{M}$. Indeed, notice that if this is not the case: $\sum_{i \in [M]} f_i \leq M \cdot f_{\phist} < 1$, which is a contradiction with the fact that this sum is equal to $1$.

Assume next that the instance with the maximum frequency, $\phist$, is different than the instance $\ist$ for which it holds that $\bxst \in S_{\ist}$. Then it holds that:
\[
\mu(\bxst) \geq \mu(\xst_{\phist}) \Leftrightarrow \sum_{\substack{j \in [M];\\ j \neq \ist}} f_j \cdot b_j + f_{\ist} \cdot \must_{\ist} \geq \sum_{\substack{j \in [M];\\ j \neq \phist}} f_j \cdot b_j + f_{\phist} \cdot \must_{\phist}
\]
where the first inequality is by the definition of $\bxst$ being the mean-optimal arm. Rearranging the above, we get:
\begin{equation}\label{eq:bef-sub-f}
f_{\ist} \cdot (\must_{\ist} - b_{\ist}) \geq f_{\phist}\cdot (\must_{\phist} - b_{\phist}) \Leftrightarrow f_{\ist} \geq \frac{\must_{\phist} - b_{\phist}}{\must_{\ist} - b_{\ist}}
\end{equation}
where for the division we use the fact that $\must_i - b_i \geq \nicefrac{1}{3}$ for all $i \in [M]$. Using this, together with the fact that $\mu_i - b_i \leq 1$, \refeq{eq:bef-sub-f} becomes: $f_{\ist} \geq f_{\phist}/3$. Substituting $f_{\phist} \geq 1/M$ we get the result.
\end{proof}
{We proceed by characterizing the arms that \emph{cannot} be $\eps$-optimal in terms of their mean rewards in the induced instance; these arms $x$ will have $\gapIID(x) > \Omega(1)$.}

\begin{claim}\label{claim:const-gap-irrelevant}
Arms $x \in \calA \setminus \cup_{i \in [M]} S_i$ have $\gapIID(x)\geq c$, where $c = \Omega(\nicefrac{1}{9M})$ is a constant.
\end{claim}

\begin{proof}
We start by the definition of $\gapIID(x)$:
\begin{align*}
\gapIID(x) &= \mu(\bxst) - \mu(x) \\
            &\geq \sum_{\substack{j \in [M];\\ j \neq \ist}} f_j \cdot b_j + f_{\ist} \cdot \must_{\ist} - \sum_{j \in [M]}f_j \cdot b_j \tag{$x \in \calA \setminus \cup_{i \in [M]} S_i$} \\
            &= f_{\ist} \cdot (\must_{\ist} - b_{\ist}) \\
            &\geq f_{\ist} \cdot \frac{1}{3} \tag{by assumption: $\must_i - b_i \geq \nicefrac{1}{3}$} \\
            &\geq \frac{1}{3M} \cdot \frac{1}{3} \tag{Claim~\ref{claim:freq-ist}}
\end{align*}
This concludes our proof.
\end{proof}

{Arms $x \in \calA \setminus \cup_{i \in [M]} S_i$ are not the only ones for which $\gapIID(x) > \Omega(1)$, as shown next.}

\begin{claim}\label{claim:irrelevant-small-freq-gap}
Arms $x \in S_i$ with $i \in [M]: f_i \leq \nicefrac{1}{18M}$ have $\gapIID(x) \geq c'$ , where $c' = \Omega(\nicefrac{1}{18M})$ is a constant.
\end{claim}

\begin{proof}
We start again by the definition of $\gapIID(x)$:
\begin{align*}
\gapIID(x)   &= \mu (\bxst) - \mu (x) \geq \mu (\bxst) - \mu (\xst_i) &\tag{$\mu(x) \leq \mu(\xst_i), \forall x \in S_i$}\\
                    &= \sum_{\substack{j \in [M];\\ j \neq \ist}} f_j \cdot b_j + f_{\ist} \cdot \must_{\ist} - \sum_{\substack{j \in [M];\\ j \neq i}} f_j \cdot b_j - f_i \cdot \must_i \\
                    &= f_{\ist} \cdot (\must_{\ist} - b_{\ist}) - f_i \cdot (\must_i - b_i) \\
                    &\geq \frac{1}{3M} \cdot \frac{1}{3} - f_i \cdot (\must_i - b_i)\tag{Claim~\ref{claim:freq-ist} and $\must_i - b_i \geq \nicefrac{1}{3}$}\\
                    &\geq \frac{1}{9M} - \frac{1}{18 M}\cdot 1 \tag{$\must_i - b_i \leq 1$ and $f_i \leq \nicefrac{1}{18M}$} \\
                    &= \frac{1}{18M}
\end{align*}
This concludes our proof.
\end{proof}
{We are now ready to prove Lemma~\ref{lem:key-lemma}.}

\begin{proof}[Proof of Lemma~\ref{lem:key-lemma}]
Due to Claims~\ref{claim:const-gap-irrelevant} and~\ref{claim:irrelevant-small-freq-gap} for $\gapIID(x) \leq \eps$ where $\eps < o(1)$, we only need to focus on arms $x$ for which it holds that $x \in S_j$ for some $j \in [M]$ such that $f_j > \nicefrac{1}{18M}$.

Fix an arm $x$ such that $\gapIID(x) \leq \eps$ and let $j$ be such that $x \in S_j$ and $f_j > \nicefrac{1}{18M}$. Then:
\begin{align*}
\eps    &\geq \gapIID(x) = \mu (\bxst) - \mu(x) = \mu(\bxst) - \mu(\xst_j) + \mu(\xst_j) - \mu(x)\\
        &= \mu(\bxst) - \mu (\xst_j) + \underbrace{\left(\must_j - \mu_j(x) \right)}_{\gapIID_j(x)} \cdot f_j \\
        &\geq \gapIID_j(x) \cdot f_j \tag{$\mu(\bxst) \geq \mu(\xst_j)$} \\
        &\geq \frac{1}{18M} \cdot \gapIID_j (x) \tag{Claim~\ref{claim:irrelevant-small-freq-gap}}
\end{align*}
Hence, for arms $x$ such that $\gapIID(x) \leq \eps$ it holds that $\gapIID_j(x) \leq 18M \cdot \eps = \calO(\eps)$ for some stochastic instance $j \in [M]$. This concludes our proof.
\end{proof}

In the next step of the proof, we will connect an arm's stochastic gap of the induced adversarial instance with its adversarial gap. The result is similar in flavor to Proposition~\ref{prop:application-azuma}. Recall that $\gapIID(x)$ corresponds to the stochastic gap for arm $x$ in the induced adversarial instance.

\begin{proposition}\label{prop:application-azuma-examples}
Fix time $t$. For any arm $x \in \reprset$, with probability at least $1 - \nicefrac{1}{T}$ for the induced adversarial instance of the example it holds that:
\[
\left|\gap_t(x) - \gapIID(x) \right| \leq 3 \sqrt{\frac{2 \ln \left(T \cdot |\reprset| \right)}{t}}
\]
where $\gapIID(x)$ is the stochastic gap of the mean rewards in the induced adversarial instance.
\end{proposition}

\begin{proof}
The proof is similar to the one of Proposition~\ref{prop:application-azuma}, but there are a few subtle differences. Observe that:
\[
\E \left[ \frac{1}{t} G(x) \right] = \E \left[\frac{1}{t}\sum_{\tau \in [t]}g_{I_\tau}(x) \right] = \sum_{i \in [M]} f_i \cdot \mu_i(x) = \mu(x)
\]
Fixing the sequence $\{I_\tau\}_{\tau \in [t]}$, we can define the following martingale:
\[
Y_{t'} = \sum_{\tau \in [t']} \left(g_{I_\tau}\left(\bxst\right) - g_{I_\tau}\left(x\right) \right) - t' \cdot \left( \mu^{t'}\left(\bxst \right) - \mu^{t'}(x) \right)
\]
where $\mu^{t'}(x)$ is the mean reward of arm $x \in \calA$ in the induced adversarial instance of sequence $\{I_\tau\}_{\tau \leq t'}$ with $t' \leq t$. In other words, denoting by $f_i^{t'}$ the empirical frequency of instance $i$ in sequence $\{I_\tau\}_{\tau \in t'}$ with $t' \leq t$: $\mu^{t'}(x) = \sum_{i \in [M]} f^{t'}_i \cdot \mu(x)$. To see that $Y_{t'}$ is indeed a martingale, let $f_i^{t'}$ be the empirical frequency of instance $i$ in sequence $\{I_\tau\}_{\tau \in t'}$ with $t' \leq t$ and note that:
\begin{align*}
\E \left[Y_{n+1} | Y_1, \dots, Y_n \right] &= \E \left[ \sum_{\tau \in [n+1]} \left(g_{I_\tau}\left(\bxst\right) - g_{I_\tau}\left(x\right) \right) - (n+1) \cdot \left( \mu^{n+1}\left(\bxst \right) - \mu^{n+1}(x) \right) \big| Y_1, \dots, Y_n \right] \\
&= \sum_{i \in [M]} f_i^{n+1} \cdot \mu\left(\bxst \right) - \sum_{i \in [M]} f_i^{n+1} \cdot \mu\left(\bxst \right) + Y_n \\
&= Y_n
\end{align*}
where the second equation is due to the definition of $\mu^{t'}(x)$ and the fact that: $\E \left[g_{I_\tau}(x) \cdot \1 \{\tau \leq t' \leq t \} \right] = \sum_{i \in [M} f_i^{t'} \cdot \mu_i(x)$. Applying the Azuma-Hoeffding inequality (Lemma~\ref{lem:azuma}), and since $|Y_{t'+1} - Y_{t'}| \leq 1$ for rewards bounded in $[0,1]$ we have: $\Pr [Y_{t'}] \geq \sqrt{2t \ln (1 /\delta)}] \leq \delta, \forall \delta >0$. The rest of the steps in the proof are similar to the proof of Proposition~\ref{prop:application-azuma} and hence we omit them. 
\end{proof}

We are now ready to complete the proof of Theorem~\ref{thm:example}.

\begin{proof}[Proof of Theorem~\ref{thm:example}]
It remains to show that the arms for which $\gap_t(x)$ is adequately small in the induced adversarial instance, have small stochastic gap $\gapIID(x)$ with high probability. Then, by Lemma~\ref{lem:key-lemma} we can relate $\gap_t(x)$ with the stochastic gap $\gapIID_i(x)$ for some instance $i \in [M]$.

Let arm $x \in \reprset$ be such that $\gap_t(x) \leq 30 \eps \cdot \ln (T) \cdot \sqrt{d \ln \left( \DblC \cdot T\right)}$ for $\eps = (3\sqrt{t})^{-1}$. Then, from Proposition~\ref{prop:application-azuma-examples} it holds that with probability at least $1 - \nicefrac{1}{T}$:
\begin{align*}
\gapIID(x)  &\leq 30 \eps \cdot \ln (T) \cdot \sqrt{d \ln \left( \DblC \cdot T\right)} + \eps \sqrt{18 \ln \left( T \cdot |\reprset| \right)} \\
            &\leq \underbrace{31 \eps \cdot \ln (T) \cdot \sqrt{d \cdot \ln \left( \DblC \cdot T\right) \cdot \ln \left( T \cdot |\reprset | \right)}}_{\eps'}
\end{align*}

From Lemma~\ref{lem:key-lemma}, since $\gapIID(x) \leq \eps'$, there exists an instance $j \in [M]$ such that $\gapIID_j(x) \leq 18 M \eps'$. From the definition of zooming dimension for this instance $j$ these arms can be covered by at most $\gamma \cdot (18 M \eps')^{-z}$ sets of diameter $18M\eps'$, where $z = \zoomdim_\gamma$. Equivalently, these arms can be covered by:
\begin{align*}
&\gamma \cdot \left( \frac{18M\eps'}{\eps}\right)^{\log (\DblC)} \cdot \left(18 M \eps' \right)^{-z} \\
& = \gamma \cdot \left(558 \cdot M \cdot \ln (T) \cdot \sqrt{d \cdot \ln(\DblC \cdot T) \cdot \ln (T \cdot |\reprset|)} \right)^{\log (\DblC) - z} \cdot \eps^{-z}
\end{align*}
sets of diameter $\eps$. Since the $\advzoomdim$ is the smallest dimension needed to cover these arms (because originally $\gap_t(x) \leq O ( \eps \ln (T) \sqrt{d \ln (\DblC T)})$, then $\advzoomdim_{\gamma'} \leq z$ for $\gamma' = \gamma\cdot\calO \left( M \ln (T) \sqrt{d \ln \left( \DblC \cdot T \right) \cdot \ln \left( T \cdot |\reprset| \right)}\right)$.
\end{proof}

\newpage

\newpage
\section{Extension to Arbitrary Metric Spaces}
\label{sec:metrics}
In this appendix, we sketch out an extension to arbitrary metric spaces. The main change is that the zooming tree is replaced with a more detailed decomposition of the action space. Similar decompositions have been implicit in all prior work on adaptive discretization, starting from \citep{LipschitzMAB-JACM,xbandits-nips08}. No substantial changes in the algorithm or analysis are needed.

\xhdr{Preliminaries.}
Fix subset $S\subset \arms$ and $\eps>0$. The diameter of $S$ is
    $\sup_{x,y\in S'}\dist(x,y)$.
An \emph{$\eps$-covering} of $S$ is a collection of subsets $S'\subset \arms$ of diameter at most $\eps$ whose union covers $S$.
The $\eps$-covering number of $S$, denoted $\covnum_\eps(S)$, is the smallest cardinality of an $\eps$-covering. Note that the covering property in \eqref{eq:dims-generic} can be restated as
$\inf \cbr{ d\geq 0:\;
    \covnum_\eps(\arms_\eps)\leq \gamma\cdot \eps^{-d},
    \quad \forall \eps>0}$.

A \emph{greedy $\eps$-covering} of $S$ is an $\eps$-covering constructed by the following ``greedy'' algorithm: while there is a point $x\in S$ which is not yet covered, add the closed ball $B(x,\eps/2)$ to the covering. Thus, this $\eps$-covering  consists of closed balls of radius $\eps/2$ whose centers are at distance more than $\eps/2$.

A \emph{rooted} directed acyclic graph (DAG) is a DAG with a single source node, called the \emph{root}. For each node $u$, the distance from the root is called the \emph{height} of $u$ and denoted $h(u)$. The subset of nodes reachable from $u$ (including $u$ itself) is called the \emph{sub-DAG} of $u$. For an edge $(u,v)$, we say that $u$ is a \emph{parent} and $v$ is a \emph{child} relative to one another. The set of all children of $u$ is denoted $\c(u)$.

\xhdr{Metric Space Decomposition.} Our decomposition is a rooted DAG,  called \emph{Zooming DAG}, whose nodes correspond to balls in the metric space.

\begin{definition}[zooming DAG]
A \emph{zooming DAG} is a rooted DAG of infinite height. Each node $u$ corresponds to a closed ball $B(u)$ in the action space, with radius $r(u) = 2^{-h(u)}$ and center $x(u)\in\arms$. These objects are called, respectively, the \emph{action-ball}, the \emph{action-radius}, and the \emph{action-center} of $u$. The following properties are enforced:
\begin{OneLiners}

\item[(a)] each node $u$ is covered by the children:
    $B(u) \subset \cup_{v\in \c(u)} B(v)$.

\item[(b)] each node $u$ overlaps with each child $v$:
    $B(u)\cap B(v)\neq \emptyset$.

\item[(c)] for any two nodes of the same action-radius $r$,
their action-centers are at distance $>r$.

\end{OneLiners}
The \emph{action-span} of $u$ is the union of all action-balls in the sub-DAG of $u$.
\end{definition}

Several implications are worth spelling out:
\begin{itemize}
\item the nodes with a given action-radius $r$  cover the action space (by property (a)), and there are at most $\covnum_r(\arms)$ of them (by property (c)). Recall that
    $\covnum_r(\arms) \leq \gamma\cdot r^{-d}$,
where $d$ is the covering dimension with multiplier $\gamma$.

\item each node $u$ has at most $\covnum_{r(u)/2}(B(u))\leq \DblC$ children (by properties (b,c)), and its action-span lies within distance
    $3\,r(u)$ from its action-center (by property (b)).
\end{itemize}

A zooming DAG exists, and can be constructed as follows. The nodes with a given action-radius $r$ are constructed as a greedy $(2r)$-cover of the action space. The children of each node $u$ are all nodes of action-radius $r(u)/2$ whose action-balls overlap with $B(u)$.

Our algorithm only needs nodes of height up to $O(\log T)$. We assume that some ``zooming DAG'', denoted \ZoomDAG, is fixed and known to the algorithm.

Note that a given node in \ZoomDAG may have multiple parents. Our algorithm adaptively constructs subsets of \ZoomDAG that are directed trees. Hence a definition:

\begin{definition}[zooming tree]
A subgraph of \ZoomDAG is called a \emph{zooming tree} if it is a finite directed tree rooted at the root of \ZoomDAG. The \emph{ancestor path} of node $u$ is the path from the root to $u$.
\end{definition}

For a $d$-dimensional unit cube, \ZoomDAG can be defined as a zooming tree, as per Section~\ref{sec:algo}.

\xhdr{Changes in the Algorithm.} When zooming in on a given node $u$, it activates all children of $u$ in \ZoomDAG that are not already active (whereas the version in Section~\ref{sec:algo} activates all children of $u$). The representative arms $\repr_t(u)$ are chosen from the action-ball of $u$.

\xhdr{Changes in the Analysis.}
We account for the fact that the action-span of each node $u$ lies within $3\,r(u)$ of its action-center (previously it was just $r(u)$). This constant $3$ is propagated throughout.

\end{document}